\documentclass[letterpaper]{article} 
\usepackage{aaai25}  
\usepackage{times}  
\usepackage{helvet}  
\usepackage{courier}  
\usepackage[hyphens]{url}  
\usepackage{graphicx} 
\usepackage{natbib}  
\usepackage{caption} 
\usepackage{algorithm}
\usepackage{algpseudocode}
\usepackage{amsmath,amsfonts}
\usepackage{amsthm}
\usepackage{amssymb}
\usepackage{xcolor}
\usepackage{subfigure}
\usepackage{multirow} 
\usepackage{booktabs} 
\usepackage[utf8]{inputenc}
\usepackage{enumerate} 

\usepackage{newfloat}
\usepackage{listings}
\DeclareCaptionStyle{ruled}{labelfont=normalfont,labelsep=colon,strut=off} 
\floatstyle{ruled}
\newfloat{listing}{tb}{lst}{}
\floatname{listing}{Listing}
\newtheorem{theorem}{\textbf {Theorem}}
\newtheorem{definition}{\textbf {Definition}}
\newtheorem{lemma}{\textbf {Lemma}}

\newtheorem{corollary}{\textbf {Corollary}}
\newtheorem{assumption}{\textbf {Assumption}}

\title{Non-Convex Tensor Recovery from Local Measurements}
\author{
     Tongle Wu\textsuperscript{\rm 1}, Ying Sun\textsuperscript{\rm 1}\equalcontrib,  Jicong Fan\textsuperscript{\rm 2}\equalcontrib \\
}

\affiliations{
    \textsuperscript{\rm 1}School of Electrical Engineering and Computer Science, Pennsylvania State University\\
    \textsuperscript{\rm 2}School of Data Science, The Chinese University of Hong Kong, Shenzhen\\
    \textsuperscript{\rm 1}\{tfw5381, ybs5190\}@psu.edu, \textsuperscript{\rm 2}fanjicong@cuhk.edu.cn
}

\usepackage{bibentry}

\begin{document}

\maketitle

\begin{abstract}

Motivated by the settings where sensing the entire tensor is infeasible, this paper proposes a novel tensor compressed sensing model, where measurements are only obtained from sensing each lateral slice via mutually independent matrices. 
Leveraging the low tubal rank structure, we reparameterize the unknown tensor ${\boldsymbol {\mathcal X}}^\star$ using two compact tensor factors and formulate the recovery problem as a nonconvex minimization problem. To solve the problem, we first propose an alternating minimization algorithm, termed \textsf{Alt-PGD-Min}, that iteratively optimizes the two factors using a projected gradient descent and an exact minimization step, respectively. Despite nonconvexity, we prove that \textsf{Alt-PGD-Min} achieves $\epsilon$-accuracy recovery with $\mathcal O\left( \kappa^2 \log \frac{1}{\epsilon}\right)$ iteration complexity and $\mathcal O\left( \kappa^6rn_3\log n_3 \left( \kappa^2r\left(n_1 + n_2 \right) + n_1 \log \frac{1}{\epsilon}\right) \right)$  sample complexity, where $\kappa$ denotes tensor condition number of $\boldsymbol{\mathcal X}^\star$. 
To further accelerate the convergence, especially when the tensor is ill-conditioned with large $\kappa$,  we prove  \textsf{Alt-ScalePGD-Min}  that preconditions the gradient update using an approximate Hessian that can be computed efficiently.  We show that \textsf{Alt-ScalePGD-Min} achieves $\kappa$ independent iteration complexity $\mathcal O(\log \frac{1}{\epsilon})$ and improves the sample complexity to $\mathcal O\left( \kappa^4 rn_3 \log n_3 \left( \kappa^4r(n_1+n_2) + n_1 \log \frac{1}{\epsilon}\right) \right)$. Experiments validate the effectiveness of the proposed methods.

\end{abstract} 

%

\section{Introduction}
Motivated by the well-known compressed sensing and matrix sensing \cite{candes2006robust,recht2010guaranteed} problems, tensor compressed sensing (TCS) has attracted increasing attention in recent years \cite{shi2013guarantees,rauhut2017low,tong2022scaling,chen2019non}. The goal of TCS is to recover a tensor $\boldsymbol{\mathcal X}^\star \in \mathbb{R}^{n_1 \times n_2 \times n_3}$ from a few measurements $\boldsymbol y \in \mathbb{R}^m (m\ll n_1n_2n_3)$, where $\boldsymbol y= {\mathcal A}\left( \boldsymbol{\mathcal X}^\star \right)$ and $\mathcal A: \mathbb{R}^{n_1\times n_2 \times n_3} \rightarrow \mathbb{R}^m$ is a linear operator. Since this problem is ill-posed for arbitrary $\boldsymbol{\mathcal X}^\star$, the success of TCS relies on the existence of the low dimensional intrinsic structure in the original high-order tensor $\boldsymbol{\mathcal X}^\star$. Such a structure has been widely validated and utilized in many real-world applications such as dynamic MRI \cite{yu2014multidimensional,gong2024non}, video compression \cite{li2022tensor,ma2019deep,liu2023real}, snapshot compressive imaging \cite{ma2019deep,liu2023real}, quantum computing \cite{ran2020tensor,kuzmin2024learning}, image recovery \cite{fan2023euclideannorminduced}, and collaborative filtering \cite{fan2022multimode}.

In the literature, different tensor decompositions can induce different definitions of tensor rank \cite{kolda2009tensor}, which are often more complex than matrix rank. Consequently, the non-uniqueness and complexity of the tensor rank make TCS a non-trivial extension of matrix sensing. Most works of TCS assume that the ground truth tensor $\boldsymbol{\mathcal X}^\star$ has a low Tucker rank \cite{han2022optimal,luo2023low,ahmed2020tensor,mu2014square} or a low tubal rank \cite{zhang2020rip,hou2021robust,lu2018exact,liu2023tensor}, which are induced by Tucker decomposition \cite{tucker1966some} and tensor Singular Value Decomposition (t-SVD) \cite{kilmer2011factorization} respectively. In these works, the measurements are given by 
\begin{align}
{y}_i = \langle \boldsymbol{\mathcal A}_i, \boldsymbol{\mathcal X}^\star \rangle, ~~~i \in [m],\label{TCS_1}
\end{align}
where $\boldsymbol{\mathcal A}_i \in \mathbb{R}^{n_1 \times n_2 \times n_3}$ has i.i.d. zero-mean Gaussian entries and can sense entire $\boldsymbol{\mathcal X}^\star$.


It should be pointed out that in many scenarios, it is difficult to sense the entire tensor $\boldsymbol{\mathcal X}^\star$, preventing the application of the sensing model \eqref{TCS_1}. 
For instance, due to memory or privacy limitations, large-scale tensor data may be partitioned into multiple smaller tensors stored in a distributed network \cite{moothedath2024decentralized,singh2024byzantine,wu2024implicit}. Another example is when the tensor $\boldsymbol{\mathcal X}^\star$, such as images or videos, is collected in an on-the-fly streaming setting \cite{srinivasa2019decentralized}.

To address the challenge, we propose a novel tensor compressed sensing model where each measurement is generated from locally sensing a slice of $\boldsymbol{\mathcal X}^\star$. Our model is detailed as follows.
\begin{definition}
\textbf{(Local TCS)} For each lateral slice $i\in [n_2]$, its $j$-th local measurement $y_{ji}$ is obtained by 
\begin{align}
y_{ji} = \langle  \boldsymbol{\mathcal A}_i(:,j,:), \boldsymbol{\mathcal X}^\star(:,i,:)  \rangle, ~~i \in [n_2],\; j \in [m], \label{TCS_2}
\end{align}
where the $\boldsymbol{\mathcal A}_i(:,j,:) \in \mathbb R^{n_1 \times n_3}$ denotes the $j$-th sensing matrix for $i$-th lateral slice of $\boldsymbol{\mathcal X}^\star$.
\end{definition}
Take the dynamic video sensing mentioned before as an example. Under the local TCS model in \eqref{TCS_2}, the entire video is modeled as $\boldsymbol{\mathcal X}^\star$, with each lateral slice $\boldsymbol{\mathcal X}^\star(:,i,:)$ representing the video frame at $i$-th timestamp \cite{li2017low,zhou2017tensor,wang2020robust}. The observations $\left\{y_{ji}\right\}_{j=1}^m$ is obtained by measuring the $i$-th frame of video $\boldsymbol{\mathcal X}^\star$. 
We aim to  recover  $\boldsymbol{\mathcal X}^\star \in \mathbb R^{n_1\times n_2 \times n_3}$ from the measurements  $\left\{y_{ji}\right\}_{i,j=1}^{i=n_2,j=m}$ obtained via~\eqref{TCS_2}. Under this framework,  fundamental questions to understand are: 
\begin{center}
  \emph{Under what conditions can we provably recover $\boldsymbol{\mathcal X}^\star$,\\
  and how to compute the solution  efficiently?} 
\end{center}

This paper considers the problem above under the structural assumption that the ground truth $\boldsymbol{\mathcal X}^\star$ 
is of low tubal rank with $r \ll \min \left\{n_1,n_2,n_3\right\}$ (see Definition~\ref{def:t-rank}). We focus on low tubal rank for two key reasons. First, it can be computed more efficiently by solving multiple SVDs in the Fourier domain compared to CP and Tucker ranks \cite{zhang2014novel}. Second, the convolution operator in this model is particularly effective at capturing the ``spatial-shifting'' properties of data \cite{liu2019low,wu2022low,wu2024smooth}. Our main contributions are summarized as follows.

\begin{itemize}
    \item We introduce a novel local TCS model \eqref{TCS_2} for tensor compressed sensing with measurements obtained by lateral slice-wise sensing. Compared to the traditional TCS model \eqref{TCS_1},  local sensing does not rely on the availability of the entire tensor,  which greatly enlarges its applicable scenarios, such as real-time and distributed processing.

    \item We formulate the recovery problem as a nonconvex minimization problem based on the low tubal rank tensor factorization for $\boldsymbol{\mathcal X}^\star$. An alternating minimization algorithm, called \textsf{ Alt-PGD-Min}, is proposed to solve the problem with efficient computations per iteration. We show that under suitable conditions on the sensing operator,
    with $\mathcal O\left( \kappa^6rn_3\log n_3 \left( \kappa^2r\left(n_1 + n_2 \right) + n_1 \log \frac{1}{\epsilon}\right) \right)$ samples \textsf{ Alt-PGD-Min} computes a solution that is $\epsilon$-close to $\boldsymbol{\mathcal X}^\star$ in $\mathcal O\left( \kappa^2 \log \frac{1}{\epsilon}\right)$ iterations, where $\kappa$ is the tensor condition number of $\boldsymbol{\mathcal X}^\star$.     
    \item To improve the dependency of both the sample and iteration complexity on $\kappa$, we further proposed  \textsf{Alt-ScalePGD-Min} that preconditions the gradient step in \textsf{ Alt-PGD-Min} using an approximation of the Hessian matrix that is cheap to compute. We show that by incorporating the preconditioner, \textsf{Alt-ScalePGD-Min}  iteration complexity $\mathcal O(\log \frac{1}{\epsilon})$ that is independent of $\kappa$, 
    and improves the sample complexity to $\mathcal O\left( \kappa^4 rn_3 \log n_3 \left( \kappa^4r(n_1+n_2) + n_1 \log \frac{1}{\epsilon}\right) \right)$.
    \item We validated the proposed local sensing model and algorithms on both synthetic and real-world data. Numerical results show that the proposed algorithms can achieve effective performance in the local TCS model \eqref{TCS_2}.
    
\end{itemize}

\section{Related Work}
\paragraph{Tensor Compressed Sensing (TCS).} Canonical TCS problems \eqref{TCS_1} based on Tucker and t-SVD decompositions have been extensively investigated in recent years. Studies in \cite{shi2013guarantees,rauhut2017low,ahmed2020tensor,mu2014square,chen2019non,han2022optimal,luo2023low} utilized convex or non-convex optimization methods to solve low Tucker rank based tensor CS. For TCS with low tubal rank under the t-SVD framework, \citet{lu2018exact} proposed a convex method that minimizes the tensor nuclear norm (TNN) with order optimal sample complexity. \citet{zhang2020rip} proposed a regularized TNN minimization method with provable robust recovery performance from noisy observations based on the defined tensor Restricted Isometry Property (RIP). \citet{hou2021robust} proposed convex methods to solve one-bit TCS from binary observations and provided robust recovery guarantees. \citet{liu2024low} developed theoretical guarantees for the non-convex gradient descent method, which deals with exact low tubal rank and overparameterized tensor factorizations for solving the model \eqref{TCS_1}. \citet{liu2023tensor} fused low-rankness and local-smoothness of real-world tensor data in TCS and obtained the provable enhanced recovery guarantee. However, all existing TCS studies focus on models where random measurement tensors have access to the entire ground truth tensors, rather than recovering the low-rank tensor through local measurements as proposed in our TCS model in \eqref{TCS_2}.
\paragraph{CS from Local Measurements.} Compared to the canonical CS model, the investigation of CS that recovers from local measurements as model \eqref{TCS_2} is less explored. \citet{nayer2022fast,vaswani2024efficient,srinivasa2019decentralized,srinivasa2023sketching,lee2023approximately} considered matrix sensing model that involves recovering a low-rank matrix from independent compressed measurements of each of its columns. \citet{srinivasa2019decentralized,srinivasa2023sketching,lee2023approximately} proposed convex programming methods that minimize relevant mixed norms with provable guarantees, while \citet{nayer2022fast,vaswani2024efficient} proposed efficient non-convex method and obtained improved iteration and sample complexities.

Nevertheless, to our knowledge, all existing studies on CS from local measurements focused solely on two-dimensional matrices, rather than higher-order tensors that widely exist in science and engineering. Reshaping a tensor into a matrix format to apply matrix methods overlooks the interactions across all dimensions and destroys the inherent structures of data. Therefore, it is essential to study TCS from local measurements, as proposed in our model \eqref{TCS_2}, which has not been addressed in the literature, despite its significant applications in areas like video compression for online or distributed settings \cite{wu2024implicit,srinivasa2019decentralized}. Our experimental results such as Figures \ref{fig1} and \ref{fig4} will demonstrate the superiority of tensor CS over matrix CS.

\section{Preliminaries}
\subsection{Notations}
We use letters $x, \boldsymbol x, \boldsymbol X,  {\boldsymbol {\mathcal X}}$ to denote scalars, vectors, matrices, and tensors, respectively. Let $\left\{a_n,b_n\right\}_{n\geq 1}$ be any two positive series. We write $a_n \gtrsim b_n$ (or $a_n \lesssim b_n$) if there exists a universal constant $c>0$ such that $a_n \geq c b_n$ (or $a_n \leq c b_n$). The notations of $a_n= \Omega (b_n) $ and $a_n = \mathcal O \left( b_n\right)$ share the same meaning with $a_n \gtrsim b_n$ and $a_n \lesssim b_n$.

The $i$-th horizontal, lateral, and frontal slice matrix of $\boldsymbol{\mathcal{X}}$ are denoted as ${\boldsymbol {\mathcal X}}(i,:,:)$, ${\boldsymbol {\mathcal X}}(:,i,:)$, and ${\boldsymbol {\mathcal X}}(:,:,i)$ respectively. The  $(i,j,k)$-th element is denoted as ${\boldsymbol {\mathcal X}}_{ijk}$. 
For simplicity, we also use $\boldsymbol X^{(i)}$  to denote the $i$-th frontal slice. $\boldsymbol{\mathcal X}(i) \in \mathbb R^{n_1 \times 1 \times n_3}$  denotes the tensor that only composed of the $i$-th lateral slice of $\boldsymbol{\mathcal X}$. The inner product of tensors is denoted as $\langle \boldsymbol {\mathcal X}, \boldsymbol {\mathcal Y}\rangle = \sum_{ijk}{\boldsymbol{\mathcal X}_{ijk} \boldsymbol{\mathcal Y}_{ijk}}$. The Frobenius norm of tensor is denoted as $\| \boldsymbol{\mathcal X}\|_F = \sqrt{\sum_{ijk}{\boldsymbol {\mathcal X}_{ijk}^2}}$. We use fft($\boldsymbol {\mathcal X},[\;],3)=\overline{\boldsymbol{\mathcal X}}\in \mathbb{C}^{n_1\times n_2\times n_3}$ to denote performing DFT on all the tubes of $\boldsymbol {\mathcal{X}} \in \mathbb{R}^{n_1\times n_2 \times n_3}$. The inverse FFT on $\overline{\boldsymbol{\mathcal X}}$ can turn it back to the original tensor, i.e., $\boldsymbol {\mathcal{X}}$=ifft$(\overline{\boldsymbol{\mathcal X}},[\;],3)$.


\subsection{Definitions, Tensor Factorization and  Tubal Rank}

Unfold and Fold operators for a tensor are defined as
\begin{align}
 & \text{Unfold}\left( \boldsymbol{\mathcal X} \right) := \left[ \boldsymbol X^{(1)};\cdots ; \boldsymbol X^{(n_3)} \right] \\
 & \text{Fold}\left(\text{Unfold}\left( \boldsymbol{\mathcal X} \right)\right) := \boldsymbol{\mathcal X}.
\end{align}
Denote the block circulant matrix of $\boldsymbol{\mathcal X}$  as
\begin{equation*}
\begin{aligned}\text{bcirc}(\boldsymbol{\mathcal X}):=\left[
\begin{matrix}
\boldsymbol X^{(1)} & \boldsymbol X^{(n_3)} & \cdots & \boldsymbol X^{(2)}\\
\boldsymbol X^{(2)} & \boldsymbol X^{(1)} & \cdots & \boldsymbol X^{(3)}\\
\vdots & \vdots & \ddots & \vdots \\
\boldsymbol X^{(n_3)} & \boldsymbol X^{(n_3-1)} & \cdots & \boldsymbol X^{(1)}
\end{matrix}\right].
\end{aligned}
\end{equation*}
The above $ \text{bcirc}(\boldsymbol{\mathcal X}) $ can be block diagonalized as
\begin{align}
\left( \boldsymbol F_{n_3} \otimes \boldsymbol I_{n_1}\right) \cdot  \text{bcirc}(\boldsymbol{\mathcal X})  \cdot \left( \boldsymbol F_{n_3} \otimes \boldsymbol I_{n_2}\right) = \overline{\boldsymbol X},
\end{align}
where $\boldsymbol F_{n}$ denotes the $n$-dimensional discrete Fourier transformation matrix, $\otimes$ denotes the Kronecker product and $\overline{\boldsymbol X}$ is defined as:
\begin{align}
\overline{\boldsymbol X}: = \text{bdiag} \left(\overline{\boldsymbol{\mathcal X}}\right) := \text{diag}\left(\overline{\boldsymbol X}^{(1)};\cdots;\overline{\boldsymbol X}^{(3)}\right).
\end{align}

With the above definitions, we introduce the following arithmetic operations for tensors.
\noindent \begin{definition}\label{t-t}\textbf{(Tensor-Tensor product (T-product ))} \cite{kilmer2011factorization} 
The tensor product between tensors $\boldsymbol{\mathcal X} \in \mathbb{R}^{n_1 \times n_2 \times n_3}$ and $\boldsymbol{\mathcal Y} \in \mathbb{R}^{n_2 \times n_4 \times n_3}$ is defined as: 
\begin{equation} 
\boldsymbol{\mathcal X}*\boldsymbol{\mathcal Y}=\text{Fold}(\text{bicrc}(\boldsymbol{\mathcal X} )\text{Unfold}(\boldsymbol{\mathcal Y}))\quad \in \mathbb{R}^{n_1\times n_4\times n_3}.
\end{equation}
\end{definition}

\begin{definition}\textbf{(Conjugate transpose)} \cite{lu2020tensor}  The conjugate transpose of $\boldsymbol{\mathcal X} \in \mathbb{R}^{n_1 \times n_2 \times n_3}$ is $\boldsymbol{\mathcal X}^c$ that
$
\boldsymbol{\mathcal X}^c(:,:,1) =\left( \boldsymbol X^{(1)}\right)^c, \boldsymbol{\mathcal X}^c(:,:,n_3+2-i) = \left(\boldsymbol X^{(i)}\right)^c
$ for $2\leq i \leq n_3$, where  $\boldsymbol X^c$ is conjugate transpose of $\boldsymbol X$.
\end{definition}

\begin{definition}\textbf{(Identity tensor)} \cite{kilmer2011factorization} If $\boldsymbol {\mathcal I}(:,:,1) = \boldsymbol I_n$ and $\boldsymbol {\mathcal I}(:,:,i) = \boldsymbol 0_n$ for $2\leq i\leq n_3$, then $\boldsymbol{\mathcal I} \in \mathbb R^{n\times n \times n_3}$ is defined as the identity tensor. 
\end{definition}

\begin{definition}\textbf{(Orthogonal tensor)} \cite{kilmer2011factorization} A tensor $\boldsymbol{\mathcal Q} \in \mathbb R^{n\times n \times n_3}$ is defined as the orthogonal tensor if $\boldsymbol{\mathcal Q}* \boldsymbol{\mathcal Q}^c = \boldsymbol{\mathcal Q}^c * \boldsymbol{\mathcal Q} = \boldsymbol{\mathcal I}$.
\end{definition}

\begin{definition} \textbf{(Tensor inverse)} \cite{kilmer2011factorization}
An $n\times n \times n_3$ tensor $\boldsymbol{\mathcal X}$ has an inverse $\boldsymbol{\mathcal Y}$ if
$
\boldsymbol{\mathcal X}* \boldsymbol{\mathcal Y} = \boldsymbol{\mathcal I} \; ~~\text{and} \;~~  \boldsymbol{\mathcal Y} *  \boldsymbol{\mathcal X} = \boldsymbol{\mathcal I}.
$
If $\boldsymbol{\mathcal X}$ is invertible, we use $\boldsymbol{\mathcal X}^{-1}$ to denote its inverse.
\end{definition}


\begin{definition}\textbf{(F-diagonal tensor)} \cite{kilmer2011factorization} If all of frontal slices of $\boldsymbol{\mathcal X}$ are diagonal matrices, then $\boldsymbol{\mathcal X}$ is called an $f$-diagonal tensor.
\end{definition}

\begin{theorem}\label{thm1} \textbf{(t-SVD)} \cite{lu2020tensor}
Let $ \boldsymbol{\mathcal X} \in \mathbb R^{n_1\times n_2 \times n_3}$. Then it can be factorized as
\begin{align}
\boldsymbol{\mathcal X} = \boldsymbol{\mathcal U}*\boldsymbol{\mathcal S}* \boldsymbol{\mathcal V}^c,
\end{align}
where $\boldsymbol{\mathcal U} \in \mathbb R^{n_1\times n_1 \times n_3}, \boldsymbol{\mathcal V} \in \mathbb R^{n_2 \times n_2 \times n_3}$ are orthogonal tensors and $\boldsymbol{\mathcal S} \in \mathbb R^{n_1 \times n_2 \times n_3}$ is an $f$-diagonal tensor.
\end{theorem}

Similar to matrices, the tensor QR factorization is defined as follows.

\begin{theorem} \label{thm2} \textbf{(T-QR)} \cite{kilmer2011factorization}
Let $\boldsymbol{\mathcal X} \in \mathbb R^{n_1 \times n_2 \times n_3}$. Then it can be factorized as 
\begin{align}
\boldsymbol{\mathcal X} = \boldsymbol{\mathcal Q} * \boldsymbol{\mathcal{R}}, \label{t-QR}
\end{align}
where $\boldsymbol{\mathcal Q} \in \mathbb R^{n_1 \times n_1 \times n_3}$ is orthogonal, and $\boldsymbol{\mathcal{R}}^{n_1 \times n_2 \times n_3}$ is an $f$-upper triangular tensor whose frontal slices are all upper triangular matrices.
\end{theorem}

\begin{definition}\textbf{(Tubal rank)}\label{def:t-rank}
\cite{lu2020tensor} For a tensor $\boldsymbol{\mathcal X} \in \mathbb R^{n_1\times n_2 \times n_3}$, its tubal rank is defined as the number of nonzero singular tubes of $\boldsymbol{\mathcal S}$, where $\boldsymbol{\mathcal S}$ is the $f$-diagonal tensor obtained from t-SVD of $\boldsymbol{\mathcal X}$. Specially, 
\begin{align}
\text{rank}_t\left( \boldsymbol{\mathcal X}\right) = \#\left\{i: \boldsymbol{\mathcal S}(i,i,:) \neq \boldsymbol 0\right\}.
\end{align}
\end{definition}

Lastly, we introduce the tensor spectral norm and condition number.

\begin{definition}\textbf{(Tensor spectral norm )} \cite{lu2018exact} The spectral norm of $\boldsymbol{\mathcal X} \in \mathbb R^{n_1 \times n_2 \times n_3}$ is defined as
\begin{align}
\left\|  \boldsymbol{\mathcal X} \right\| = \sigma_{\max} \left( \text{bcirc} \left( \boldsymbol{\mathcal X}\right)\right) = \max_{i \in  n_3} \sigma_{\max} \left(\overline{\boldsymbol X}^{(i)}\right), 
\end{align}
where $\sigma_{\max}(\boldsymbol X)$ denotes maximum singular value of $\boldsymbol X$.
\end{definition}

\begin{definition}\label{def-9} \textbf{(Tensor condition number)} The condition number of $\boldsymbol{\mathcal X} \in \mathbb R^{n_1 \times n_2 \times n_3}$ is defined as the condition number of $\text{bcirc}(\boldsymbol{\mathcal X})$ as
\begin{align}
\kappa \left( \boldsymbol{\mathcal X} \right) = \kappa \left( \text{bcirc}(\boldsymbol{\mathcal X}) \right) = \frac{ \sigma_{\max}\left( \text{bcirc}\left( \boldsymbol{\mathcal X} \right) \right) }{ \sigma_{\min}\left( \text{bcirc}\left(\boldsymbol{\mathcal X}\right) \right) },
\end{align}
where the $\sigma_{\min}\left( \text{bcirc}\left(\boldsymbol{\mathcal X}\right) \right)$  denotes the smallest nonzero singular value of $\text{bcirc}(\boldsymbol{\mathcal X})$. 
\end{definition} 

If the condition number $\kappa\left(\boldsymbol{\mathcal X}\right)$ is close to 1, the tensor $\boldsymbol{\mathcal X}$ is said to be well-conditioned. Conversely, if the condition number $\kappa\left(\boldsymbol{\mathcal X}\right)$ is large, then $\boldsymbol{\mathcal X}$ is deemed ill-conditioned \cite{tong2022scalings}. From Figure \ref{fig1} and Figure \ref{fig2}, recovering the ill-conditioned low-rank tensor $\boldsymbol{\mathcal X}^\star$ in the TCS problems is more challenging.

\section{Algorithms and Theoretical Results}
Given the local TCS model~\eqref{TCS_2}, a natural formulation is to minimize the fitting loss
\begin{align}
\hat f(\boldsymbol{\mathcal X}) := \sum_{i=1}^{n_2} \sum_{j=1}^m {\left( y_{ji} - \langle \boldsymbol{\mathcal A}_i(:,j,:), \boldsymbol{\mathcal X}(:,i,:)\rangle \right )^2. }  \label{ori_loss}  
\end{align}
under the constraint that the tubal rank of $\boldsymbol{\mathcal X}$ is at most equal to $r$. Based on the t-SVD, we can reparameterize the variable as $\boldsymbol{\mathcal X} = \boldsymbol{\mathcal U}*\boldsymbol{\mathcal V}$ to incorporate the low-rank constraint, with $\boldsymbol{\mathcal U} \in \mathbb R^{n_1 \times r \times n_3}$ satisfying the orthogonality constraint $\boldsymbol{\mathcal U}^c * \boldsymbol{\mathcal U} = \boldsymbol{\mathcal I}_r$ and $\boldsymbol{\mathcal V} \in \mathbb R^{r \times n_2 \times n_3}$. Overall, the optimization problem is written as:
\begin{equation}\label{rea_loss} 
    \begin{aligned}
        & \min_{\boldsymbol{\mathcal U},\boldsymbol{\mathcal V}} && \!\!f(\boldsymbol{\mathcal U},\boldsymbol{\mathcal V}) = \sum_{i=1}^{n_2} \sum_{j=1}^m {\left(y_{ji} - \langle \boldsymbol{\mathcal A}_i(:,j,:), \boldsymbol{\mathcal U} * \boldsymbol{\mathcal V}(:,i,:)\rangle \right)^2 } \\
        & \text{s.t.} && \boldsymbol{\mathcal U}^c * \boldsymbol{\mathcal U} = \boldsymbol{\mathcal I}_r.
    \end{aligned}
\end{equation}
Note that the reparameterization also significantly reduces the number of variables under small $r$, unlocking the potential of designing low-complexity algorithms.
However, as a tradeoff, it introduces nonconvexity through the tensor product in the objective and the orthogonality constraint.

This section proposes algorithms to solve \eqref{rea_loss} to the \emph{global minimum} despite nonconvexity. 
The approach consists of two stages. The first stage employs a spectral initialization to find an initial point $\boldsymbol{\mathcal U}_0$ that is sufficiently close to the minimizer, thus providing a warm start for the second stage. The second stage is based on a local search strategy that alternately optimizes $\boldsymbol{\mathcal U}$  and $\boldsymbol{\mathcal V}$ according to \eqref{rea_loss}. In the remainder of this section, we provide a detailed introduction to the two stages.

The following mild assumptions are made on $\boldsymbol{\mathcal{X}}^\star$ and the sensing operator $\left\{ \boldsymbol{\mathcal A}_i\right\}_{i=1}^{n_3}$ for obtaining our results.

\begin{assumption}\label{assum}
The ground truth $\boldsymbol{\mathcal X}^\star \in \mathbb R^{n_1 \times n_2 \times n_3}$ has tubal rank $r \ll \min \{n_1, n_2, n_3\}$. Its skinny t-SVD is $\boldsymbol{\mathcal X}^\star  = \boldsymbol{\mathcal U}^\star*\boldsymbol{\mathcal S}^\star*\left( \boldsymbol{\mathcal V}^\star \right)^c$ that\;$\boldsymbol{\mathcal U}^\star \in \mathbb R^{n_1 \times r \times n_3}, \boldsymbol{\mathcal S}^\star \in \mathbb R^{r\times r \times n_3}, \boldsymbol{\mathcal V}^\star \in \mathbb R^{r \times n_2 \times n_3}$ and $\boldsymbol{\mathcal Z}^\star = \boldsymbol{\mathcal S}^\star*\left( \boldsymbol{\mathcal V}^\star \right)^c$. There exists a finite constant $\mu$ such that $\mathop{\max}\limits_{i \in [n_2] }\left\|  \boldsymbol{\mathcal Z}^\star(:,i,:)\right\|_F \leq \mu \sqrt{\frac{r}{n_2}} \left\| \boldsymbol{\mathcal X}^\star \right\|$. 
\end{assumption}

This assumption is similar to the tensor incoherence condition in the low tubal rank tensor recovery literature \cite{zhang2016exact,lu2018exact,lu2020tensor,zhang2020low} and ensures that our problem remains well-posed. 

\begin{assumption}\label{assump:gaussian-measure}
Each sensing tensor $\boldsymbol{\mathcal A}_i \in \mathbb R^{n_1 \times m \times n_3}, \forall i \in [n_2]$, has i.i.d. standard Gaussian entries.
\end{assumption}

\subsection{Stage I: Truncated Spectral Initialization}

The idea of spectral-based initialization, which is used for providing a “warm start” within the basin of attraction for $\boldsymbol{\mathcal X}^\star$, has been extensively utilized in various non-convex low-rank matrix and tensor recovery problems \cite{cai2019nonconvex,liu2024low}.
Inspired by the truncation technique~\cite{chen2015solving,wang2017solving,vaswani2024efficient}, we design the following initialization method for the local sensing model \eqref{TCS_2}. Specifically, let
\begin{align}
\hat{ \boldsymbol{\mathcal X}}_0 (:,i,:) = \frac 1 {m_0} \sum_{j=1}^{m_0}  y_{ji} \boldsymbol{\mathcal A}_i(:,j,:)\cdot \mathbf{1}_{ \left\{\left| y_{ji}\right| \leq \sqrt \alpha \right\} }, \label{init_eq}
\end{align}
where $\mathbf{1}_{ \left\{\left| y_{ij}\right| \leq \sqrt \alpha \right\} }$ denotes an indicator function that is equal to $1$ when $\left|y_{ij}\right| \leq \sqrt \alpha $ and 0 otherwise. $m_0$ is the number of sensing matrix for per lateral slice. The $\alpha$ is the threshold in truncation and its formula is given in the subsequent theorem. The reason for truncation is that we can use tight sample complexity to bound the concentration of $\hat{\boldsymbol{\mathcal X}}_0$ on $\boldsymbol{\mathcal X}^\star$. Subsequently, we perform the QR decomposition:
\begin{align}
 \hat{ \boldsymbol{\mathcal X}}_0 =   {\boldsymbol{\mathcal Q}}_0 * {\boldsymbol{\mathcal R}}_0,
 \end{align}
and  initialize the orthogonal tensor $\boldsymbol{\mathcal U}_0 \in \mathbb R^{n_1 \times r \times n_3}$ to be the first $r$ lateral slices of ${\boldsymbol{\mathcal Q}}_0$, i.e.,
\begin{align}
\boldsymbol{\mathcal U}_0 = \boldsymbol{\mathcal Q}_0(:,1:r,:). \nonumber
\end{align}


The following measure defines the sine of the largest angle between tensor subspaces spanned by their lateral slices.
\begin{definition} \textbf{(Principal angle distance)}
For two orthogonal tensors $\boldsymbol{\mathcal A}_1, \boldsymbol{\mathcal A}_2 \in \mathbb R^{n_1 \times r \times n_3}$, the principal angle distance between $\boldsymbol{\mathcal A}_1$ and $\boldsymbol{\mathcal A}_2$ is defined as 
\begin{align}
\text{Dis}\left( \boldsymbol{\mathcal A}_1, \boldsymbol{\mathcal A}_2\right) = \left\| \left( \boldsymbol{\mathcal I}_r - \boldsymbol{\mathcal A}_1*\boldsymbol{\mathcal A}_1^c \right)* \boldsymbol{\mathcal A}_2\right\|.
\end{align}
\end{definition}
Based on this measure, we can prove the effectiveness of the proposed initialization method in the following theorem.
\begin{theorem} \label{thm3} 
Consider the TCS model \eqref{TCS_2} under Assumption \ref{assum} and~\ref{assump:gaussian-measure}. 
The initialization $\boldsymbol{\mathcal U}_0$ in Algorithm \ref{algo1} satisfies
\begin{align}
\text{Dis}\left( \boldsymbol{\mathcal U_0}, \boldsymbol{\mathcal U}^\star \right) \leq \frac{0.016}{\sqrt r \kappa^2 } \label{init}
\end{align}
with a probability at least 
\begin{align}
1 & - \exp\left(c_1\left(n_1+n_2\right)\log n_3 - \frac{c_2m_0n_2}{\kappa^8\mu^2n_3r^2} \right)  \nonumber \\
& \quad - \exp\left( -c_3 \frac{m_0n_2}{\kappa^8\mu^2 r^2}\right) , \label{pro_init}    
\end{align}
where $\kappa$ is the tensor condition number of $\boldsymbol{\mathcal X}^\star$ and $c_1,c_2,c_3$ are universal constants that are independent of model parameters.
\end{theorem}

Theorem \ref{thm3} immediately implies the following sample complexity for our spectral initialization scheme.
\begin{corollary}\label{coro1}
In the same setting as Theorem \ref{thm3}, if the sample size $m_0$ for initialization satisfies 
\begin{align}
 m_0 n_2 \gtrsim  \kappa^8 \mu^2r^2n_3(n_1+n_2)\log n_3, \label{sample-init} 
\end{align}
then \eqref{init} holds with  probability at least $1-\frac{1}{(n_1+n_3)^{10}}$. 
\end{corollary}

Compared to the total number of entries in $\boldsymbol{\mathcal U}^\star$ and $\boldsymbol{\mathcal Z}^\star$, which is $rn_3(n_1+n_2)$, Corollary~\ref{coro1} shows a good initialization can be achieved with a sample size only having an additional factor of $r\log n_3$ (modulus constants $\kappa,\mu$).

\subsection{Stage II: Local Search}

The second stage concerns iteratively refining the initial point $\boldsymbol{\mathcal U}_0$ computed by Stage I by local search. Since the objective $f$ in~\eqref{rea_loss} is bi-convex in $\boldsymbol{\mathcal U}$ and $ \boldsymbol{\mathcal V}$, based on this structure, we first propose the  \textsf{Alt-PGD-Min} algorithm that alternately updates these two factors. 

\subsubsection{The \textsf{Alt-PGD-Min} Algorithm.}
Let $\boldsymbol{\mathcal U}_t$ and $\boldsymbol{\mathcal V}_t$ be the values of $\boldsymbol{\mathcal U}$ and $\boldsymbol{\mathcal V}$ at iteration $t$.

1) \emph{Exact minimization for $\boldsymbol{\mathcal V}$}: Fixing $\boldsymbol{\mathcal U}_t$, the lateral slices of $\boldsymbol{\mathcal V}$ are decoupled in problem~\eqref{rea_loss}. Thus, we can update each lateral slice $\boldsymbol{\mathcal V}(:,i,:)$ in parallel by solving the following minimization problem
\begin{align}
\boldsymbol{\mathcal V}_{t}(i) & \in \mathop{\arg\min}\limits_{\boldsymbol{\mathcal B} \in \mathbb R^{r\times 1 \times n_3}} \sum_{j=1}^{m_c} { \left(  y_{ji} - \langle \boldsymbol{\mathcal U}_t^c *\boldsymbol{\mathcal A}_i(j),  \boldsymbol{\mathcal B} \rangle \right)^2 },\label{firt_V}
\end{align}
which can be reformulated as follows based on Definition \ref{t-t}
\begin{align}
\boldsymbol v_{t,i} \in \mathop{\arg\min}\limits_{ \boldsymbol v \in \mathbb R^{rn_3}} \left\|  \left( \text{bcirc}\left( \boldsymbol{\mathcal U}_t^c \right)\cdot \text{Unfold}\left(\boldsymbol{\mathcal A}_i \right)  \right)^c  \cdot \boldsymbol v - \boldsymbol y_i          \right\|^2. \label{second_v}
\end{align}
The problem \eqref{second_v} is a least squares problem. Thus, for every $i \in  [n_2]$, a closed-form solution can be derived as follows:
\begin{align}\label{least_solver}
& \boldsymbol H_{t,i} = \text{bcirc}\left( \boldsymbol{\mathcal U}_t^c \right)\cdot \text{Unfold}\left(\boldsymbol{\mathcal A}_i \right), \nonumber \\
& \boldsymbol v_{t,i}  = \left( \boldsymbol H_{t,i} \boldsymbol H_{t,i}^c  \right)^{-1} \boldsymbol H_{t,i}  \boldsymbol y_i, \nonumber \\
& \boldsymbol{\mathcal V}_{t}(i) = \text{Fold} \left( \boldsymbol v_{t,i} \right). 
\end{align}

2) \emph{Projected gradient descent for $\boldsymbol{\mathcal U}$}: 
Although for fixed $\boldsymbol{\mathcal V}_{t}$, $f$ is also convex in $\boldsymbol{\mathcal U}$, 
in pursuit of computation-efficient update, instead of performing exact minimization, we employ a first-order gradient descent step to update $\boldsymbol{\mathcal U}$ \cite{gu2024low}, followed by a projection step onto the orthogonality constraint set. Specifically, we first compute
\begin{align}
&\hat{\boldsymbol{\mathcal U}}_{t+1}  = \boldsymbol{\mathcal U}_t - \eta \sum_{i=1}^{n_2}\sum_{j=1}^{m_c} \left( y_{ji} - \langle \boldsymbol{\mathcal U}_t^c*\boldsymbol{\mathcal A}_i(j) , \boldsymbol{\mathcal V}_{t}(i)\rangle \right)  \nonumber \\
& ~\quad  \quad \quad \cdot \boldsymbol{\mathcal A}_i(j)* \left( \boldsymbol{\mathcal V}_{t}(i) \right)^c, \label{ori-U}
\end{align}
with step size $\eta>0$. Then we obtain a tensor $\hat{\boldsymbol{\mathcal Q}}_{t+1}$ by the QR decomposition $ \hat{\boldsymbol{\mathcal U}}_{t+1} = \hat{\boldsymbol{\mathcal Q}}_{t+1} * \hat{\boldsymbol{\mathcal{R}}}_{t+1} $. The updated of $\boldsymbol{\mathcal U}$ is given by
\begin{align}
\boldsymbol{\mathcal U}_{t+1} = \hat{\boldsymbol{\mathcal Q}}_{t+1}(:,1:r,:) \label{qr}.
\end{align}

\begin{algorithm}[!h]
\caption{ \textsf{Alt-PGD-Min/Alt-ScalePGD-Min} }
\begin{algorithmic}[1] 
\Require 
Number of iteration $T$, total sensing tensors with sample splitting $\{\{ \boldsymbol{\mathcal A}_i^{(k)}\}_{k=1}^{2T+1}\}_{i=1}^{n_2}$, corresponding sample-splitting local measurements $\{ \{ \boldsymbol y_i^{(k)} \}_{k=1}^{2T+1}\}_{i=1}^{n_2}$, $r,\kappa,\mu,n_2$, step size $\eta$.

\For{$t = 0, 1, \ldots, T-1$} \\
\emph{ $\triangleright \; { Update \; \boldsymbol{\mathcal U} }$ }
\If{$t=0$} \quad  $\triangleright$ \emph{ Initialization }
    \State Set $ \boldsymbol{\mathcal A}_i = \boldsymbol{\mathcal A}_i^{(2T)}, \boldsymbol y_{i} = \boldsymbol y_i^{(2T)}, \forall i \in [n_2]$,

    \State Calculate $\alpha = C\frac{\kappa^2\mu^2\sum_{i=1}^{n_2}\sum_{j=1}^m y_{ji}}{mn_2}$,

    \State Set $ \boldsymbol{\mathcal A}_i = \boldsymbol{\mathcal A}_i^{(2T+1)}, \boldsymbol y_{i} = \boldsymbol y_i^{(2T+1)}, \forall i \in [n_2]$,
    
    \State Construct $\hat{\boldsymbol{\mathcal X}}_0$ as \eqref{init_eq},

   \State Conduct QR decomposition $\hat{\boldsymbol{\mathcal X}}_0 = \hat{\boldsymbol{\mathcal Q}}_{0}*\hat{\boldsymbol{\mathcal R}}_{0}$, 
   
    \State Initialize $\boldsymbol{\mathcal U}_0$ by top-$r$ lateral slices of $\hat{\boldsymbol{\mathcal Q}}_{0}$.

\Else


         \For{$i = 1, 2, \ldots, n_2$}
 \State $ \boldsymbol{\mathcal A}_i = \boldsymbol{\mathcal A}_i^{(T+t)}, \boldsymbol y_{i} = \boldsymbol y_i^{(T+t)}, \forall i \in [n_2]$,

\State $\boldsymbol H_{t-1,i} = \text{bcirc}(\boldsymbol{\mathcal U}^c_{t-1})\cdot \text{Unfold}(\boldsymbol{\mathcal A}_i)$,
\State $\boldsymbol b_{t-1,i} = \boldsymbol H_{t-1,i}^c \text{Unfold}(\boldsymbol{\mathcal V}_{t-1}(i)) - \boldsymbol y_i$,
\State $\boldsymbol{\mathcal T}_{t-1}(:,i,:) = \sum_{j=1}^m (\boldsymbol b_{t-1,i})_j\boldsymbol{\mathcal A}_i(:,j,:)$,
\EndFor  

/***\textsf{Alt-PGD-Min}***/ 
\State $\hat{\boldsymbol{\mathcal U}}_{t}  = \boldsymbol{\mathcal U}_t - \eta \boldsymbol{\mathcal T}_{t-1}*\boldsymbol{\mathcal V}_{t-1}^c $,

{/***\textsf{Alt-ScalePGD-Min}***/}
\State $\hat{\boldsymbol{\mathcal U}}_{t}  = \boldsymbol{\mathcal U}_t - \eta \boldsymbol{\mathcal T}_{t-1}*\boldsymbol{\mathcal V}_{t-1}^c {*(\boldsymbol{\mathcal V}_{t-1}*\boldsymbol{\mathcal V}_{t-1}^c)^{-1}}$, 

\State Calculate $\boldsymbol{\mathcal U}_t$  as \eqref{qr} by QR decomposition,
\EndIf  

\noindent \emph{ $\triangleright \; { Update \; \boldsymbol{\mathcal V} }$ }

\For{$i = 1, 2, \ldots, n_2$}

\State $ \boldsymbol{\mathcal A}_i = \boldsymbol{\mathcal A}_i^{(t+1)}, \boldsymbol y_{i} = \boldsymbol y_i^{(t+1)}, \forall i \in [n_2]$,

\State $\boldsymbol{H}_{t,i} = \text{bcirc}(\boldsymbol{\mathcal U}^c_t)\cdot \text{Unfold}(\boldsymbol{\mathcal A}_i)$,

\State $ \boldsymbol v_{t,i} = \left( \boldsymbol{H}_{t,i} \boldsymbol{H}_{t,i}^c \right)^{-1} \boldsymbol{H}_{t,i} \boldsymbol y_i $,

\State $\boldsymbol{\mathcal V}_t(i) = \text{Fold}(\boldsymbol v_{t,i})$.
\EndFor

\noindent \emph{ $\triangleright \; { Update \; \boldsymbol{\mathcal X} }$ }
 \State $\boldsymbol{\mathcal X}_t = \boldsymbol{\mathcal U}_t * \boldsymbol{\mathcal V}_t $, 

\EndFor 
\Ensure
Recover tensor $\boldsymbol{\mathcal X}_{T-1}$.
\end{algorithmic}
\label{algo1}
\end{algorithm}

The complete \textsf{Alt-PGD-Min} algorithm is described in Algorithm \ref{algo1}. It is worth mentioning that we use the sample-splitting technique in Algorithm \ref{algo1}, where we divide the total samples and measurements pairs of each slice into $2T+1$ groups as $\{ \boldsymbol{\mathcal A}_i^{(k)} \}_{k=1}^{2T+1}$. The last two groups are used for initialization that the sample size for each lateral slice is $m_0$, and  \textsf{Alt-PGD-Min} draws two fresh groups of samples from the remaining groups per iteration, where the sample size for each lateral slice is $m_c$. The splitting strategy ensures statistical independence of the samples across iterations, which is a key component to simplifying the convergence analysis and has been used in various low matrix and tensor learning algorithms \cite{hardt2014fast,jain2015fast,ding2020leave,cai2022provable}. 

The computational complexity per iteration in \textsf{Alt-PGD-Min} for updating $\boldsymbol{\mathcal U}$  and $\boldsymbol{\mathcal V}$ is    
$\mathcal O\left( n_1n_3m_cr + n_3m_cr + n_1n_2n_3r + n_1n_3r^2        \right)  $ and $\mathcal O \left(   n_1n_2n_3m_cr + m_c(rn_3)^2n_2 + (n_3r)^3 n_2     \right)$, respectively.





\begin{theorem}\label{thm4} 
In the same setting as Theorem \ref{thm3}, if the initialization $\boldsymbol{\mathcal U}_0$ satisfies \eqref{init} and $\eta = \frac{c_\eta }{{m_c} \left\| \boldsymbol{\mathcal X}^\star \right\|^2} $ with $c_\eta \leq 0.9$, then the iterates generated by  \textsf{Alt-PGD-Min} satisfies 
\begin{align}
& \text{Dis}(\boldsymbol{\mathcal U}_t,\boldsymbol{\mathcal U}^\star) \leq \left( 1 -\frac{0.84c_\eta}{\kappa^2}\right)^t \cdot \text{Dis}(\boldsymbol{\mathcal U}_0,\boldsymbol{\mathcal U}^\star), \nonumber \\
& \left\| \boldsymbol{\mathcal X}_t(i)- \boldsymbol{\mathcal X}^\star(i) \right\|_F \leq 1.4  \text{Dis}(\boldsymbol{\mathcal U}_t,\boldsymbol{\mathcal U}^\star) \cdot \left\| \boldsymbol{\mathcal X}^\star(i) \right\|_F, \forall i \in [n_2] \label{con-rate}
\end{align}
with  a probability  at least
\begin{align}
1 &- \exp \left( c_4(n_1+r)\log n_3 - \frac{c_5m_cn_2}{\kappa^4\mu^2n_3r}\right) \nonumber \\
& - \exp \left(\log n_2 + r \log n_3 - c_6 m_c \right), \label{GD-whp}
\end{align}  
where $c_4,c_5,c_6$ are universal positive constants independent from model parameters.
\end{theorem}

Theorem \ref{thm4} shows that even \eqref{rea_loss} is non-convex, if the initialization 
$\boldsymbol{\mathcal U}_0$ is sufficiently close to the minimizer, the iterations of \textsf{Alt-PGD-Min} will converge at a linear rate to $\boldsymbol{\mathcal X}^\star$. 


\begin{corollary} \label{coro2}
In the same setting as Theorem \ref{thm4}, if the $\eta = \frac{0.8 }{m_c \left\| \boldsymbol{\mathcal X}^\star\right\|}$ and the sample size $m_c$ satisfies
\begin{align}
m_c n_2 \gtrsim \kappa^4\mu^2rn_1 n_3 \log n_3, \; \text{and}\; m_c \gtrsim \max\left\{\log n_2, r \log n_3 \right\}, \label{GD-Sample}
\end{align}
then it takes $T = c_7 \kappa^2 \log \frac{1}{\epsilon}$
 iterations for \textsf{Alt-PGD-Min} to achieve  $\epsilon$-accuracy recovery, i.e.,
\begin{align}
& \text{Dis}\left( \boldsymbol{\mathcal U}_T , \boldsymbol{\mathcal U}^\star \right) \leq \epsilon, \nonumber \\
& \left\| \boldsymbol{\mathcal X}_T(i)- \boldsymbol{\mathcal X}^\star(i) \right\|_F \leq 1.4\epsilon \left\| \boldsymbol{\mathcal X}^\star(i) \right\|_F, \forall i \in [n_2] \label{GD-accu}
\end{align}
with  probability at least $1-\frac{1}{\left(n_1+r\right)^{10}}$.
\end{corollary}

The sample complexity given by \eqref{GD-Sample} scales linearly with $r$, showing improved dependence on $r$ compared to that for the initialization Stage I. When the number of lateral slices is large enough such that $n_2 \gtrsim \kappa^4\mu^2n_1\log n_3$, the order of $m_c$ becomes $\mathcal O(rn_3)$, which is significantly smaller than the size of lateral slice $n_1n_3$ as $r\ll \min(n_1,n_2)$. 

Combining the results of Corollary \ref{coro1} and Corollary \ref{coro2}, we can immediately conclude the overall sample complexity of \textsf{Alt-PGD-Min} as follows.
\begin{corollary} \label{coro3} 
Consider the TCS model \eqref{TCS_2} under Assumption \ref{assum} and~\ref{assump:gaussian-measure}. 
For \textsf{Alt-PGD-Min} to achieve $\epsilon$-accuracy recovery as described by \eqref{GD-accu} with high probability at least $1 - \frac{2}{(n_1+r)^{10}}$, the total sample complexity $m$ for each lateral slice is
\begin{align}
 m n_2 \gtrsim \kappa^6\mu^2rn_3\log n_3 \left( \kappa^2r\left(n_1 + n_2 \right) + n_1 \log \frac{1}{\epsilon}\right)  \label{total_sam}  
\end{align}
and $m \gtrsim \kappa^2 \max \left\{ \log n_2, r \log n_3 \right\}\log \frac{1}{\epsilon}$.
\end{corollary}

The total sample complexity in \eqref{total_sam} comprises two parts: one is from initialization and the other is from iterative refinements. The dependency of the sample complexity on recovery accuracy $\epsilon$ is because of the sample splitting strategy introduced in the algorithm, which is common in all the analyses where such a technique is adopted \cite{jain2013low, hardt2014fast,ding2020leave,vaswani2024efficient}.

\subsubsection{\textsf{Alt-ScalePGD-Min}: Acceleration by Preconditioning.}
To mitigate the influence of large $\kappa$ and improve the algorithm efficiency, especially for ill-conditioned problems, we propose to accelerate by pre-conditioning the gradient step that is sensitive to $\kappa$. 
Recall \eqref{ori-U} and let $\boldsymbol b_{t,i} := \boldsymbol H_{t,i}^c \boldsymbol v_{t,i} - \boldsymbol y_i, i \in [n_2]$ and $\boldsymbol{\mathcal T}_{t}(:,i,:) := \sum_{j=1}^{m_c} (\boldsymbol b_{t,i})_j \cdot \boldsymbol{\mathcal A}_i(:,j,:)$, we rewrite the updating step as 
\begin{align}
\hat{\boldsymbol{\mathcal U}}_{t+1} = \boldsymbol{\mathcal U}_t - \eta \boldsymbol{\mathcal T}_{t} * \boldsymbol{\mathcal V}^c_t. \label{com-gd}
\end{align}
Comparing to the  gradient step \eqref{com-gd}, the key difference of \textsf{Alt-ScalePGD-Min} is that it preconditions the search direction of $\boldsymbol{\mathcal U}_t$ by inverse of $\boldsymbol{\mathcal V}_t* \boldsymbol{\mathcal V}_t^c$, i.e.,
\begin{align}
\hat{\boldsymbol{\mathcal U}}_{t+1} = \boldsymbol{\mathcal U}_t - \eta \boldsymbol{\mathcal T}_t*\boldsymbol{\mathcal V}_t^c*\left( \boldsymbol{\mathcal V}_t* \boldsymbol{\mathcal V}_t^c \right)^{-1}. \label{U_inverse}
\end{align}
Note that the inverse in \eqref{U_inverse} is easy to compute because 
%
the related tensor has a size of $r\times r \times n_3$, significantly smaller than the dimension of the tensor factors. Thus, each iteration of scaled GD incurs minor additional complexity $\mathcal O \left(  n_1 r^2 n_3 + r^3 n_3 \right)$ than the GD of \textsf{Alt-PGD-Min} in \eqref{com-gd}. 

The convergence of \textsf{Alt-ScalePGD-Min} is given in the following theorem.

\begin{theorem}\label{thm5}
Consider the TCS model \eqref{TCS_2} under Assumption \ref{assum} and~\ref{assump:gaussian-measure}. Let the step size $\eta= \frac{c_\eta}{m_c}$ and all other parameters are set the same as Theorem \ref{thm4}, then the iterates generated by  \textsf{Alt-ScalePGD-Min} satisfy
\begin{align}
\text{Dis}\left(\boldsymbol{\mathcal U}_t, \boldsymbol{\mathcal U}^\star \right) \leq \left( 1 - 0.89 c_\eta \right)^t \cdot \text{Dis}\left(\boldsymbol{\mathcal U}_0, \boldsymbol{\mathcal U}^\star \right)
\end{align}
with at least the same probability as \eqref{GD-whp}.
\end{theorem}

\begin{corollary}\label{coro4}
In the same setting as Theorem \ref{thm5}, if $\eta = \frac{0.8}{m_c}$ and sample size for each lateral slice $m_c$ satisfies \eqref{GD-Sample}, then to obtain the $\epsilon$-accuracy recovery as described by \eqref{GD-accu}, the iteration complexity for Alt-Scale-GD is 
\begin{align}
T = c_7 \log \frac{1}{\epsilon}.
\end{align}.
\end{corollary}
\vspace{-5mm}
Theorem~\ref{thm5} and Corollary~\ref{coro4} show \textsf{Alt-ScalePGD-Min} converges linearly at a rate that is \emph{independent of} the condition number $\kappa$, significantly improving over the $\mathcal O\left( \kappa^2 \log \frac 1\epsilon \right)$ complexity of \textsf{Alt-PGD-Min}.

\begin{corollary}
Consider the TCS model \eqref{TCS_2} under Assumption \ref{assum} and~\ref{assump:gaussian-measure}. The total sample complexity $m$ for each lateral slice to achieve $\epsilon$-accuracy recovery as \eqref{GD-accu} with high probability at least $1-\frac{2}{(n_1+r)^{10}}$ is
\begin{align}
mn_2 \gtrsim \kappa^4 \mu^2rn_3 \log n_3 \left( \kappa^4r(n_1+n_2) + n_1 \log \frac{1}{\epsilon}\right) \label{scale-sam}
\end{align}
and $m \gtrsim \max \left\{ \log n_2, r \log n_3 \right\}\log \frac{1} {\epsilon}$.
\end{corollary}
 Due to the same initialization, the first term in \eqref{total_sam} and \eqref{scale-sam} are the same. However, \textsf{Alt-ScalePGD-Min} improves over \textsf{Alt-PGD-Min} by a factor of $\kappa^2$ in the second term due to improved convergence rate. When the recovery accuracy is sufficiently high such that the second term dominates the first, i,e., $\log \frac{1}{\epsilon} \gtrsim \kappa^4r(1+\frac{n_2}{n_1})$, the total sample complexity of \textsf{Alt-ScalePGD-Min} is $\mathcal O \left( \kappa^4\mu^2rn_1n_3 \log n_3 \log \frac{1}{\epsilon} \right)$, significantly improving upon the $\mathcal O \left( \kappa^6\mu^2rn_1n_3 \log n_3 \log \frac{1}{\epsilon} \right)$ of \textsf{Alt-PGD-Min} for large $\kappa$. 

\section{Experiments}
We evaluate our proposed methods on both synthetic and real-world data. Since model \eqref{TCS_2} has not been studied in existing works, we can only compare our method with the low-rank matrix column-wise CS method (LRcCS) \cite{nayer2022fast}, which is closest to our TCS model in application. To compare with this method, we conduct $\text{Unfold}(\boldsymbol{\mathcal X}^\star)$ operation which reshapes each lateral slice into a vector. The performance of algorithms is measured by the relative recovery error $\frac{\left\| \boldsymbol{\mathcal X}_t - \boldsymbol{\mathcal X}^\star \right\|_F}{\left\| \boldsymbol{\mathcal X}^\star \right\|_F}$, which is plotted concerning the iteration number.


\subsection{Synthetic Data}
We generate a synthetic tensor with $n_1 = n_3 = 20, n_2 = 400, r=4$, of which the details are in the Appendix. The sample sizes are  $m_0=200$ and $m_c=100$ without sample splitting. We test three algorithms with the same step size and sample sizes under different $\kappa=1,2,4$. The results are plotted in Figure \ref{fig1}, which shows that both our proposed methods have linear convergence rates while LRcCS fails in all cases. The convergence rate of Alt-PGD-Min becomes slower with increasing $\kappa$  while  Alt-ScalePGD-Min converges with independence on $\kappa$, with all curves overlaying on each other. 
\begin{figure}[htbp]
 \centering
\includegraphics[width=6cm,height=4cm]{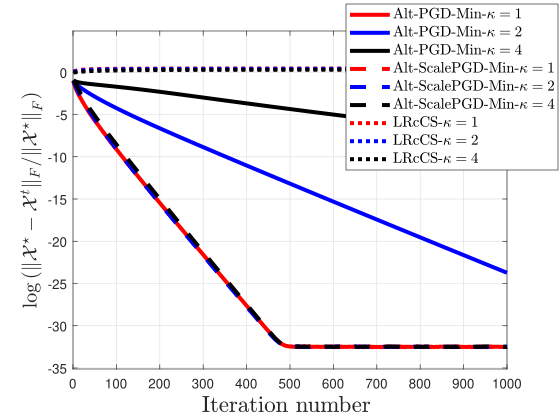}
\caption{Spectral initialization with $m_0=200,m_c=100$.}\label{fig1}
\end{figure}

In the second setting, the data generation is the same as the first. However, we validate the performance under random initialization, which has i.i.d. standard Gaussian entries. Without good initialization, we increase the $m_c$ slightly to $m_c=120$. The results depicted in Figure \ref{fig2} show that both of the proposed algorithms converge with linear rates after a small number of iterations 
during the initial phase while LRcCS still does not work. With larger $\kappa$, Alt-PGD-Min slows down significantly while the convergence speed of Alt-ScalePGD-Min remains almost the same with almost negligible initial phases.

\begin{figure}[htbp]
 \centering
\includegraphics[width=6cm,height=4cm]{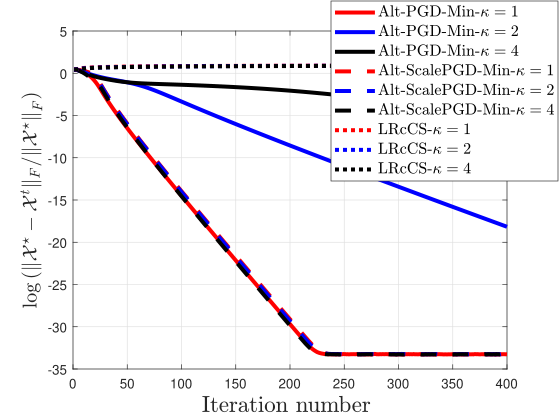}
\caption{Random initialization with $m_c=120$.}\label{fig2}
\end{figure}
\subsection{Video Compressed Sensing}
We test the proposed TCS model \eqref{TCS_2} in the plane video sequence (only approximately low tubal rank) that has been used in previous work such as \cite{nayer2019phaseless}. It has $105$ frames and each frame has been resized into $48\times 64$. We use the same samples with size $m=1600$ for initialization and iteration. We set $r=10$ for three methods. For each method, we tune the step size to guarantee it converges and achieves performance as good as possible. The visual and quantitative comparison results are shown in Figure \ref{fig3} and Figure \ref{fig4}, respectively. 
\begin{figure}[htbp]
\subfigure[ \scriptsize Original]{
\begin{minipage}[t]{0.5\linewidth}
\centering
\includegraphics[scale=0.85]{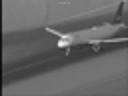}
\end{minipage}%
}%
\subfigure[\scriptsize LRcCS (PSNR = 33.69)]{
\begin{minipage}[t]{0.5\linewidth}
\centering
\includegraphics[scale=0.85]{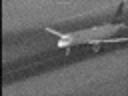}
\end{minipage}%
}%

\subfigure[\scriptsize Alt-PGD-Min (PSNR = 30.55)]{
\begin{minipage}[t]{0.5\linewidth}
\centering
\includegraphics[scale=0.85]{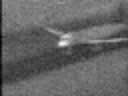}
\end{minipage}%
}%
\subfigure[ \scriptsize Alt-ScalePGD-Min (PSNR = 41.18)]{
\begin{minipage}[t]{0.5\linewidth}
\centering
\includegraphics[scale=0.85]{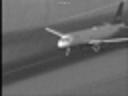}
\end{minipage}%
}%
\caption{Visualization of frame-7 in recovered videos.}
\label{fig3}
\end{figure}
Alt-ScalePGD-Min has the fastest convergence rate with the best recovery performance (Figure \ref{fig3}. (a) selected at 20-th iteration). Both Alt-PGD-Min and LRcCS converge very slowly with unsatisfied performance (Figure \ref{fig3}. (c) and (d) are selected by running 8000 iterations). This is because the video has very large matrix and tensor condition numbers that result in slow convergence rates.

We can observe that while Alt-PGD-Min outperforms LRcCS in synthetic data experiments, it performs worse in video compressive sensing. This discrepancy is because the synthetic data is generated to be exactly low-tubal-rank with a controlled tensor condition number. In contrast, the video data is only approximately low tubal rank and has a significantly larger tensor condition number compared to its reshaped matrix condition number. As a result, given the same number of iterations, Alt-PGD-Min performs worse than LRcCS in the video CS scenario.

\begin{figure}[htbp]
\subfigure[ PSNR for each frame ]{
\begin{minipage}[t]{0.47\linewidth}
\centering
\includegraphics[scale=0.28]{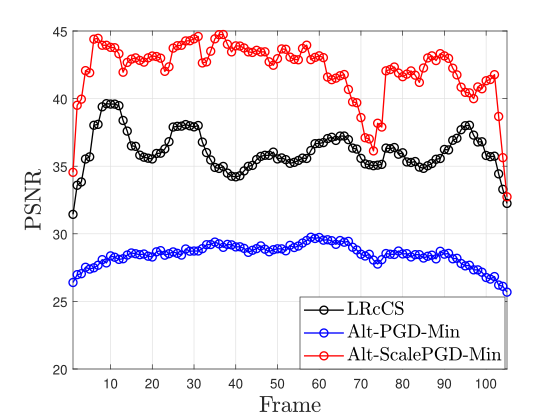}
\end{minipage}%
}%
\subfigure[Convergence comparison]{
\begin{minipage}[t]{0.47\linewidth}
\centering
\includegraphics[scale=0.28]{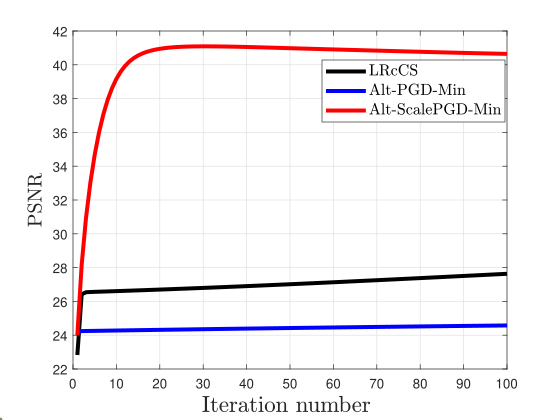}
\end{minipage}%
} 
\caption{Quantitative comparison}
\label{fig4}
\end{figure}


\section{Conclusion}
We introduced a novel local TCS model for low-tubal-rank tensors and developed two algorithms to solve this model with theoretical guarantees. The numerical results demonstrated the effectiveness of our method. There are several problems worth studying in the future. For instance, it would be valuable to investigate whether the current theoretical guarantees can be established without the sample-splitting technique. Additionally, exploring the possibility of achieving global convergence under random initialization, without spectral initialization, remains challenging. Lastly, improving the complexity dependence on the tensor condition number is another interesting direction for future work. 

\bigskip

\section*{Acknowledgments}
This work was partially supported by the Youth Program 62106211 of the National Natural Science Foundation of China.

\bibliography{aaai25}


\onecolumn
\appendix
\begin{center}
    \huge \textbf{Supplementary Materials}
\end{center}

The first part of this supplementary material gives more detailed discussions of our results. Then we provide proofs of the main theorems in the main paper. Specifically, Section \ref{sec1}, Section \ref{sec2} and Section \ref{sec3} give the proofs of Theorem 3, Theorem 4, and Theorem 5, respectively. The technical lemmas are provided in Section \ref{sec4}. Finally, we describe synthetic data generation in Section \ref{sec5} and give more experiments on video data and MRI data to evaluate the effectiveness of proposed methods on the introduced local tensor compressed sensing model.

\section{Discussion of our results}
\subsection{More comparisons with related works}
Because the local TCS for the low tubal rank tensor is a new model and has not been investigated yet, we can only compare our work with existing global low tubal rank tensor compressed sensing results. We compare with works TNN \cite{lu2018exact}, IR-t-TNN \cite{wang2021generalized}, and FGD \cite{liu2024low}. TNN is the convex method that applies the ADMM method to minimize the tensor nuclear norm to achieve recovery. IR-t-TNN utilizes a non-convex penalty and proposes iteratively reweighted t-TNN to solve. FGD uses a similar non-convex low tubal rank decomposition as ours and uses the gradient descent method to obtain the tensor factor. However, then only consider the symmetric positive semi-definite (T-PSD) ground truth tensor $\boldsymbol{\mathcal X}^\star = \boldsymbol{\mathcal F}*\boldsymbol{\mathcal F}^\star$, while we consider then general ground truth tensor with asymmetric decomposition as $\boldsymbol{\mathcal X}^\star = \boldsymbol{\mathcal U}*\boldsymbol{\mathcal V}^\star$. In addition, compared with these works, our method is the first to achieve the tensor condition number independent linear convergence rate.   

Specifically, we compare that for achieving $\epsilon$-accuracy recovery, the convergence rate, the dominant computational complexity per iteration, and the total sample complexity of these works and our methods in Table \ref{tabc}.
\begin{table}[h!] 
    \centering
    \caption{Comparison of t-SVD based methods for TCS}
    \resizebox{\textwidth}{!}{ 
    \begin{tabular}{lcccccc}
        \toprule
        \textbf{Methods} & \textbf{Measurement} & \textbf{Convexity} & \textbf{Convergence Rate} & \textbf{Computational complexity} & \textbf{Sample Complexity (number of sensing tensor)} \\
        \midrule
        TNN [1]         & global & convex  & sub-linear & $\mathcal O\left(n^3n_3+n^2n_3\log n_3 \right)$ & $\mathcal O\left(r(2n-r)n_3 \right)$   \\
        IR-t-TNN [2]    & global & non-convex & sub-linear & $\mathcal O\left(n^3n_3+n^2n_3\log n_3 \right)$ & $\times$ \\
        FGD [3]         & global & non-convex & $\kappa \log \frac{1}{\epsilon}$ & $\mathcal{O}\left( rn^2n_3 + rnn_3 \log n_3 \right)$ & $\times$ \\
         \textbf{Alt-PGD-Min}            & local  & non-convex & $\kappa^2 \log \frac{1}{\epsilon}$ & $\mathcal{O} \left( \left(n^2n_3^2r^2+nn_3^3r^3\right)\log n_3 \right)$ & $\mathcal O\left( \kappa^6rn_3\log n_3 \left( \kappa^2r\left(n_1 + n_2 \right) + n_1 \log \frac{1}{\epsilon}\right) \right)$ \\
         \textbf{Alt-ScalePGD-Min}            & local  & non-convex & $ \log \frac{1}{\epsilon}$ & $\mathcal{O} \left( \left(n^2n_3^2r^2+nn_3^3r^3\right)\log n_3 \right)$ & $\mathcal O\left( \kappa^4 r n_3\log n_3 \left( \kappa^2r\left(n_1 + n_2 \right) + n_1 \log \frac{1}{\epsilon}\right) \right)$ \\
        \bottomrule
    \end{tabular}
    }
\label{tabc}
\end{table}

It can be observed that only our work considers the low tubal rank tensor recovery from local measurements. In addition, our method is also the unique work that can achieve a linear convergence rate and simultaneously provide the sample complexity guarantee. For the sample complexity, if the number of lateral slices is sufficiently large such that $n_2 \gg n_1$, then our sample complexity becomes $\mathcal O(r^2n_3\log n_3)$ which is much smaller than $\mathcal O (r(n_1+n_2-r)n_3)$ of TNN under TCS.

\subsection{Intuition of sample complexity under local TCS}
Based on our local measurement sensing model, each lateral slice of ground truth tensor $\boldsymbol{\mathcal X}^\star \in \mathbb R^{n_1\times n_2 \times n_3}$ can be represented by $ \boldsymbol{\mathcal X}^\star(:,i,:) = \boldsymbol{\mathcal U}^\star * \boldsymbol{\mathcal V}^\star(:,i,:)$ where $\boldsymbol{\mathcal U}^\star \in \mathbb R^{n_1 \times r \times n_3},  \boldsymbol{\mathcal V}^\star(:,i,:) \in \mathbb R^{r \times 1\times n_3}$.  
The degree of freedom for learning $\boldsymbol{\mathcal U}^\star$ is $n_1n_3r$, which is shared by all $n_2$ lateral slices. Based on our uniform Assumption 1, the average samples for each lateral slice that can guarantee learning the $\boldsymbol{\mathcal U}^\star$ simultaneously is at least order $\mathcal O\left(\frac{n_1n_3r}{n_2} \right)$ which has a linear speedup with the number of lateral slices $n_2$. In addition, to learn its own $\boldsymbol{\mathcal V}^\star(:,i,:)$ which has a degree of freedom $rn_3$, each lateral slice at least needs $\mathcal O \left( rn_3 \right)$ order samples. Thus, intuitively, the information theoretical total sample complexity for each lateral slice is at least order $\mathcal O\left( \frac{n_1n_3r}{n_2} + rn_3 \right)$. 
Compared with our sample complexity $\mathcal O\left(  \frac{n_1n_3r^2}{n_2} \log n_3  + r^2n_3 \log n_3 \right)$ in (30) of our paper, we only have additional multiplicative term $r\log n_3$, which is very small for considered low
tubal rank tensor.

\section{Proof of Theorem 3}\label{sec1}
The proof of this theorem is based on the t-SVD version of Wedin's $\sin \theta$ theory and the concentration of the initialization obtained by the truncated spectral method around the ground truth.  

 To proceed, we first introduce two lemmas. Lemma \ref{lemma1} provides the expression for the estimate  $\boldsymbol{\mathcal X}_0$ in terms of the conditional expectation $\mathbb E \left[ \boldsymbol{\mathcal X}_0 |\alpha \right]$. This expression shows that $\mathbb E \left[ \boldsymbol{\mathcal X}_0 |\alpha \right]$ can be represented by the ground truth $\boldsymbol{\mathcal X}^\star$, where the coefficients are determined by the truncated threshold.  

\begin{lemma}\label{lemma1}
The initialization tensor $\boldsymbol{\mathcal X}_0$ in Algorithm 1 has property that 
\begin{align}
\mathbb E \left[ \boldsymbol{\mathcal X}_0 | \alpha \right] = \boldsymbol{\mathcal X}^\star * \boldsymbol{\mathcal D}, \label{for_ex_x0}
\end{align}
where the $\boldsymbol{\mathcal D} \in \mathbb R^{n_2 \times n_2 \times n_3}$ is defined as 
\begin{align}
&\boldsymbol{\mathcal D}(:,:,k) = \boldsymbol 0_{n_2}, \forall k = 2,\cdots, n_3\nonumber \\
&\boldsymbol{\mathcal D}(i,i,1) = \mathbb E\left[ \xi^2 \boldsymbol 1_{ \left\{ \left\| \boldsymbol{\mathcal X}^\star(i)\right\|^2_F\xi^2 \leq \alpha \right\}} \right], \forall i \in [n_2],
\end{align}
where the $\xi$ is the standard random Gaussian variable. 
\end{lemma}

\begin{proof}
We first can rewrite $\boldsymbol{\mathcal X}^\star(i)$ as $\boldsymbol{\mathcal X}^\star(i) = \left\| \boldsymbol{\mathcal X}^\star(i) \right\|_F \boldsymbol{\mathcal Q}_i * \mathring{\mathfrak {e}}_1$, where $\boldsymbol{\mathcal Q}_i \in \mathbb R^{n_1 \times n_1 \times n_3}$ is an orthogonal tensor that the first lateral slice $\boldsymbol{\mathcal Q}{(1)}$ has same direction as $\boldsymbol{\mathcal X}^\star(i)$. $\mathring{\mathfrak {e}}_1$ is the tensor column basis, which is a tensor of size $n_1 \times 1 \times n_3$ with its $(i,1,1)$-th entry equaling 1 and the rest equaling 0 \citep{lu2020tensor}. Then $\forall i \in [n_2]$, based on truncated spectral initialization, we have
\begin{align}
 \mathbb E \left[ \boldsymbol{\mathcal X}_0(i) | \alpha \right] & = \mathbb E \left[ \frac 1 m \sum_{j=1}^m \langle \boldsymbol{\mathcal A}_i(j), \boldsymbol{\mathcal X}^\star(i) \rangle \boldsymbol{\mathcal A}_i(j) \boldsymbol 1_{ \left\{ \ |y_{ji}|\leq \sqrt \alpha \right\} } \Biggm| \alpha \right] \nonumber \\
& \overset{(a)}{=} \mathbb E \left[ \frac 1 m \sum_{j=1}^m  \boldsymbol{\mathcal Q}_i * \boldsymbol{\mathcal Q}_i^c * \boldsymbol{\mathcal A}_i(j) \langle \boldsymbol{\mathcal A}_i(j), \boldsymbol{\mathcal X}^\star(i) \rangle \boldsymbol 1_{ \left\{ \ |y_{ji}|\leq \sqrt \alpha \right\} } \Biggm| \alpha \right] \nonumber \\ 
& \overset{(b)}{=} \mathbb E \left[ \frac 1 m \sum_{j=1}^m  \boldsymbol{\mathcal Q}_i * \boldsymbol{\mathcal Q}_i^c * \boldsymbol{\mathcal A}_i(j) \langle \boldsymbol{\mathcal A}_i(j), \boldsymbol{\mathcal Q}_i * \mathring{\mathfrak {e}}_1 \rangle \left\| \boldsymbol{\mathcal X}^\star(i)\right\|_F \boldsymbol 1_{ \left\{ \left\| \boldsymbol{\mathcal X}^\star(i)\right\|_F | \langle  \boldsymbol{\mathcal A}_i(j), \boldsymbol{\mathcal Q}_i * \mathring{\mathfrak {e}}_1  \rangle | \leq \sqrt \alpha    \right\}} \Biggm | \alpha \right] \nonumber \\
& \overset{(c)}{=} \mathbb E \left[ \frac 1 m \sum_{j=1}^m \boldsymbol{\mathcal Q}_i  * \hat{\boldsymbol{\mathcal A_i}}(j) \langle \hat{\boldsymbol{\mathcal A_i}}(j) , \mathring{\mathfrak {e}}_1\rangle \left\| \boldsymbol{\mathcal X}^\star(i)\right\|_F  \boldsymbol 1_{ \left\{ | \langle \hat{\boldsymbol{\mathcal A_i}}(j) , \mathring{\mathfrak {e}}_1\rangle | \leq \frac {\sqrt \alpha   }{ \left\| \boldsymbol{\mathcal X}^\star(i)\right\|_F  }\right\}} \Biggm | \alpha \right] \nonumber \\
& \overset{(d)}{=} \mathbb E \left[ \frac 1 m \sum_{j=1}^m \boldsymbol{\mathcal Q}_i  * \hat{\boldsymbol{\mathcal A_i}}(j) \hat{\boldsymbol{\mathcal A_i}}(j)(1,j,1) \left\| \boldsymbol{\mathcal X}^\star(i)\right\|_F  \boldsymbol 1_{ \left\{ | \hat{\boldsymbol{\mathcal A_i}}(j)(1,j,1)   | \leq \frac {\sqrt \alpha   }{ \left\| \boldsymbol{\mathcal X}^\star(i)\right\|_F  }\right\}} \Biggm | \alpha \right] \nonumber \\
& \overset{(e)}{=} \mathbb E \left[ \frac 1 m \sum_{j=1}^m \left\| \boldsymbol{\mathcal X}^\star(i)\right\|_F \boldsymbol{\mathcal Q}_i  * \mathring{\mathfrak {e}}_1 \left( \hat{\boldsymbol{\mathcal A_i}}(j)(1,j,1) \right)^2   \boldsymbol 1_{ \left\{ | \hat{\boldsymbol{\mathcal A_i}}(j)(1,j,1)   | \leq \frac {\sqrt \alpha   }{ \left\| \boldsymbol{\mathcal X}^\star(i)\right\|_F  }\right\}} \Biggm |\alpha\right] \nonumber \\
& \overset{(f)}{=}  \boldsymbol{\mathcal X}^\star * \mathbb E\left[ \xi^2 \boldsymbol 1_{ \left\{ \left\| \boldsymbol{\mathcal X}^\star(i)\right\|^2_F\xi^2 \leq \alpha \right\}} \right] \mathring{\mathfrak {e}}_i.
\end{align}
Equality (a) holds because $\boldsymbol{\mathcal Q}_i$ is orthogonal. (b) is due to the fact $\boldsymbol{\mathcal X}^\star(i) = \left\| \boldsymbol{\mathcal X}^\star(i) \right\|_F \boldsymbol{\mathcal Q}_i * \mathring{\mathfrak {e}}_1$. (c) is  because $\boldsymbol{\mathcal Q}_i$ is orthogonal and definition $\hat{\boldsymbol{\mathcal A_i}} = \boldsymbol{\mathcal Q}_i^c * \boldsymbol{\mathcal A_i}$. (d) is due to the definition of basis $\mathring{\mathfrak {e}}_i$. Equality (e) holds because  $\hat{\boldsymbol{\mathcal A_i}}$ has same probability distribution as $\boldsymbol{\mathcal A}_i$ since $\boldsymbol{\mathcal Q}_i$ is orthogonal. (f) is based on the definition of T-product.

\end{proof}

We can now leverage the result above to prove Lemma \ref{lemma2}, which provides an upper bound for $\text{Dis}(\boldsymbol{\mathcal  U_0}, \boldsymbol{\mathcal U}^\star)$ based on t-SVD version of Wedin's $\sin \theta$ theory. This bound primarily depends on the extent to which the estimated initialization concentrates around the ground truth $\boldsymbol{\mathcal X}^\star$.  
\begin{lemma}\label{lemma2}
With the same setting as Theorem 3, initialization orthogonal space $\boldsymbol{\mathcal U}_0$ is close $\boldsymbol{\mathcal U}^\star$ enough as
\begin{align}
\text{Dis}(\boldsymbol{\mathcal  U_0}, \boldsymbol{\mathcal U}^\star)   \leq \sqrt 2 \frac{\left\| \boldsymbol{\mathcal X_0} - \mathbb E \left[ \boldsymbol{\mathcal X}_0 | \alpha\right]\right\|}{ \sigma^\star_{\min} \min_j \mathcal D(j,j,1)  - \left\| \boldsymbol{\mathcal X_0} - \mathbb E \left[ \boldsymbol{\mathcal X}_0 | \alpha\right]\right\| }.
\end{align}

\end{lemma}
\begin{proof}
Based on the Lemma \ref{lemma1}, we have t-SVD of $ \boldsymbol{\mathcal M}^\star : = \mathbb E[ \boldsymbol{\mathcal X}_0|\alpha] $ as 
\begin{align}
  \boldsymbol{\mathcal M}^\star = \boldsymbol{\mathcal X}^\star * \boldsymbol{\mathcal D} = \boldsymbol{\mathcal U}^\star* \boldsymbol{\mathcal Q}*\hat{\boldsymbol{\mathcal S}}* \hat{\boldsymbol{\mathcal V}}, \label{t-svd-expec-x0}
\end{align}  
Where $ \boldsymbol{\mathcal Q}, \hat{\boldsymbol{\mathcal V}}$ are orthogonal tensors. The formula $\boldsymbol{\mathcal U}^\star* \boldsymbol{\mathcal Q}$ is because $\mathbb E[ \boldsymbol{\mathcal X}_0|\alpha]$ and $\boldsymbol{\mathcal X}^\star$ have same tensor column space. Based on \eqref{t-svd-expec-x0}, we can conclude that 
\begin{align}
\hat{\boldsymbol{\mathcal S}} = \boldsymbol{\mathcal Q}^c * \boldsymbol{\mathcal S}^\star* \boldsymbol{\mathcal V}^\star* \boldsymbol{\mathcal D}* \hat{\boldsymbol{\mathcal V}}^c.
\end{align}
Firstly, we can conclude that $\sigma_{r+1} \left( \overline{\boldsymbol M^\star}^{(i)}\right) = 0,\forall i \in [n_3]$ due to $\boldsymbol{\mathcal X}^\star$ has $r$ tubal rank. Next, there is 
\begin{align}
\sigma_r \left( \overline{\boldsymbol M^\star}^{(i)}\right) & \overset{(a)}{=} \sigma_r \left( \overline{\boldsymbol U^\star}^{(i)} \overline{\boldsymbol Q}^{(i)}  \overline{\hat{\boldsymbol S}}^{(i)}   \overline{\hat{\boldsymbol V}}^{(i)}\right) \nonumber \\
 & \overset{(b)}{=} \sigma_r \left( \overline{\hat{\boldsymbol S}}^{(i)} \right) \nonumber \\
 & \overset{(c)}{=} \sigma_{\min} \left(  \overline{\boldsymbol Q^c}^{(i)}  \overline{\boldsymbol S^\star}^{(i)} \overline{\boldsymbol V^\star}^{(i)} \overline{\boldsymbol D}^{(i)}   \overline{ \hat{\boldsymbol V}^c}^{(i)} \right) \nonumber \\
 & \overset{(d)}{\geq} \sigma_{\min}  \left( \overline{\boldsymbol Q^c}^{(i)} \right)  \sigma_{\min}  \left( \overline{\boldsymbol S^\star}^{(i)}  \right)  \sigma_{\min}  \left( \overline{\boldsymbol V^\star}^{(i)} \right)  \sigma_{\min}  \left(  \overline{\boldsymbol D}^{(i)}  \right) \sigma_{\min}  \left(  \overline{ \hat{\boldsymbol V}^c}^{(i)}    \right) \nonumber \\
 & \overset{(e)}{\geq} 1  \cdot \sigma_{\min}^\star \cdot 1 \cdot  \min_i \boldsymbol{\mathcal D}(i,i,1) \cdot 1 \nonumber \\
 & = \sigma_{\min}^\star \min_i \boldsymbol{\mathcal D}(i,i,1) \label{lower-bound}
\end{align}
where (a) is due to FFT property, (b) and (c) are due to orthogonal property. (d) is because Lemma \ref{lemma12}. Inequality (e) is due to Lemma \ref{lemm3}. Based on the definition of $\text{Dis}(\boldsymbol{\mathcal  U_0}, \boldsymbol{\mathcal U}^\star)$, we have
\begin{align}
\text{Dis}(\boldsymbol{\mathcal  U_0}, \boldsymbol{\mathcal U}^\star) & = \left\|   \left( \boldsymbol{\mathcal I} - \boldsymbol{\mathcal U}_0* \boldsymbol{\mathcal U}_0^c \right) * \boldsymbol{\mathcal U}^\star \right\| \nonumber \\
& = \max_i \left\|  \left(  \boldsymbol I - \overline{\boldsymbol U_0}^{(i)} \left( \overline{\boldsymbol U_0}^{(i)}  \right)^c  \right) \overline{\boldsymbol U^\star}^{(i)}  \right\| \nonumber \\
& \overset{(a)}{\leq} \sqrt 2 \max_i \frac{ \max \left\{ \left\|\left( \overline{\boldsymbol X_0}^{(i)} - \mathbb E \left[ \overline{\boldsymbol X_0}^{(i)} \Big|\alpha \right] \right)^c \overline{\boldsymbol U^\star}^{(i)} \right\|, \left\|\left( \overline{\boldsymbol X_0}^{(i)} - \mathbb E \left[ \overline{\boldsymbol X_0}^{(i)} \Big|\alpha \right] \right) \left( \overline{ \hat{\boldsymbol V}}^{(i)} \right)^{c} \right\| \right\}    }{  \sigma_r \left( \overline{\boldsymbol M^\star}^{(i)}\right) - \sigma_{r+1} \left( \overline{\boldsymbol M^\star}^{(i)}\right)- \left\| \overline{\boldsymbol X_0}^{(i)} - \mathbb E \left[ \overline{\boldsymbol X_0}^{(i)} \Big|\alpha \right]  \right\| } \nonumber \\
& \overset{(b)}{\leq } \sqrt 2 \max_i \frac{ \left\| \overline{\boldsymbol X_0}^{(i)} - \mathbb E \left[ \overline{\boldsymbol X_0}^{(i)} \Big|\alpha \right]  \right\| }  {  \sigma^\star_{\min} \min_i \boldsymbol{\mathcal D}(i,i,,1) - \left\| \overline{\boldsymbol X_0}^{(i)} - \mathbb E \left[ \overline{\boldsymbol X_0}^{(i)} \Big|\alpha \right]  \right\| } \nonumber \\
& \overset{(c)}{\leq } \sqrt 2  \frac{  \max_i \left\| \overline{\boldsymbol X_0}^{(i)} - \mathbb E \left[ \overline{\boldsymbol X_0}^{(i)} \Big|\alpha \right]  \right\|  }{  \sigma^\star_{\min} \min_i \boldsymbol{\mathcal D}(i,i,,1) - \left\| \overline{\boldsymbol X_0}^{(i)} - \mathbb E \left[ \overline{\boldsymbol X_0}^{(i)} \Big|\alpha \right]  \right\|} \nonumber \\
& \overset{(d)}{ = } \sqrt 2 \frac{ \left\|  \boldsymbol{\mathcal X}_0 - \mathbb E \left[ \boldsymbol{\mathcal X}_0|\alpha \right] \right\|}{ \sigma^\star_{\min} \min_i \boldsymbol{\mathcal D}(i,i,1)  - \left\|  \boldsymbol{\mathcal X}_0 - \mathbb E \left[ \boldsymbol{\mathcal X}_0|\alpha \right] \right\| }. \label{upper-dis_init}
\end{align}
The second equality and (d) are due to the definition of tensor spectral norm. (a) is due to the matrix Wedin theorem as Lemma \ref{lemma4}. (b) is based on the result in \eqref{lower-bound}. (c) is due to the monotonic property. 
 \end{proof}
Based on upper bound formula for $\text{Dis}(\boldsymbol{\mathcal  U_0}, \boldsymbol{\mathcal U}^\star)$ in Lemma \ref{lemma2}, we can prove the Theorem 3 now. The key technique is utilizing the property of circular convolution operator in the T-product to give a tight upper bound of concentration $ \left\| \boldsymbol{\mathcal X_0} - \mathbb E \left[ \boldsymbol{\mathcal X}_0 | \alpha\right]\right\| $. 
\begin{proof}
Based on the upper bound in \eqref{upper-dis_init}, we should measure the concentration of $\boldsymbol{\mathcal X}_0$ on conditional expectation $ \mathbb E \left[ \boldsymbol{\mathcal X}_0|\alpha \right] $. Since there is
\begin{align}
\boldsymbol{\mathcal X}_0 - \mathbb E \left[ \boldsymbol{\mathcal X}_0|\alpha \right] = \frac 1 m \sum_{i=1}^{n_2} \sum_{j=1}^m y_{ji} \boldsymbol 1_{\left\{  | y_{ji}| \leq \sqrt \alpha \right\}} \boldsymbol{\mathcal A}_i(j) *\mathring{\mathfrak {e}}_i^c - \mathbb E \left[    y_{ji}  \boldsymbol 1_{\left\{  | y_{ji}| \leq \sqrt \alpha \right\}} \boldsymbol{\mathcal A}_i(j) *\mathring{\mathfrak {e}}_i^c   \right].  
\end{align}
Based on the variational form of the spectral norm, we have 
\begin{align}
\left\|  \boldsymbol{\mathcal X}_0 - \mathbb E \left[ \boldsymbol{\mathcal X}_0|\alpha \right] \right\| = \max_{\boldsymbol w \in  \boldsymbol \Theta^{n_1} , \boldsymbol z \in  \boldsymbol \Theta^{n_2} } \langle \overline{\boldsymbol X_0} - \mathbb E \left[ \overline{\boldsymbol X_0} | \alpha \right] , \boldsymbol w \boldsymbol z^c \rangle, 
\end{align}
where the set $\boldsymbol \Theta^k$ is defined as $\boldsymbol \Theta^k : = \boldsymbol{\mathcal B}^k \bigcap \boldsymbol{\mathcal S}^k$. The set $\boldsymbol{\mathcal B}^k$ denote the block sparse vectors, specially, $\boldsymbol{\mathcal B}^k : = \left\{  \boldsymbol x \in \mathbb R^{kn_3}| \boldsymbol x = \left[\boldsymbol x_1^T, \cdots, \boldsymbol x_i^T, \cdots, \boldsymbol x_{n_3}^T \right] \right\}$, where $\boldsymbol x_i  \in \mathbb R^k$ and there exists a $j$ such that $\boldsymbol x_j \neq \boldsymbol 0$ and $\boldsymbol x_i = \boldsymbol 0$ for all $i\neq j$. The $\boldsymbol{\mathcal S}^k$ is defined as $\boldsymbol{\mathcal S}^k:= \left\{  \boldsymbol x \in \mathbb R^{kn_3}| \left\| \boldsymbol x \right\|=1 \right\}$. For any fixed $\boldsymbol w \in \boldsymbol \Theta^{n_1}, \boldsymbol z \in \boldsymbol \Theta^{n_2}$, there is
\begin{align}
\langle \overline{\boldsymbol X_0} - \mathbb E \left[ \overline{\boldsymbol X_0} | \alpha \right] , \boldsymbol w \boldsymbol z^c \rangle & = \frac 1 m \sum_{i=1}^{n_2} \sum_{j=1}^m y_{ji} \boldsymbol 1_{ \left\{ |y_{ji}| \leq \sqrt \alpha\right\} }\boldsymbol w^c \overline{\boldsymbol{\mathcal A}_i(j)} \;  \overline{\mathring{\mathfrak {e}}_i^c } \boldsymbol z   -  \boldsymbol w^c \mathbb  E \left[    y_{ji} \boldsymbol 1_{ \left\{ |y_{ji}| \leq \sqrt \alpha\right\} } \boldsymbol w^c \overline{\boldsymbol{\mathcal A}_i(j)} \; \overline{\mathring{\mathfrak {e}}_i^c }  \right] \boldsymbol z\nonumber \\
& \overset{(a)}{ = } \sum_{i,j} \frac{y_{ji}}{m}  \boldsymbol 1_{ \left\{ |y_{ji}| \leq \sqrt \alpha\right\} } \langle \left( \boldsymbol F_{n_3} \otimes \boldsymbol I_{n_1} \right) \text{bcirc} \left( \boldsymbol{\mathcal A}_i(j)* \mathring{\mathfrak {e}}_i^c  \right) \left( \boldsymbol F^{-1}_{n_3}\otimes \boldsymbol I_{n_2}\right)  , \boldsymbol w \boldsymbol z^c \rangle \nonumber \\
& \quad  -  \boldsymbol w^c \mathbb  E \left[    y_{ji} \boldsymbol 1_{ \left\{ |y_{ji}| \leq \sqrt \alpha\right\} } \boldsymbol w^c \overline{\boldsymbol{\mathcal A}_i(j)} \; \overline{\mathring{\mathfrak {e}}_i^c }\right] \boldsymbol z  \nonumber \\
& \overset{(b)}{ = } \sum_{i,j} \frac{y_{ji}}{m}  \boldsymbol 1_{ \left\{ |y_{ji}| \leq \sqrt \alpha\right\} } \langle\boldsymbol{\mathcal A}_i(j) ,\text{bcirc}^\star \left( \left( \boldsymbol F^{-1}_{n_3} \otimes \boldsymbol I_{n_1}\right) \boldsymbol w \boldsymbol z^c  \left( \boldsymbol F_{n_3} \otimes \boldsymbol I_{n_2} \right)  \right) * \mathring{\mathfrak {e}}_i  \rangle  \nonumber \\
& \quad -  \boldsymbol w^c \mathbb  E \left[    y_{ji} \boldsymbol 1_{ \left\{ |y_{ji}| \leq \sqrt \alpha\right\} } \boldsymbol w^c \overline{\boldsymbol{\mathcal A}_i(j)} \; \overline{\mathring{\mathfrak {e}}_i^c }\right] \boldsymbol z, \label{invidi-con}
\end{align}
where (a) is because the block circulant matrix can be block diagonalized and (b) is  due to $\text{bcirc}^\star$ is the joint operator of $\text{bcirc}$ which maps a matrix to a tensor. 

Because the individual terms in the last line of \eqref{invidi-con} are mutually independent sub-Gaussian random variables with zero mean and sub-Gaussian norm $K_{ji}$ as
\begin{align}
K_{ji} & \leq \frac{\sqrt \alpha}{m} \left\|  \text{bcirc}^\star \left( \left( \boldsymbol F^{-1}_{n_3} \otimes \boldsymbol I_{n_1}\right) \boldsymbol w \boldsymbol z^c  \left( \boldsymbol F_{n_3} \otimes \boldsymbol I_{n_2} \right)  \right) * \mathring{\mathfrak {e}}_i  \right\|_F \nonumber \\
& \leq \frac{\kappa \mu \sqrt{1+\epsilon_1}\left\| \boldsymbol{\mathcal X}^\star \right\|_F}{m\sqrt n_2} \left\|  \text{bcirc}^\star \left( \left( \boldsymbol F^{-1}_{n_3} \otimes \boldsymbol I_{n_1}\right) \boldsymbol w \boldsymbol z^c  \left( \boldsymbol F_{n_3} \otimes \boldsymbol I_{n_2} \right)  \right) * \mathring{\mathfrak {e}}_i  \right\|_F,
\end{align}
where the $\epsilon_1$ is a universal constant and is defined subsequently. The second inequality is because $\alpha \leq \kappa^2\mu^2(1+\epsilon_1) \frac{\left\| \boldsymbol{\mathcal X}^\star\right|_F^2}{n_2}$ hold with probability at least $1 - \exp\left( c_3 \frac{mn_2\epsilon_1^2}{\kappa^2\mu^2} \right)$ based on sub-exponential Berstein inequality. Let's define $t:= \epsilon_1 \left\| \boldsymbol{\mathcal X}^\star \right\|_F$. Thus, there is 
\begin{align}
\frac{t^2}{ \sum_{i=1}^{n_2} \sum_{j=1}^m K_{ji}^2 } & \overset{(a)}{\geq}\frac{\epsilon_1^2 \left\| \boldsymbol{\mathcal X}^\star \right\|_F^2   }{ \sum_{i=1}^{n_2} \sum_{j=1}^m \frac{(1+\epsilon_1)\kappa^2 \mu^2}{m^2 n_2} \left\| \boldsymbol{\mathcal X}^\star\right\|_F^2  \left\|  \text{bcirc}^\star \left( \left( \boldsymbol F^{-1}_{n_3} \otimes \boldsymbol I_{n_1}\right) \boldsymbol w \boldsymbol z^c  \left( \boldsymbol F_{n_3} \otimes \boldsymbol I_{n_2} \right)  \right) * \mathring{\mathfrak {e}}_i  \right\|_F^2   } \nonumber \\
& = \frac{m n_2 \epsilon_1^2}{(1+\epsilon_1)\kappa^2\mu^2 \sum_{i=1}^{n_2} \left\|  \text{bcirc}^\star \left( \left( \boldsymbol F^{-1}_{n_3} \otimes \boldsymbol I_{n_1}\right) \boldsymbol w \boldsymbol z^c  \left( \boldsymbol F_{n_3} \otimes \boldsymbol I_{n_2} \right)  \right) * \mathring{\mathfrak {e}}_i  \right\|_F^2  } \nonumber \\
& \overset{(b)}{ = } \frac{m n_2 \epsilon_1^2}{(1+\epsilon_1)\kappa^2\mu^2  \left\|  \text{bcirc}^\star \left( \left( \boldsymbol F^{-1}_{n_3} \otimes \boldsymbol I_{n_1}\right) \boldsymbol w \boldsymbol z^c  \left( \boldsymbol F_{n_3} \otimes \boldsymbol I_{n_2} \right)  \right) * \boldsymbol{\mathcal I}  \right\|_F^2  } \nonumber \\
& \overset{(c)}{ = }   \frac{mn_2\epsilon_1^2}{(1+\epsilon_1)\kappa^2\mu^2 n_3 \left\| \boldsymbol w \boldsymbol z^c \right\|_F^2} \nonumber \\
&  \overset{(d)}{=} \frac{mn_2\epsilon_1^2}{(1+\epsilon_1)\kappa^2\mu^2n_3},
\end{align}
where (b) is due to the definition of T-product and (d) is due to
property that $ \left\| \boldsymbol w\boldsymbol z^c \right\|_F^2 = \text{Tr}(\boldsymbol z \boldsymbol w^c \boldsymbol w\boldsymbol z^c ) =  \text{Tr}( \boldsymbol z \boldsymbol z^c ) = \text{Tr}( \boldsymbol z^c \boldsymbol z ) = 1$. (c) is due to $ \left\|  \text{bcirc}^\star \left( \left( \boldsymbol F^{-1}_{n_3} \otimes \boldsymbol I_{n_1}\right) \boldsymbol w \boldsymbol z^c  \left( \boldsymbol F_{n_3} \otimes \boldsymbol I_{n_2} \right)  \right) * \boldsymbol{\mathcal I}  \right\|_F = \sqrt {n_3} \left\|\boldsymbol w \boldsymbol z^c \right\|_F$. 

Since the above holds for fixed $\boldsymbol w, \boldsymbol z$ with the condition that $\alpha \leq \kappa^2\mu^2(1+\epsilon_1) \frac{\left\| \boldsymbol{\mathcal X}^\star\right|_F^2}{n_2}$. Thus, based on union bound and Lemma \ref{lemma5} for the epsilon net argument of sub-Gaussian Bernstein inequality, we have
\begin{align}
\left\|  \boldsymbol{\mathcal X}_0 - \mathbb E \left[ \boldsymbol{\mathcal X}_0  | \alpha\right]  \right\| & \overset{(a)}{\leq} 1.4 \epsilon_1 \left\| \boldsymbol{\mathcal X}^\star \right\|_F \nonumber \\
& \overset{(b)}{ = } \frac{\epsilon_0}{\sqrt r \kappa } \left\| \boldsymbol{\mathcal X}^\star \right\|_F \nonumber \\
& \overset{(c)}{ \leq } \frac{\epsilon_0}{\sqrt r \kappa} \sqrt{\frac{r(\sigma^\star_{\max})^2n_3}{n_3}} \nonumber \\
& = \epsilon_0 \sigma^\star_{\min} \label{concen-upper}
\end{align}
with probability at least 
\begin{align}
1 - \exp \left( c_1 (n_1 + n_2) \log n_3 - \frac{c_2 \epsilon_0^2 m n_2}{\kappa^4\mu^2n_3 r}  \right) - \exp \left( - c_3 \frac{mn_2\epsilon_0^2}{\kappa^4 \mu^2 r}\right).
\end{align}
(a) is because we choose upper bound as $t:= \epsilon_1 \left\| \boldsymbol{\mathcal X}^\star \right\|_F$ and (b) is because we set $\epsilon_1 : = \frac{\epsilon_0}{1.4 \sqrt r \kappa}$ where $\epsilon_0$ is the constant that would be determined subsequently. (c) is due to 
$\left\| \boldsymbol{\mathcal X}^\star \right\|_F =  \frac{\left\|\overline{\boldsymbol X^\star }\right\|_F }{\sqrt n_3}  \leq  \frac{\sqrt{n_3r(\sigma_{\max}^\star)^2}}{\sqrt n_3}$. The term $c_1 (n_1+n_3) \log n_3$ is because the covering number of $\boldsymbol \Theta^k $ is order of $\mathcal O\left(  n_3 e^k \right)$ based on Lemma \ref{lemma13}.

Finally, we can substitute the upper bound of concentration in \eqref{concen-upper} into \eqref{upper-dis_init} to obtain that
\begin{align}
\text{Dis}(\boldsymbol{\mathcal U}_0, \boldsymbol{\mathcal U}^\star) & \overset{(a)}{ \leq }  \sqrt 2 \frac{ \epsilon_0 \sigma_{\min}^\star }{   \sigma_{\min}^\star \min_j \mathcal D(j,j,1) - \epsilon_0 \sigma_{\min}^\star  } \nonumber \\
 & \overset{(b)}{ \leq } \frac{  \epsilon_0 \sigma_{\min}^\star   }{ 0.92 \sigma_{\min}^\star  - \epsilon_0 \sigma_{\min}^\star} \nonumber \\
 & \overset{(c)}{ \leq }  1.6 \epsilon_0 \nonumber \\
 & \overset{(d)}{ =  } \frac{0.016}{\sqrt r \kappa^2}  
\end{align}
with probability at least
\begin{align}
1 - \exp \left( c_1 (n_1 + n_2) \log n_3 - \frac{c_2 \epsilon_0^2 m n_2}{\kappa^8\mu^2n_3 r^2}  \right) - \exp \left( - c_3 \frac{mn_2}{\kappa^8 \mu^2 r^2}\right). \label{final-pro-init}
\end{align}
The (a) is due to the monotone property and (b) is due to the Fact 3.9 in \cite{nayer2022fast}. (c), (d) and \eqref{final-pro-init} are because we set $\epsilon_0 : = \frac{0.01}{\sqrt r \kappa^2}$. 
\end{proof}

\section{Proof of Theorem 4}\label{sec2}
The main idea behind the proof of Theorem 4 relies on gradient decomposition, where the gradient is divided into the population gradient and an additional concentration error term. For controlling the concentration error of the empirical gradient, we need the following two lemmas that are useful to bound the concentration error tightly.

The following lemma can link the estimating error of $\boldsymbol{\mathcal V}_t$ in the exact minimization stage with the principle angel distance $Dis(\boldsymbol{\mathcal U}_t, \boldsymbol{\mathcal U}^\star)$ between current estimating subspace $\boldsymbol{\mathcal U}_t$ and ground truth subspace $\boldsymbol{\mathcal U}^\star$ in the projected gradient stage. 

\begin{lemma}\label{lemma6}
Consider the sequence generate by Algorithm 1 and defined $\boldsymbol{\mathcal G}_t = \boldsymbol{\mathcal U}^c_t*\boldsymbol{\mathcal U}^\star*\boldsymbol{\mathcal V}^\star$, then 
\begin{align}
\left\| \boldsymbol{\mathcal V}_t (i) - \boldsymbol{\mathcal G}_t(i)\right\|_F \leq  0.36 \text{Dis} \left( \boldsymbol{\mathcal U}_t, \boldsymbol{\mathcal U}^\star \right) \left\| \boldsymbol{\mathcal V}^\star(i) \right\|_F, \quad  i \in  [n_2]
\end{align}
with probability at least $1-n_2 \exp\left(r \log n_3 - c_6. m\right)$
\end{lemma}

\begin{proof}
Based on the exact minimization problem for $\boldsymbol{\mathcal V}$, we recall its explicit updating formula as
\begin{align}
\text{Unfold} \left( \boldsymbol{\mathcal V}_{t}(i)\right) & = \left(  \text{bcirc}\left( \boldsymbol{\mathcal U}^c_t \right) \text{Unfold}\left( \boldsymbol{\mathcal A}_i \right) \cdot \left(  \text{bcirc}\left( \boldsymbol{\mathcal U}^c_t \right) \text{Unfold}\left( \boldsymbol{\mathcal A}_i \right) \right)^c \right)^{-1}  \nonumber \\
& \cdot \text{bcirc}\left( \boldsymbol{\mathcal U}^c_t \right) \text{Unfold}\left( \boldsymbol{\mathcal A}_i \right) \cdot \left( \text{bcirc}\left( \left( \boldsymbol{\mathcal U}^\star \right)^c \right) \text{Unfold}\left( \boldsymbol{\mathcal A}_i \right) \right)^c \text{Unfold} \left( \boldsymbol{\mathcal V}^\star(i) \right) \nonumber \\
& = \left( \boldsymbol{ Q}_{t,i} \boldsymbol{ Q}_{t,i}^c  \right)^{-1} \boldsymbol{Q}_{t,i} \cdot\left( \text{bcirc}\left( \left( \boldsymbol{\mathcal U}^\star \right)^c \right) \text{Unfold}\left( \boldsymbol{\mathcal A}_i \right) \right)^c \text{Unfold} \left( \boldsymbol{\mathcal V}^\star(i) \right) ,
\end{align}
where we have define $ \boldsymbol{Q}_{t,i}  : =  \text{bcirc}\left( \boldsymbol{\mathcal U}^c_t \right) \text{Unfold}\left( \boldsymbol{\mathcal A}_i \right) $. Then we have the following equality based on definition of $\boldsymbol{\mathcal G}_t$
\begin{align}
\text{Unfold} \left( \boldsymbol{\mathcal G}_t(i)\right) = \text{bcirc} \left( \boldsymbol{\mathcal U}^c_t \right) \text{bcirc} \left( \boldsymbol{\mathcal U}^\star \right) \text{Unfold} \left( \boldsymbol{\mathcal V}^\star(i) \right), \label{def-G}
\end{align}
which is used to link the principle distance measure to the estimating error for $\boldsymbol{\mathcal V}$ as 
\begin{align}
\text{Unfold} \left( \boldsymbol{\mathcal V}_t(i) - \boldsymbol{\mathcal G}_t(i)\right) & { = } \left( \boldsymbol{ Q}_{t,i} \boldsymbol{ Q}_{t,i}^c  \right)^{-1} \boldsymbol{Q}_{t,i} \cdot \left(  \text{Unfold} \left( \boldsymbol{\mathcal A}_i \right)\right)^c \text{bcirc}\left( \boldsymbol{\mathcal U}^\star \right) \text{Unfold} \left( \boldsymbol{\mathcal V}^\star(i) \right) \nonumber \\
& \quad - \text{bcirc} \left( \boldsymbol{\mathcal U}^c_t \right) \text{bcirc} \left( \boldsymbol{\mathcal U}^\star \right) \text{Unfold} \left( \boldsymbol{\mathcal V}^\star(i) \right) \nonumber \\
& =  \left( \boldsymbol{ Q}_{t,i} \boldsymbol{ Q}_{t,i}^c  \right)^{-1} \boldsymbol{Q}_{t,i} \cdot \left(   \text{Unfold} \left( \boldsymbol{\mathcal A}_i \right)\right)^c \cdot  \nonumber \\
& \quad \left( \boldsymbol{I} - \text{bcirc} \left( \boldsymbol{\mathcal U}_t\right) \text{bcirc} \left( \boldsymbol{\mathcal U}_t^c\right) \right) \text{bcirc} \left( \boldsymbol{\mathcal U}^\star \right) \text{Unfold} \left(  \boldsymbol{\mathcal V}^\star(i) \right).  \label{diffence}
\end{align}
In above equality, we have used Lemma \ref{lemma14}. Then we have
\begin{align}
\mathbb E \left[ \boldsymbol Q_{t,i} \boldsymbol Q_{t,i}^c\right] & = \mathbb E \left[  \text{bcirc} \left( \boldsymbol{\mathcal U}_t^c \right) \text{Unfold}\left( \boldsymbol{\mathcal A}_i \right)   \left( \text{Unfold}\left( \boldsymbol{\mathcal A}_i \right)   \right)^c \text{bcirc} \left( \boldsymbol{\mathcal U}_t\right)\right] \nonumber \\
& \overset{(a)}{ = }  m \boldsymbol I_{rn_3} \text{bcirc} \left( \boldsymbol{\mathcal U}_t^c \right)  \text{bcirc} \left( \boldsymbol{\mathcal U}_t\right) \nonumber \\
& \overset{(b)}{ = }  m \boldsymbol I_{rn_3} \text{bcirc} \left( \boldsymbol{\mathcal U}^c_t * \boldsymbol{\mathcal U}_t \right) \nonumber \\
& \overset{(c)}{ = } m \boldsymbol I_{rn_3}, \label{expec-q}
\end{align}
where the (a) is due to  $\mathbb E \left[ \text{Unfold}\left( \boldsymbol{\mathcal A}_i \right)   \left( \text{Unfold}\left( \boldsymbol{\mathcal A}_i \right)   \right)^c \right] = m \boldsymbol I_{n_1n_3}$ due to $\boldsymbol{\mathcal A}_i$ has i.i.d. standard Gaussian entries under Assumption 2. (b) is due to the Lemma \ref{lemma7}. (c) holds because $\boldsymbol{\mathcal U}_t$ has orthogonal tensor columns given by the QR decomposition in Algorithm 1.

Next, we aim to bound the spectral norm of $\boldsymbol Q_{t,i}$, based on \eqref{expec-q}, we can first obtain the upper bound of $\left\|  \boldsymbol Q_{t,i} \boldsymbol Q_{t,i}^c - m \boldsymbol I_{rn_3} \right\|$. Then based on the definition of $\boldsymbol Q_{t,i}$, we have
\begin{align}
\left \| \boldsymbol Q_{t,i} \boldsymbol Q_{t,i}^c - \mathbb E \left[ \boldsymbol Q_{t,i} \boldsymbol Q_{t,i}^c \right] \right\|   & =   \left\| \text{bcirc}\left( \boldsymbol{\mathcal U}^c_t \right) \text{Unfold}\left( \boldsymbol{\mathcal A}_i \right) \left(  \text{Unfold}\left( \boldsymbol{\mathcal A}_i \right)  \right)^c    \text{bcirc}\left( \boldsymbol{\mathcal U}_t \right)  - \mathbb E \left[ \boldsymbol Q_{t,i} \boldsymbol Q_{t,i}^c \right] \right\|   \nonumber \\
& =  \left\| (\boldsymbol F_{n_3} \otimes \boldsymbol I_r) \text{bcirc}\left( \boldsymbol{\mathcal U}^c_t \right) ( \boldsymbol F_{n_3}^{-1} \otimes \boldsymbol I_{n_1} ) ( \boldsymbol F_{n_3} \otimes \boldsymbol I_{n_1} ) \text{Unfold}\left( \boldsymbol{\mathcal A}_i \right) \right. \nonumber \\
& \quad  \left. \cdot   \left(  \text{Unfold}\left( \boldsymbol{\mathcal A}_i \right)  \right)^c  (\boldsymbol F_{n_3}^{-1} \otimes \boldsymbol I_{n_1}) (\boldsymbol F_{n_3} \otimes \boldsymbol I_{n_1}) \text{bcirc}\left( \boldsymbol{\mathcal U}_t \right) (\boldsymbol F_{n_3}^{-1} \otimes \boldsymbol I_{r}) \right.    \nonumber \\
& \quad \left.  -    \mathbb E \left[ \boldsymbol Q_{t,i} \boldsymbol Q_{t,i}^c \right]    \right\| \nonumber \\
& = \left \| \overline{\boldsymbol U^c_t}  ( \boldsymbol F_{n_3} \otimes \boldsymbol I_{n_1} ) \text{Unfold}\left( \boldsymbol{\mathcal A}_i \right) \cdot \left(  \text{Unfold}\left( \boldsymbol{\mathcal A}_i \right)  \right)^c  (\boldsymbol F_{n_3}^{-1} \otimes \boldsymbol I_{n_1}) \overline{\boldsymbol U}_t \right. \nonumber \\
& \quad \left. - \mathbb E \left[ \boldsymbol Q_{t,i} \boldsymbol Q_{t,i}^c \right]  \right\|. 
\end{align}
Because $ \boldsymbol Q_{t,i} \boldsymbol Q_{t,i}^c$ is symmetric, based on the variational formula of the spectral norm, we have
\begin{align}
\left \| \boldsymbol Q_{t,i} \boldsymbol Q_{t,i}^c - \mathbb E \left[ \boldsymbol Q_{t,i} \boldsymbol Q_{t,i}^c \right] \right\|  = \max_{ \boldsymbol w \in \boldsymbol \Theta^r  } \left\| \left(\left(  \text{Unfold}\left( \boldsymbol{\mathcal A}_i \right)  \right)^c  \sqrt {n_3} (\boldsymbol F_{n_3}^{-1} \otimes \boldsymbol I_{n_1})   \overline{\boldsymbol U}_t - \sqrt { \mathbb E \left[  \boldsymbol Q_{t,i} \boldsymbol Q_{t,i}^c  \right] } \right) \boldsymbol w \right\|^2.
\end{align}
Because $ \sqrt{n_3} (\boldsymbol F_{n_3}^{-1} \otimes \boldsymbol I_{n_1})   \overline{\boldsymbol U}_t $ is orthogonal matrix, for any fixed $\boldsymbol w \in \boldsymbol \Theta^r $, the maximized term is summation of squares of Gaussian random variables. Utilizing a similar argument of $\epsilon$-net sub-exponential Berstein inequality as initialization, we can obtain
\begin{align}
\text{Pr} \left(  \left\|  \boldsymbol Q_{t,i} \boldsymbol Q_{t,i}^c - m \boldsymbol I_{rn_3} \right\| \leq 1.4\epsilon_2 m \right) \geq 1 - \exp \left( r\log n_3 - \epsilon_2 m \right),
\end{align}
where the $\exp(r\log n_3)$ term is due to the covering number of $\boldsymbol \Theta^r$ is order of $n_3 \exp(r)$ based on Lemma \ref{lemma13}. Thus, we have $\sigma_{\min} \left( \boldsymbol Q_{t,i} \boldsymbol Q_{t,i}^c \right) - m \geq -1.4\epsilon_2m$, which indicates that $\left\| ( \boldsymbol Q_{t,i} \boldsymbol Q_{t,i}^c )^{-1}\right\| \leq \frac{1}{(1-1.4\epsilon_2)m} = \frac{1}{0.86m}$, where we set $\epsilon_2 : = 0.1$. 

Next, we need to bound the other term in \eqref{diffence} with a similar derivation as the above result. There is 
\begin{align}
& \mathbb E \left[  \boldsymbol Q_{t,i} \left(   \text{Unfold} \left( \boldsymbol{\mathcal A}_i \right)\right)^c \cdot \left( \boldsymbol{I} - \text{bcirc} \left( \boldsymbol{\mathcal U}_t\right) \text{bcirc} \left( \boldsymbol{\mathcal U}_t^c\right) \right) \text{bcirc} \left( \boldsymbol{\mathcal U}^\star \right) \text{Unfold} \left(  \boldsymbol{\mathcal V}^\star(i) \right) \right] \nonumber \\
& \quad = \mathbb  E \left[  \text{bcirc}(\boldsymbol{\boldsymbol U}^c_t) \text{Unfold}(\boldsymbol{\mathcal A}_i) \left(   \text{Unfold} \left( \boldsymbol{\mathcal A}_i \right)\right)^c \cdot \left( \boldsymbol{I} - \text{bcirc} \left( \boldsymbol{\mathcal U}_t\right) \text{bcirc} \left( \boldsymbol{\mathcal U}_t^c\right) \right) \text{bcirc} \left( \boldsymbol{\mathcal U}^\star \right) \text{Unfold} \left(  \boldsymbol{\mathcal V}^\star(i) \right) \right] \nonumber \\
& \quad = m \mathbb E \left[  \text{bcirc}(\boldsymbol{\boldsymbol U}^c_t) \cdot \left( \boldsymbol{I} - \text{bcirc} \left( \boldsymbol{\mathcal U}_t\right) \text{bcirc} \left( \boldsymbol{\mathcal U}_t^c\right) \right) \text{bcirc} \left( \boldsymbol{\mathcal U}^\star \right) \text{Unfold} \left(  \boldsymbol{\mathcal V}^\star(i) \right) \right] \nonumber \\
& \quad = m \mathbb E \left[   \left( \text{bcirc}(\boldsymbol{\boldsymbol U}^c_t) -  \text{bcirc}(\boldsymbol{\boldsymbol U}^c_t) \text{bcirc} \left( \boldsymbol{\mathcal U}_t\right) \text{bcirc} \left( \boldsymbol{\mathcal U}_t^c\right) \right) \text{bcirc} \left( \boldsymbol{\mathcal U}^\star \right) \text{Unfold} \left(  \boldsymbol{\mathcal V}^\star(i) \right) \right] \nonumber \\
& \quad = \boldsymbol 0,
\end{align}
where the last equality is because $  \text{bcirc}(\boldsymbol{\boldsymbol U}^c_t) \text{bcirc} \left( \boldsymbol{\mathcal U}_t\right) \text{bcirc} \left( \boldsymbol{\mathcal U}_t^c\right)  = \text{bcirc}(\boldsymbol{\mathcal U}^c_t) - \text{bcirc} \left( \boldsymbol{\mathcal U}^c_t * \boldsymbol{\mathcal U}_t * \boldsymbol{\mathcal U}^c_t\right) = {\boldsymbol{\mathcal U}}^c_t$ based on Lemma \ref{lemma7}. Then based on similar $\epsilon$-net sub-exponential Berstein inequality, we have at least probability $1-\exp \left( r \log n_3 - \epsilon_3 m\right)$, there is
\begin{align}
& \left\|   \boldsymbol Q_{t,i} \left(   \text{Unfold} \left( \boldsymbol{\mathcal A}_i \right)\right)^c \cdot \left( \boldsymbol{I} - \text{bcirc} \left( \boldsymbol{\mathcal U}_t\right) \text{bcirc} \left( \boldsymbol{\mathcal U}_t^c\right) \right) \text{bcirc} \left( \boldsymbol{\mathcal U}^\star \right) \text{Unfold} \left(  \boldsymbol{\mathcal V}^\star(i) \right)  \right\| \leq 1.4 \epsilon_3 m \nonumber \\
& \quad \quad \cdot \left\|    \left( \boldsymbol{I} - \text{bcirc} \left( \boldsymbol{\mathcal U}_t\right) \text{bcirc} \left( \boldsymbol{\mathcal U}_t^c\right) \right) \text{bcirc} \left( \boldsymbol{\mathcal U}^\star \right) \text{Unfold} \left(  \boldsymbol{\mathcal V}^\star(i) \right)      \right\| \nonumber \\
& \quad = 1.4 \epsilon_3 m \left\| \left( \boldsymbol{\mathcal I} - \boldsymbol{\mathcal U}_t * \boldsymbol{\mathcal U}_t^c \right) * \boldsymbol{\mathcal V}^\star(i)  \right\|_F.
\end{align}
Finally, based on \eqref{diffence}, $\forall i \in [n_2]$, we have
\begin{align}
\left\| \boldsymbol{\mathcal V}_t(i) - \boldsymbol{\mathcal G}_t(i)\right\|_F & \leq \frac{1.4\epsilon_3}{0.86}  \left\| \left( \boldsymbol{\mathcal I} - \boldsymbol{\mathcal U}_t * \boldsymbol{\mathcal U}_t^c \right) * \boldsymbol{\mathcal V}^\star(i)  \right\|_F \nonumber \\
& \leq  0.36 \text{Dis}(\boldsymbol{\mathcal U}_t, \boldsymbol{\mathcal U}^\star) \left\| \boldsymbol{\mathcal V}^\star(i) \right\|_F
\end{align}
with probability at least $1 - \exp \left( r\log n_3 - c_6 m \right)$. The second inequality is due to setting $\epsilon_3:= 0.22$. We would finish proving this lemma by taking union bound along $i\in [n_2]$. 
\end{proof}

Building on Lemma \ref{lemma6}, we derive the following key facts required to prove the contraction of projected gradient descent. The crucial condition is that the current estimated subspace $\boldsymbol{\mathcal U}^t$ is sufficiently close to the ground truth $\boldsymbol{\mathcal U}^\star$. This is precisely why we introduced the truncated spectral initialization in the algorithm, ensuring that this condition is met.
\begin{lemma}\label{lemma8}
If the current $\boldsymbol{\mathcal U}_t$ generated by Algorithm 1 is close enough to $\boldsymbol{\mathcal U}^\star$ such that $\text{Dis}(\boldsymbol{\mathcal U}_t, \boldsymbol{\mathcal U}^\star) \leq \frac{0.016}{\sqrt r \kappa^2 }$ and $m \gtrsim \max\left\{ \log n_2, r \log n_3 \right\}$, then with high probability, there are

(1) $\left\| \boldsymbol{\mathcal V}(i) \right\|_F \leq 1.1 \left\|\boldsymbol{\mathcal V}^\star(i)\right\|_F$.

(2) $\left\| \boldsymbol{\mathcal X}_t(i) -  \boldsymbol{\mathcal X}^\star(i)\right\|_F \leq 1.36 \text{Dis}(\boldsymbol{\mathcal U}_t, \boldsymbol{\mathcal U}^\star) \left\|\boldsymbol{\mathcal V}^\star(i)\right\|_F$. 

(3) $\left\| \boldsymbol{\mathcal V}_t - \boldsymbol{\mathcal G}_t \right\|_F \leq 0.36  \text{Dis}(\boldsymbol{\mathcal U}_t, \boldsymbol{\mathcal U}^\star) \sqrt r \left\| \boldsymbol{\mathcal X}^\star \right\|$. 

(4) $\left\|\boldsymbol{\mathcal X}_t - \boldsymbol{\mathcal X}^\star \right\|_F \leq 1.36 \text{Dis}(\boldsymbol{\mathcal U}_t, \boldsymbol{\mathcal U}^\star) \sqrt r \left\| \boldsymbol{\mathcal X}^\star \right\|$. 

(5) $\sigma_{\min}(\boldsymbol{\mathcal V}_t) \geq 0.94 \sigma^\star_{\min}$.

(6) $\left\| \boldsymbol{\mathcal V}_t \right\| \leq 1.1 \left\| \boldsymbol{\mathcal X}^\star \right\|$.
\end{lemma}

\begin{proof}
(1) $\left\| \boldsymbol{\mathcal V}_t(i) \right\|_F \leq \left\|\boldsymbol{\mathcal V}_t(i) - \boldsymbol{\mathcal G}_t(i)\right\|_F + \left\|  \boldsymbol{\mathcal G}_t(i) \right\|_F \leq 0.36 \text{Dis}(\boldsymbol{\mathcal U}_t, \boldsymbol{\mathcal U}^\star) \left\| \boldsymbol{\mathcal V}^\star(i) \right\|_F + \left\| \boldsymbol{\mathcal V}^\star(i) \right\|_F \leq 1.1 \left\| \boldsymbol{\mathcal V}^\star(i) \right\|_F$. The second and third inequality is due to Lemma \ref{lemma6} and condition that $\text{Dis}(\boldsymbol{\mathcal U}_t, \boldsymbol{\mathcal U}^\star) \leq \frac{0.016}{\sqrt r \kappa^2 }$, respectively.

\noindent (2) \begin{align}
 \left\| \boldsymbol{\mathcal X}_t(i)  - \boldsymbol{\mathcal X}^\star(i) \right\|_F & = \left\| \boldsymbol{\mathcal X}_t(i)  - \boldsymbol{\mathcal U}_t*\boldsymbol{\mathcal G}_t(i) +    \boldsymbol{\mathcal U}_t*\boldsymbol{\mathcal G}_t(i) - \boldsymbol{\mathcal X}^\star(i)  \right\|_F \nonumber \\
& = \left\| \boldsymbol{\mathcal U}_t*\left( \boldsymbol{\mathcal V}_t(i) - \boldsymbol{\mathcal G}_t(i) \right)  + ( \boldsymbol{\mathcal U}_t * \boldsymbol{\mathcal U}_t^c - \boldsymbol{\mathcal I})*\boldsymbol{\mathcal U}^\star * \boldsymbol{\mathcal V}^\star(i)   \right\|_F \nonumber \\
& \leq \left\|   \boldsymbol{\mathcal V}_t(i) - \boldsymbol{\mathcal G}_t(i)  \right\|_F + \text{Dis}(\boldsymbol{\mathcal U}_t, \boldsymbol{\mathcal U}^\star) \left\|\boldsymbol{\mathcal V}^\star(i) \right\|_F \nonumber \\
& \leq 1.36  \text{Dis}(\boldsymbol{\mathcal U}_t, \boldsymbol{\mathcal U}^\star)  \left\|\boldsymbol{\mathcal V}^\star(i) \right\|_F,
\end{align}
where the last inequality is due to Lemma \ref{lemma6}.

\noindent (3) $\left\| \boldsymbol{\mathcal V}_t - \boldsymbol{\mathcal G}_t\right\|_F = \sqrt { \sum_{i=1}^{n_2} \left\| \boldsymbol{\mathcal V}_t(i) - \boldsymbol{\mathcal G}_t(i)  \right\|_F^2   } \leq 0.36  \text{Dis}(\boldsymbol{\mathcal U}_t, \boldsymbol{\mathcal U}^\star) \left\| \boldsymbol{\mathcal V}^\star \right\|_F = \frac{ 0.36  \text{Dis}(\boldsymbol{\mathcal U}_t, \boldsymbol{\mathcal U}^\star)}{\sqrt {n_3}} \left\| \overline{\boldsymbol V^\star } \right\|_F \leq \frac{ 0.36  \text{Dis}(\boldsymbol{\mathcal U}_t, \boldsymbol{\mathcal U}^\star)}{\sqrt {n_3}} \sqrt{r n_3 \left\| \boldsymbol{\mathcal X}^\star \right\|^2 } = 0.36  \text{Dis}(\boldsymbol{\mathcal U}_t, \boldsymbol{\mathcal U}^\star) \sqrt r \left\| \boldsymbol{\mathcal X}^\star \right\|  $. The first inequality is due to Lemma \ref{lemma6} and second equality is because $\forall \boldsymbol{\mathcal X} \in \mathcal \mathbb R^{n_1 \times n_2 \times n_3}$, there is $\left\| \boldsymbol{\mathcal X} \right\|_F  = \frac{\left\| \overline{\boldsymbol X}\right\|_F}{\sqrt {n_3}}$. The second inequality is due to $\left\|  \overline{\boldsymbol V^\star}\right\|_F^2 = \left\|  \overline{\boldsymbol X^\star}\right\|_F^2 \leq n_3 r \left\| \boldsymbol{\mathcal X}^\star \right\|_F^2 $. 

\noindent (4) $\left\|  \boldsymbol{ \mathcal X}_t - \boldsymbol{\mathcal X}^\star \right\|_F \leq 1.36 \text{Dis}(\boldsymbol{\mathcal U}_t, \boldsymbol{\mathcal U}^\star) \left\| \boldsymbol{\mathcal V}^\star \right\|_F \leq 1.36 \text{Dis}(\boldsymbol{\mathcal U}_t, \boldsymbol{\mathcal U}^\star) \left\| \boldsymbol{\mathcal X}^\star \right\| $. The first inequality is due to fact (2).

(5) Firstly, we have the following lower bound for the $\sigma_{\min}(\boldsymbol{\mathcal V}_t)$ as
\begin{align}
\sigma_{\min}(\boldsymbol{\mathcal V}_t) = \sigma_{\min}(\overline{\boldsymbol V_t}) \geq \sigma_{\min}(\overline{\boldsymbol G_t}) - \left\| \overline{\boldsymbol V_t} - \overline{\boldsymbol G_t}\right\| \geq \sigma_{\min}(\overline{\boldsymbol G_t}) - \left\| \overline{\boldsymbol V_t} - \overline{\boldsymbol G_t}\right\|_F, \label{first-lower}
\end{align}
where we have used Weyl's inequality in the first inequality. Now, we derive the lower bound for $ \sigma_{\min}(\overline{\boldsymbol G_t})$ as follows
\begin{align}
\sigma_{\min}(\overline{\boldsymbol G_t})  & = \sigma_{\min}(\overline{\boldsymbol G_t}^c) = \sigma_{\min} \left( \left( \overline{\boldsymbol U_t^c} \,\overline{\boldsymbol U^\star}\, \overline{\boldsymbol V^\star}\right)^c \right) \nonumber \\
& = \sigma_{\min} \left(  \left( \overline{\boldsymbol V^\star}\right)^c  \left(\overline{\boldsymbol U^\star}\right)^c \left(  \overline{\boldsymbol U^c_t} \right)^c      \right) \nonumber \\
& \overset{(a)}{\geq} \sigma_{\min}^\star \sigma_{\min} \left(  \left( \overline{\boldsymbol U^\star}\right)^c \left( \overline{\boldsymbol U^c_t} \right)^c  \right) \nonumber \\
& = \sigma_{\min}^\star \sqrt{    \sigma_{\min} \left(       \overline{\boldsymbol U^c_t}\, \overline{\boldsymbol U^\star}\, \left(\overline{\boldsymbol U^\star} \right)^c\, \left( \overline{\boldsymbol U^c_t} \right)^c \right)      } \nonumber \\
& \overset{(b)}{ = } \sigma_{\min}^\star \sqrt{    \sigma_{\min} \left( \overline{\boldsymbol U^c_t}\, \overline{\boldsymbol U^\star}\, \overline{\left(\boldsymbol U^\star\right)^c} \,  \overline{\boldsymbol U_t}  \right)      } \nonumber \\
& = \sigma_{\min}^\star \sqrt{    \sigma_{\min} \left( \overline{\boldsymbol U^c_t}\, \left( \boldsymbol I - \left( \boldsymbol I - \overline{\boldsymbol U^\star}\, \overline{\left(\boldsymbol U^\star\right)^c} \right)\right) \,  \overline{\boldsymbol U_t}  \right)      } \nonumber \\
& = \sigma_{\min}^\star \sqrt{  \sigma_{\min} \left(   \overline{\boldsymbol U^c_t}\, \overline{\boldsymbol U_t} -  \overline{\boldsymbol U^c_t} \left( \boldsymbol I - \overline{\boldsymbol U^\star}\, \overline{(\boldsymbol U^\star)^c} \right)  \overline{\boldsymbol U_t}  \right)       } \nonumber \\
& = \sigma_{\min}^\star \sqrt{  \sigma_{\min} \left(   \boldsymbol I -  \overline{\boldsymbol U^c_t} \left( \boldsymbol I - \overline{\boldsymbol U^\star}\, \overline{(\boldsymbol U^\star)^c} \right)  \overline{\boldsymbol U_t}  \right)       } \nonumber \\
& = \sigma_{\min}^\star \sqrt{ 1- \sigma_{max} \left(  \overline{\boldsymbol U^c_t} \left( \boldsymbol I - \overline{\boldsymbol U^\star}\, \overline{(\boldsymbol U^\star)^c} \right)  \overline{\boldsymbol U_t} \right)  } \nonumber \\
& \overset{(c)}{ = } \sigma_{\min}^\star \sqrt{  1 - \sigma_{\max} \left(   \left( \left( \boldsymbol I - \overline{\boldsymbol U^\star}\, \overline{({\boldsymbol U}^\star)^c}\right) \overline{\boldsymbol U_t}\right)^c \left(  \left( \boldsymbol I - \overline{\boldsymbol U^\star} \,\overline{({\boldsymbol U}^\star)^c}\right) \overline{\boldsymbol U_t} \right)      \right)     } \nonumber \\
& \overset{(d)}{ = } \sigma_{\min}^\star \sqrt{ 1 - \text{Dis}^2(\boldsymbol{\mathcal U}_t,\boldsymbol {\mathcal U}^\star)}, \label{second-lower}
\end{align}
where (a) is because any two compatible tall matrices $\boldsymbol A, \boldsymbol B$, there is $\sigma_{\min}(\boldsymbol A \boldsymbol B) \geq \sigma_{\min} (\boldsymbol A) \sigma_{\min}(\boldsymbol B)$. (b) is due to Lemma \ref{lemma9} . (c) is due to the fact that $ \boldsymbol (I - \overline{\boldsymbol U^\star}\, \overline{(\boldsymbol U^\star)^c})^2 = \boldsymbol I - \overline{\boldsymbol U^\star}\, \overline{(\boldsymbol U^\star)^c} $ and $(d)$ is the definition of $\text{Dis}(\boldsymbol{\mathcal U}_t, \boldsymbol{\mathcal U}^\star)$. 

Finally, combining the \eqref{first-lower} with \eqref{second-lower}, we have
\begin{align}
\sigma_{\min}(\boldsymbol{\mathcal V}_t) \geq \sigma_{\min}^\star \sqrt{ 1 - \text{Dis}^2(\boldsymbol{\mathcal U}_t,\boldsymbol {\mathcal U}^\star)} - 0.36 \text{Dis}^2(\boldsymbol{\mathcal U}_t,\boldsymbol {\mathcal U}^\star) \sqrt r \left\| \boldsymbol{\mathcal X}^\star \right\| \geq 0.94 \sigma_{\min}^\star. 
\end{align}
The last inequality we have used condition that $\text{Dis}(\boldsymbol{\mathcal U}_t, \boldsymbol{\mathcal U}^\star) \leq \frac{0.016}{\sqrt r \kappa^2 }$.\\ 

\noindent (6) $\left\|  \boldsymbol{\mathcal V}_t \right\| = \left\| \overline{\boldsymbol V_t }\right\| \leq \left\| \overline{\boldsymbol G_t}\right\| + \left\| \overline{\boldsymbol V_t } - \overline{\boldsymbol G_t} \right\|_F \leq \left\| \boldsymbol{\mathcal X}^\star\right\| + 0.36 \text{Dis}(\boldsymbol{\mathcal U}_t,\boldsymbol{\mathcal U}^\star) \sqrt r \left\| \boldsymbol{\mathcal X}^\star\right\| \leq 1.1 \left\| \boldsymbol{\mathcal X}^\star\right\|$. The second inequality uses the fact in (3).  
\end{proof}

 Before establishing the contraction of projected gradient descent, we first need to analyze the gradient properties, as detailed in the following lemma. A thorough analysis of the concentration of the empirical gradient around the expected gradient is essential for achieving contraction. The key technical aspect of this analysis involves leveraging the unique property of circular convolution operator in the low tubal rank tensor model to establish a tight concentration bound.

 \begin{lemma} \label{lemma10}
 The expectation and the concentration of gradient with respect to $\boldsymbol{\mathcal U}_t$ in Alt-PGD-Min method are as follows
 
  \noindent(1) $\mathbb E \left[ \nabla_{\boldsymbol{\mathcal U}^t} f \right] = m ( \boldsymbol{\mathcal X}_t - \boldsymbol{\mathcal X}^\star)*\boldsymbol{\mathcal V}_t^c$. 

  \noindent(2) If $\text{Dis}(\boldsymbol{\mathcal U}_t, \boldsymbol{\mathcal U}^\star) \leq \frac{0.016}{\sqrt r \kappa^2 }$, then 
  \begin{align}
     \left\| \mathbb E\left[ \nabla_{\boldsymbol{\mathcal U}^t} f  \right] \right\| \leq 1.36 \text{Dis}(\boldsymbol{\mathcal U}_t, \boldsymbol{\mathcal U}^\star)\sqrt r m \left\| \boldsymbol{\mathcal X}^\star \right\|^2 
  \end{align}
  holds with probability at least $1- \exp \left( \log n_2  +r \log n_3 - c_6 m \right)$.

  \noindent(3) With probability at least $1- \exp \left(c_4\log n_3 (n_1+r) - \frac{c_5 mn_2}{\kappa^4\mu^2rn_3}\right) -\exp \left( \log n_2  +r \log n_3 - c_6m \right)$, there is
  \begin{align}
    \left\|  \nabla_{\boldsymbol{\mathcal U}^t} f - \mathbb E \left[  \nabla_{\boldsymbol{\mathcal U}^t} f \right]    \right\| \leq 1.4 \epsilon_1 \text{Dis}(\boldsymbol{\mathcal U}_t, \boldsymbol{\mathcal U}^\star) m \left( \sigma_{\min}^\star \right)^2. 
  \end{align}
\end{lemma}

 \begin{proof}
(1) Based on the formulation of gradient, we have
\begin{align}
\mathbb E \left[ \nabla_{\boldsymbol{\mathcal U}^t} f  \right] & = \mathbb E \left[ \sum_{i=1}^{n_2} \left( \sum_{j=1}^m \langle \boldsymbol{\mathcal A}_i(j), \boldsymbol{\mathcal X}_t(i) - \boldsymbol{\mathcal X}^\star(i) \rangle \boldsymbol{\mathcal A}_i(j) \right)*\boldsymbol{\mathcal V}_t(i)    \right] \nonumber \\
& = m \sum_{i=1}^{n_2} ( \boldsymbol{\mathcal X}_t(i)- \boldsymbol{\mathcal X}^\star(i))* \boldsymbol{\mathcal V}_t(i) \nonumber \\
& = m ( \boldsymbol{\mathcal X}_t- \boldsymbol{\mathcal X}^\star )* \boldsymbol{\mathcal V}_t^c,
\end{align}
where the second inequality is due to definition of T-product and first equality is because the sample-splitting can make $\boldsymbol{\mathcal A}_i$ be independent on $\boldsymbol{\mathcal U}_t, \boldsymbol{\mathcal V}_t$, which can guarantee that $\forall p \in [n_1], \forall q \in [n_3]$, there is 
\begin{align}
\left[  \mathbb E \left[  \langle \boldsymbol{\mathcal A}_i(j),\boldsymbol{\mathcal X}_t(i) - \boldsymbol{\mathcal X}^\star(i) \rangle \boldsymbol{\mathcal A}_i(j) \right] \right]_{p,1,q} &= \mathbb E \left[    \left( \sum_{\tau,k} \boldsymbol{\mathcal A}_i(\tau,j,k) ( \boldsymbol{\mathcal X}_t(\tau,i,k) - \boldsymbol{\mathcal X}^\star(\tau,i,k) ) \right) \boldsymbol{\mathcal A}_i(p,j,q)   \right] \nonumber \\
& = \boldsymbol{\mathcal X}_t(p,i,q) - \boldsymbol{\mathcal X}^\star(p,i,q),
\end{align}
where the last equality due to $\boldsymbol{\mathcal A}_i$ has i.i.d. standard Gaussian entries. \\

\noindent (2) Base on (1), we have $\left\| \mathbb E\left[ \nabla_{\boldsymbol{\mathcal U}^t} f  \right] \right\| \leq m \left\| \boldsymbol{\mathcal X}_t - \boldsymbol{\mathcal X}^\star \right\|_F \left\| \boldsymbol{\mathcal V}_t^c \right\| \leq 1.5m \sqrt r \text{Dis}(\boldsymbol{\mathcal U}_t, \boldsymbol{\mathcal U}^\star) \left\| \boldsymbol{\mathcal X}^\star \right\|^2$, where the last inequality is due to fact (4) and fact (6) in Lemma \ref{lemma8}. \\

\noindent(3) We use the $\epsilon$-net argument for the sub-exponential Berstein inequality to bound the spectral norm of concentration. Based on the variational form of spectral norm, we have
\begin{align}
\left\|  \nabla_{\boldsymbol{\mathcal U}^t} f - \mathbb E \left[  \nabla_{\boldsymbol{\mathcal U}^t} f \right]    \right\| &= \max_{\boldsymbol w \in \boldsymbol \Theta^{n_1}, \boldsymbol z \in \boldsymbol \Theta^r} \sum_{i,j} \boldsymbol w^c \left(  \langle \boldsymbol{\mathcal A}_i(j), \boldsymbol{\mathcal X}_t(i) - \boldsymbol{\mathcal X}^\star(i) \rangle \cdot \overline{\boldsymbol{\mathcal A}_i(j)} \, \overline{ \boldsymbol{\mathcal V}_{t}^c(i)}\right) \boldsymbol z - \boldsymbol w^c \mathbb E \left[  \nabla_{\boldsymbol{\mathcal U}^t} f \right]    \boldsymbol z \nonumber \\
& = \max_{\boldsymbol w \in \boldsymbol \Theta^{n_1}, \boldsymbol z \in \boldsymbol \Theta^r} \sum_{i,j} \langle \boldsymbol{\mathcal A}_i(j), \boldsymbol{\mathcal X}_t(i) - \boldsymbol{\mathcal X}^\star(i) \rangle \nonumber \\
& \quad \cdot \langle \left( \boldsymbol F_{n_3}\otimes \boldsymbol I_{n_1} \right) \text{bcirc}( \boldsymbol{\mathcal A}_i(j)* \boldsymbol{\mathcal V}_t^c(i) ) \left(\boldsymbol F_{n_3}^{-1}\otimes \boldsymbol I_r \right)   ,\boldsymbol w \boldsymbol z^c \rangle - \boldsymbol w^c \mathbb E \left[  \nabla_{\boldsymbol{\mathcal U}^t} f \right]    \boldsymbol z\nonumber \\
& = \max_{\boldsymbol w \in \boldsymbol \Theta^{n_1}, \boldsymbol z \in \boldsymbol \Theta^r} \sum_{i,j} \langle \boldsymbol{\mathcal A}_i(j), \boldsymbol{\mathcal X}_t(i) - \boldsymbol{\mathcal X}^\star(i) \rangle \nonumber \\
& \quad \cdot \langle  \boldsymbol{\mathcal A}_i(j),  \text{bcirc}^* \left( \left( \boldsymbol F_{n_3}^{-1} \otimes \boldsymbol I_{n_1}\right) \boldsymbol w \boldsymbol z^c \left( \boldsymbol F_{n_3} \otimes \boldsymbol I_r\right)\right)* \boldsymbol{\mathcal V}_t(i)  \rangle - \boldsymbol w^c \mathbb E \left[  \nabla_{\boldsymbol{\mathcal U}^t} f \right]    \boldsymbol z.
\end{align}
For given fixed $\boldsymbol w, \boldsymbol z$, the summation is conducted to independent, zero-mean, sub-exponential random variables, each having a sub-exponential norm denoted by $K_{ji}$ has been bounded as
\begin{align}
K_{ji} \leq \left\|  \boldsymbol{\mathcal X}_t(i) - \boldsymbol{\mathcal X}^\star(i)\right\|_F \left\|  \text{bcirc}^* \left( \left( \boldsymbol F_{n_3}^{-1} \otimes \boldsymbol I_{n_1}\right) \boldsymbol w \boldsymbol z^c \left( \boldsymbol F_{n_3} \otimes \boldsymbol I_r\right)\right)* \boldsymbol{\mathcal V}_t(i)   \right\|_F.
\end{align}
Based on the above bound, we have
\begin{align}
\frac{t^2}{\sum_{i=1}^{n_2} \sum_{j=1}^m  K_{ji}^2} & \geq \frac{t^2}{ m \sum_{i=1}^{n_2} \left\|  \boldsymbol{\mathcal X}_t(i) - \boldsymbol{\mathcal X}^\star(i)\right\|_F^2 \left\|  \text{bcirc}^* \left( \left( \boldsymbol F_{n_3}^{-1} \otimes \boldsymbol I_{n_1}\right) \boldsymbol w \boldsymbol z^c \left( \boldsymbol F_{n_3} \otimes \boldsymbol I_r\right)\right)* \boldsymbol{\mathcal V}_t(i)   \right\|_F^2      } \nonumber \\
& \geq \frac{t^2}{m \max_i \left\|  \boldsymbol{\mathcal X}_t(i) - \boldsymbol{\mathcal X}^\star(i)  \right\|_F^2 \sum_{i=1}^{n_2}  \left\|  \text{bcirc}^* \left( \left( \boldsymbol F_{n_3}^{-1} \otimes \boldsymbol I_{n_1}\right) \boldsymbol w \boldsymbol z^c \left( \boldsymbol F_{n_3} \otimes \boldsymbol I_r\right)\right)* \boldsymbol{\mathcal V}_t(i)   \right\|_F^2     } \nonumber \\
& = \frac{t^2}{m \max_i \left\|  \boldsymbol{\mathcal X}_t(i) - \boldsymbol{\mathcal X}^\star(i)  \right\|_F^2 \left\|  \text{bcirc}^* \left( \left( \boldsymbol F_{n_3}^{-1} \otimes \boldsymbol I_{n_1}\right) \boldsymbol w \boldsymbol z^c \left( \boldsymbol F_{n_3} \otimes \boldsymbol I_r\right)\right)* \boldsymbol{\mathcal V}_t   \right\|_F^2     } \nonumber \\
& \overset{(a)}{\geq}\frac{n_2t^2}{ 1.36^2 m \text{Dis}^2 (\boldsymbol{\mathcal U}^t, \boldsymbol{\mathcal U}^\star) \mu^2 r \left\| \boldsymbol{\mathcal X}^\star \right\|^2  \left\| \boldsymbol{\mathcal V}_t \right\|^2\left\|\text{bcirc}^* \left( \left( \boldsymbol F_{n_3}^{-1} \otimes \boldsymbol I_{n_1}\right) \boldsymbol w \boldsymbol z^c \left( \boldsymbol F_{n_3} \otimes \boldsymbol I_r\right)\right)  \right\|_F^2  } \nonumber \\
& \overset{(b)}{ = } \frac{n_2t^2}{ 1.36^2 m n_3 r  \mu^2\text{Dis}^2(\boldsymbol{\mathcal U}^t, \boldsymbol{\mathcal U}^\star) \left\| \boldsymbol{\mathcal X}^\star \right\|^2 \left\| \boldsymbol{\mathcal V}_t \right\|^2  } \nonumber \\
& \overset{(c)}{\geq } \frac{n_2t^2}{ 1.1^2 \times 1.36^2 m n_3 r  \mu^2\text{Dis}^2(\boldsymbol{\mathcal U}^t, \boldsymbol{\mathcal U}^\star) \left\| \boldsymbol{\mathcal X}^\star \right\|^4  } \nonumber \\
& \overset{(d)}{ = }\frac{\epsilon_1^2 mn_2}{ 1.1^2 \times 1.36^2 n_3 \kappa^4\mu^2r}, \nonumber \\
\frac{t}{\max_{i,j} K_{ji}} & \overset{(e)}{\geq } \frac{\epsilon_1 \text{Dis}(\boldsymbol{\mathcal U}^t, \boldsymbol{\mathcal U}^\star) m \left( \sigma^\star_{\min} \right)^2  }{ \sqrt  {n_3} \max_i \left\|  \boldsymbol{\mathcal V}_t(i) \right\|_F \max_i \left\| \boldsymbol{\mathcal X}_t(i) - \boldsymbol{\mathcal X}^\star(i) \right\|_F } \nonumber \\
& \overset{(f)}{\geq } \frac{ \epsilon_1 \text{Dis}(\boldsymbol{\mathcal U}^t, \boldsymbol{\mathcal U}^\star) m (\sigma^\star_{\min})^2 }{1.1 \times 1.36 \sqrt{n_3}\text{Dis}(\boldsymbol{\mathcal U}^t, \boldsymbol{\mathcal U}^\star) \left\| \boldsymbol{\mathcal V}^\star(i) \right\|_F^2} \nonumber \\
& \overset{(g)}{\geq} \frac{\epsilon_1 mn_2}{1.1 \times 1.36 \kappa^2\mu^2 r \sqrt{n_3}}, \label{sub-exponen-norm}
\end{align}
where (a) is based on the fact (2) in Lemma \ref{lemma8} and Assumption 1. (b) is based on the property of bcirc operator and $\left\| \boldsymbol w \boldsymbol z^c \right\| = 1$. (c) is due to the fact (6) in Lemma \ref{lemma8}. (d) and (e) are due to setting $t:= \epsilon_1 \text{Dis}(\boldsymbol{\mathcal U}^t, \boldsymbol{\mathcal U}^\star)m(\sigma^\star_{\min})^2$. (f) is based on fact (1) and fact (2) in Lemma \ref{lemma8}. (g) is due to Assumption 1. Based on the sub-exponential Bernstein and union bound, for fixed $\boldsymbol w \in \boldsymbol \Theta^{n_1}, \boldsymbol z \in \boldsymbol \Theta^r$, we have 
\begin{align}
\boldsymbol w^c \left(\nabla_{\boldsymbol{\mathcal U}^t} f - \mathbb E \left[  \nabla_{\boldsymbol{\mathcal U}^t} f \right] \right) \boldsymbol z \leq \epsilon_1 \text{Dis}(\boldsymbol{\mathcal U}^t, \boldsymbol{\mathcal U}^\star)m(\sigma^\star_{\min})^2 \label{con-expect}
\end{align}
with probability at least $1-\exp\left( -\frac{c_5mn_2}{\kappa^4\mu^2rn_3}  \right) - \exp\left( \log n_2 + r\log n_3 - c_6m \right)$.
 Finally, because the covering number of $\boldsymbol \Theta^{k}$ is order of $\Tilde{\mathcal O}\left( e^k\right)$, based on Lemma \ref{lemma5} and \eqref{con-expect}, we have finished the proof.  
\end{proof}
\noindent Now equipped with Lemma \ref{lemma8} and Lemma \ref{lemma10}, we can prove  Theorem 4. 
\begin{proof}
Due to the updating of $\boldsymbol{\mathcal U}$ involving QR decomposition $\hat{\boldsymbol{\mathcal U}}_{t+1} = \boldsymbol{\mathcal U}_t - \nabla_{\boldsymbol{\mathcal U}_t} f = \boldsymbol{\mathcal U}_{t+1} * \boldsymbol{\mathcal R}_{t+1}$, we can rewrite that $\boldsymbol{\mathcal U}_{t+1} = \hat{\boldsymbol{\mathcal U}}_{t+1} * \boldsymbol{\mathcal R}_{t+1}^{-1}$. Next, we prove the contraction of $\text{Dis}(\boldsymbol{\mathcal U}_t, \boldsymbol{\mathcal U}^\star)$. First, we have
\begin{align}
\text{Dis}(\boldsymbol{\mathcal U}_{t+1}, \boldsymbol{\mathcal U}^\star) & = \left\| \left( \boldsymbol{\mathcal I} - \boldsymbol{\mathcal U}^\star * (\boldsymbol{\mathcal U}^\star)^c\right)*\hat{\boldsymbol{\mathcal U}}_{t+1}*\boldsymbol{\mathcal R}_{t+1}^{-1}  \right\| \nonumber \\
& \overset{(a)}{\leq } \frac{ \left\| \left( \boldsymbol{\mathcal I} - \boldsymbol{\mathcal U}^\star * (\boldsymbol{\mathcal U}^\star)^c\right)*\hat{\boldsymbol{\mathcal U}}_{t+1}\right\| }{ \sigma_{\min} \left( \hat{\boldsymbol{\mathcal U}}_{t+1} \right)} \nonumber \\
& = \frac{ \left\| \left( \boldsymbol{\mathcal I} - \boldsymbol{\mathcal U}^\star * (\boldsymbol{\mathcal U}^\star)^c\right)*\hat{\boldsymbol{\mathcal U}}_{t+1}\right\| }{ \sigma_{\min} \left( \boldsymbol{\mathcal U}_t - \eta \nabla_{\boldsymbol{\mathcal U}_t} f \right)} \nonumber \\
& = \frac{ \left\| \left( \boldsymbol{\mathcal I} - \boldsymbol{\mathcal U}^\star * (\boldsymbol{\mathcal U}^\star)^c\right)*\hat{\boldsymbol{\mathcal U}}_{t+1}\right\| }{ \sigma_{\min} \left(\text{bdiag} \left( \boldsymbol{\mathcal U}_t\right) - \eta \text{bdiag} \left(  \nabla_{\boldsymbol{\mathcal U}_t} f \right)   \right)} \nonumber \\
& \overset{(b)}{\leq} \frac{ \left\| \left( \boldsymbol{\mathcal I} - \boldsymbol{\mathcal U}^\star * (\boldsymbol{\mathcal U}^\star)^c\right)*\hat{\boldsymbol{\mathcal U}}_{t+1}\right\| }{ \sigma_{\min} \left( \text{bdiag} \left( \boldsymbol{\mathcal U}_t\right) \right) - \eta \left\| \text{bdiag} \left(  \nabla_{\boldsymbol{\mathcal U}_t} f \right)   \right\| } \nonumber \\
& = \frac{ \left\| \left( \boldsymbol{\mathcal I} - \boldsymbol{\mathcal U}^\star * (\boldsymbol{\mathcal U}^\star)^c\right)*\hat{\boldsymbol{\mathcal U}}_{t+1}\right\| }{ \sigma_{\min} \left(  \boldsymbol{\mathcal U}_t \right) - \eta \left\|   \nabla_{\boldsymbol{\mathcal U}_t} f    \right\| }, \label{recur-dis}
\end{align}
where (a) use the fact that $\left\| \boldsymbol{\mathcal R}_{t+1}^{-1} \right\| = \frac{1}{\sigma_{\min}(\hat{\boldsymbol{\mathcal U}}_{t+1})}$, which is because $ \left\| \boldsymbol{\mathcal R}_{t+1}^{-1} \right\| = \max_i \left\| \overline{\boldsymbol{\mathcal R}_{t+1}^{-1}}(:,:,i) \right\| =  \max_i \left\| \left(\overline{\boldsymbol{\mathcal R}_{t+1}}(:,:,i)\right)^{-1}\right\| = \frac{1}{\min_i  \sigma_{\min}\left(\overline{\boldsymbol{\mathcal R}_{t+1}}(:,:,i)\right)} = \frac{1}{\min_i \sigma_{\min}\left( \overline{\hat{\boldsymbol{\mathcal U}}_{t+1}}(:,:,i) \right)} = \frac{1}{\sigma_{\min}(\hat{\boldsymbol{\mathcal U}}_{t+1})}$. (b) uses the Weyl's inequality that $ \sigma_{\min} \left(\text{bdiag} \left( \boldsymbol{\mathcal U}_t\right) - \eta \text{bdiag} \left(  \nabla_{\boldsymbol{\mathcal U}_t} f \right)   \right) - \sigma_{\min}( \text{bdiag} \left( \boldsymbol{\mathcal U}_t\right) ) \geq -\eta \left\| \text{bdiag} \left(  \nabla_{\boldsymbol{\mathcal U}_t} f \right) \right\|$. 

Next, based on the decomposition $\hat{\boldsymbol{\mathcal U}}_{t+1} = \boldsymbol{\mathcal U}_t - \eta \mathbb E\left[ \nabla_{\boldsymbol{\mathcal U}_t } f\right] + \eta \left( \mathbb E\left[ \nabla_{\boldsymbol{\mathcal U}_t } f\right] -  \nabla_{\boldsymbol{\mathcal U}_t } f  \right)$, we have
\begin{align}
\left(\boldsymbol{\mathcal I} - \boldsymbol{\mathcal U}^\star* \left(\boldsymbol{\mathcal U}^\star\right)^c \right)* \hat{\boldsymbol{\mathcal U}} _{t+1} & = \left(\boldsymbol{\mathcal I} - \boldsymbol{\mathcal U}^\star* \left( \boldsymbol{\mathcal U}^\star \right)^c  \right)* \left(  \boldsymbol{\mathcal U}_t - \eta m \left( \boldsymbol{\mathcal U}_t* \boldsymbol{\mathcal V}_t - \boldsymbol{\mathcal U}^\star * \boldsymbol{\mathcal V}^\star \right)*\boldsymbol{\mathcal V}^c_t  \right) \nonumber \\
& \quad + \eta  \left(\boldsymbol{\mathcal I} - \boldsymbol{\mathcal U}^\star* \left( \boldsymbol{\mathcal U}^\star \right)^c \right)* \left( \mathbb E\left[ \nabla_{\boldsymbol{\mathcal U}_t } f\right] -  \nabla_{\boldsymbol{\mathcal U}_t } f  \right) \nonumber \\
& = \left(\boldsymbol{\mathcal I} - \boldsymbol{\mathcal U}^\star* \left( \boldsymbol{\mathcal U}^\star \right)^c  \right)* \boldsymbol{\mathcal U}_t*\left( \boldsymbol{\mathcal I} - \eta m \boldsymbol{\mathcal V}_t * \boldsymbol{\mathcal V}_t^c \right) \nonumber \\
& \quad + \eta  \left(\boldsymbol{\mathcal I} - \boldsymbol{\mathcal U}^\star* \left( \boldsymbol{\mathcal U}^\star \right)^c \right)* \left( \mathbb E\left[ \nabla_{\boldsymbol{\mathcal U}_t } f\right] -  \nabla_{\boldsymbol{\mathcal U}_t } f  \right), \label{for-concen}
\end{align}
where the first equality is substituting fact (1) in Lemma \ref{lemma10} and second equality is due to $\left(\boldsymbol{\mathcal I} - \boldsymbol{\mathcal U}^\star* \left( \boldsymbol{\mathcal U}^\star \right)^c  \right)*\boldsymbol{\mathcal U}^\star = \boldsymbol{\mathcal O}$. To guarantee that $\left( \boldsymbol{\mathcal I} - \eta m \boldsymbol{\mathcal V}_t * \boldsymbol{\mathcal V}_t^c \right)$ is the positive semi-definite tensor such that $ \forall 
 \boldsymbol{\mathcal W} \in \mathbb R^{r\times 1 \times n_3}$, there is $\langle \boldsymbol{\mathcal W},  \left(\boldsymbol{\mathcal I} - \boldsymbol{\mathcal U}^\star* \left( \boldsymbol{\mathcal U}^\star \right)^c  \right)* \boldsymbol{\mathcal W}\rangle \geq 0 $, which indicates that $\langle\overline{\boldsymbol W} , \left( \boldsymbol I - \eta m \overline{\boldsymbol V_t} \left(\overline{\boldsymbol V_t}\right)^c\right) \overline{\boldsymbol W}\rangle \geq 0$. This requires that the minimal eigenvalue $\lambda_{\min} \left((1-\eta m\overline{\boldsymbol V_t}^{(i)} \left( \overline{\boldsymbol V_t}^{(i)} \right)^c) \right) \geq 0 $, which means we should select step size $\eta$ to satisfy $1-\eta m \max_i \sigma_{\max}^2 \left( \overline{\boldsymbol V_t}^{(i)}\right) = 1- \eta m \left\| \boldsymbol{\mathcal V}_t \right\|^2 \geq 0$. Based on the fact (6) in Lemma \ref{lemma8}, we can set $\eta \leq \frac{0.9}{m \left\| \boldsymbol{\mathcal X}^\star \right\|^2}$ to guarantee that $\left( \boldsymbol{\mathcal I} - \eta m \boldsymbol{\mathcal V}_t * \boldsymbol{\mathcal V}_t^c \right)$ is PSD tensor. Thus, we have
\begin{align}
\left\|  \boldsymbol{\mathcal I} - \eta m \boldsymbol{\mathcal V}_t * \boldsymbol{\mathcal V}_t^c \right\|  & = \max_i \left\|  1 - \eta m \overline{\boldsymbol V_t}^{(i)} \left( \overline{\boldsymbol V_t}^{(i)} \right)^c \right\|  \nonumber \\
& = \max_i \left( 1 - \eta m \sigma_{\min}^2 \left( \overline{\boldsymbol V_t}^{(i)} \right)  \right) \nonumber \\
& = 1 - \eta m \min_i \sigma_{\min}^2 \left( \overline{\boldsymbol V_t}^{(i)} \right) \nonumber \\
& = 1 -\eta m \sigma_{\min}^2(\boldsymbol{\mathcal V}_t) \nonumber \\
& \leq 1- 0.94 \eta m  \left(\sigma_{\min}^\star \right)^2, \label{con-norm}
\end{align}
where  the last equality and inequality are due to definition of minimal singular value of tensor and the fact (5) in Lemma \ref{lemma8}, respectively. 

Finally, integrating results in \eqref{recur-dis}, \eqref{for-concen} with \eqref{con-norm} would have
\begin{align}
\text{Dis}(\boldsymbol{\mathcal U}_{t+1}, \boldsymbol{\mathcal U}^\star) & \leq \frac{ \left\| \left( \boldsymbol{\mathcal I} - \boldsymbol{\mathcal U}^\star * (\boldsymbol{\mathcal U}^\star)^c\right)*\hat{\boldsymbol{\mathcal U}}_{t+1}\right\| }{ \sigma_{\min} \left(  \boldsymbol{\mathcal U}_t \right) - \eta \left\|   \nabla_{\boldsymbol{\mathcal U}_t} f    \right\| } \nonumber \\
& \overset{(a)}{ \leq } \frac{ \text{Dis}(\boldsymbol{\mathcal U}_t, \boldsymbol{\mathcal U}^\star)\cdot \left( 1- 0.94 \eta m  \left(\sigma_{\min}^\star \right)^2 \right) + 1.4\epsilon_1 \eta m \text{Dis}(\boldsymbol{\mathcal U}_t, \boldsymbol{\mathcal U}^\star) \left(\sigma_{\min}^\star \right)^2  }{ \sigma_{\min} \left(  \boldsymbol{\mathcal U}_t \right) - \eta \left\|   \nabla_{\boldsymbol{\mathcal U}_t} f    \right\| } \nonumber \\
& \overset{(b)}{ = } \frac{ \text{Dis}(\boldsymbol{\mathcal U}_t, \boldsymbol{\mathcal U}^\star) \cdot (1 - 0.89\eta m (\sigma_{\min}^\star)^2 )  }{1-   \eta \left\|   \nabla_{\boldsymbol{\mathcal U}_t} f   \right\| } \nonumber \\
& \leq  \frac{ \text{Dis}(\boldsymbol{\mathcal U}_t, \boldsymbol{\mathcal U}^\star) \cdot (1 - 0.89\eta m (\sigma_{\min}^\star)^2 )  }{1-  \eta \left\| \mathbb E \left[ \nabla_{\boldsymbol{\mathcal U}_t} f  \right]\right\| - \eta \left\|   \nabla_{\boldsymbol{\mathcal U}_t} f - \mathbb E \left[ \nabla_{\boldsymbol{\mathcal U}_t} f  \right] \right\| } \nonumber \\
& \overset{(c)}{ \leq } \frac{ \text{Dis}(\boldsymbol{\mathcal U}_t, \boldsymbol{\mathcal U}^\star) \cdot (1 - 0.89\eta m (\sigma_{\min}^\star)^2 )  } { 1 - 1.36\eta \sqrt r m \text{Dis}(\boldsymbol{\mathcal U}_t, \boldsymbol{\mathcal U}^\star) \left\| \boldsymbol{\mathcal X}^\star \right\|^2 - \frac 1 {20} \eta m (\sigma_{\min}^\star )^2 \text{Dis}(\boldsymbol{\mathcal U}_t, \boldsymbol{\mathcal U}^\star) } \nonumber \\
& \leq  \frac{ \text{Dis}(\boldsymbol{\mathcal U}_t, \boldsymbol{\mathcal U}^\star) \cdot (1 - 0.89\eta m (\sigma_{\min}^\star)^2 )  } { 1 - 1.41\eta m \sqrt r \left\| \boldsymbol{\mathcal X}^\star \right\|^2 \text{Dis}(\boldsymbol{\mathcal U}_t, \boldsymbol{\mathcal U}^\star) } \nonumber \\
& \overset{(d)}{\leq } \text{Dis}(\boldsymbol{\mathcal U}_t, \boldsymbol{\mathcal U}^\star) \cdot (1 - 0.89\eta m (\sigma_{\min}^\star)^2 ) (1 + 2.82 \eta m \sqrt r \left\| \boldsymbol{\mathcal X}^\star \right\|^2 \text{Dis}(\boldsymbol{\mathcal U}_t, \boldsymbol{\mathcal U}^\star) ) \nonumber \\
& \overset{(e)}{ \leq } \text{Dis}(\boldsymbol{\mathcal U}_t, \boldsymbol{\mathcal U}^\star) \left( 1- 0.84 \eta m (\sigma_{\min}^\star )^2\right) \nonumber \\
& \overset{(f)}{  = } \left( 1 - \frac{0.84c_\eta}{\kappa^2}\right) \cdot \text{Dis}(\boldsymbol{\mathcal U}_t, \boldsymbol{\mathcal U}^\star)
\end{align}
with $c_\eta \leq 0.9$, where the (a) is due to \eqref{recur-dis}, \eqref{for-concen},  \eqref{con-norm} and the fact (3) in Lemma \ref{lemma10}. (b) is because we set $\epsilon_1 := \frac 1 {28}$ and $\sigma_{\min}(\boldsymbol{\mathcal U}_t) = 1$. (c) is substituting fact (2) and fact (3) in Lemma \ref{lemma10}. (d) is because $1.41\times \frac{0.9}{m \left\| \boldsymbol{\mathcal X}^\star \right\|^2} m \sqrt r \left\| \boldsymbol{\mathcal X}^\star \right\|^2 \frac{0.016}{\sqrt r \kappa^2} \leq 0.03$ and $\forall 0\leq x\leq \frac 1 2 $, there is $(1-x)(1+2x)\geq 1$. (e) is because for $ 0\leq x, y \leq 1,$ there is $(1-x)(1+y) \leq 1 -(x-y)$ and initialization condition $\text{Dis}(\boldsymbol{\mathcal U}_t, \boldsymbol{\mathcal U}^\star) \leq \frac{0.016}{\sqrt r \kappa^2}$ in Theorem 3. The second statement in Theorem 4 is directly obtained from the fact (2) in Lemma \ref{lemma8}. 
\end{proof}
 
\section{Proof of Theorem 5}\label{sec3}
Similar to the proof of Theorem 4, we begin by establishing the following facts about the preconditioned gradient. The essential distinction between the upcoming lemma and Lemma \ref{lemma10}  is that the upper bounds for the spectral norm of the population gradient and concentration in Lemma \ref{lemma10} are divided by $\sigma_{\min}^2$. This difference is due to the preconditioning operator and is crucial for achieving iteration complexity independent of $\kappa$.

 \begin{lemma} \label{lemma11}
The expectation and the concentration of scale gradient $\nabla_{\boldsymbol{\mathcal U}_t}^s f$ with respect to $\boldsymbol{\mathcal U}_t$ in Alt-Scale PGD-Min method are as follows
 
  \noindent(1) $\mathbb E \left[ \nabla_{\boldsymbol{\mathcal U}^t}^s f \right] = m ( \boldsymbol{\mathcal X}_t - \boldsymbol{\mathcal X}^\star)*\boldsymbol{\mathcal V}_t^c*( \boldsymbol{\mathcal V}_t *  \boldsymbol{\mathcal V}_t^c )^{-1}$. 

  \noindent(2) If $\text{Dis}(\boldsymbol{\mathcal U}_t, \boldsymbol{\mathcal U}^\star) \leq \frac{0.016}{\sqrt r \kappa^2 }$, then 
  \begin{align}
      \left\| \mathbb E\left[ \nabla_{\boldsymbol{\mathcal U}^t}^s f  \right] \right\| \leq 1.54 \text{Dis}(\boldsymbol{\mathcal U}_t, \boldsymbol{\mathcal U}^\star)\sqrt r m \kappa^2 
  \end{align}
  holds with probability at least $1- \exp \left( \log n_2  +r \log n_3 - c_6 m \right)$.

  \noindent(3) With probability at least $1- \exp \left(c_4\log n_3 (n_1+r) - \frac{c_5 mn_2}{\kappa^4\mu^2rn_3}\right) -\exp \left( \log n_2  +r \log n_3 - c_6m \right)$, there is
  \begin{align}
    \left\|  \nabla_{\boldsymbol{\mathcal U}^t}^s  f - \mathbb E \left[  \nabla_{\boldsymbol{\mathcal U}^t}^s f \right]    \right\| \leq 1.4 \epsilon_1 \text{Dis}(\boldsymbol{\mathcal U}_t, \boldsymbol{\mathcal U}^\star) m. 
  \end{align}
 \end{lemma}

\begin{proof}
(1) The proof is directly derived from proof of (1) in Lemma \ref{lemma10} due to the sample-splitting that $\mathbb E [ \nabla_{\boldsymbol{\mathcal U}^t}^s  f ] = \mathbb E [ \nabla_{\boldsymbol{\mathcal U}^t} f ]* ( \boldsymbol{\mathcal V}_t *  \boldsymbol{\mathcal V}_t^c )^{-1}$.\\

\noindent (2) The proof is based on proof of (2) in Lemma \ref{lemma10} that 
\begin{align}
\left\| \mathbb E\left[ \nabla_{\boldsymbol{\mathcal U}^t}^s f  \right] \right\|  & \leq \left\| \mathbb E\left[ \nabla_{\boldsymbol{\mathcal U}^t} f  \right] \right\| * \left\|  ( \boldsymbol{\mathcal V}_t *  \boldsymbol{\mathcal V}_t^c )^{-1} \right\| \nonumber \\
& \leq \frac{1.36 \text{Dis}(\boldsymbol{\mathcal U}_t, \boldsymbol{\mathcal U}^\star)\sqrt r m \left\| \boldsymbol{\mathcal X}^\star \right\|^2 }{ \sigma_{\min}^2(\boldsymbol{\mathcal V}_t)  } \nonumber \\
& \leq \frac{  1.36 \text{Dis}(\boldsymbol{\mathcal U}_t, \boldsymbol{\mathcal U}^\star)\sqrt r m \left\| \boldsymbol{\mathcal X}^\star \right\|^2  }{ 0.94^2 (\sigma_{\min}^\star)2 },
\end{align}
where the last inequality is due to the fact (5) in Lemma \ref{lemma8}.\\

\noindent(3) We also use the $\epsilon$-net version of sub-exponential Berstein inequality to control the concentration, the difference here is that we can use a tighter bound than that in vanilla gradient.  

\begin{align}
\left\|  \nabla_{\boldsymbol{\mathcal U}^t} f - \mathbb E \left[  \nabla_{\boldsymbol{\mathcal U}^t} f \right]    \right\| &= \max_{\boldsymbol w \in \boldsymbol \Theta^{n_1}, \boldsymbol z \in \boldsymbol \Theta^r} \sum_{i,j} \boldsymbol w^c \left(  \langle \boldsymbol{\mathcal A}_i(j), \boldsymbol{\mathcal X}_t(i) - \boldsymbol{\mathcal X}^\star(i) \rangle \cdot \overline{\boldsymbol{\mathcal A}_i(j)} \, \overline{ \boldsymbol{\mathcal V}_{t}^c(i)}  \cdot \left( \overline{\boldsymbol V_t} \, \overline{\boldsymbol V_t^c} \right)^{-1} \right) \boldsymbol z \nonumber \\
& \quad  - \boldsymbol w^c \mathbb E \left[  \nabla_{\boldsymbol{\mathcal U}^t} f \right]    \boldsymbol z \nonumber \\
& = \max_{\boldsymbol w \in \boldsymbol \Theta^{n_1}, \boldsymbol z \in \boldsymbol \Theta^r} \sum_{i,j} \langle \boldsymbol{\mathcal A}_i(j), \boldsymbol{\mathcal X}_t(i) - \boldsymbol{\mathcal X}^\star(i) \rangle \nonumber \\
& \quad \cdot \langle \left( \boldsymbol F_{n_3}\otimes \boldsymbol I_{n_1} \right) \text{bcirc}( \boldsymbol{\mathcal A}_i(j)* \boldsymbol{\mathcal V}_t^c(i) *( \boldsymbol{\mathcal V}_t * \boldsymbol{\mathcal V}_t^c  )^{-1}  ) \left(\boldsymbol F_{n_3}^{-1}\otimes \boldsymbol I_r \right)   ,\boldsymbol w \boldsymbol z^c \rangle \nonumber \\
& \quad - \boldsymbol w^c \mathbb E \left[  \nabla_{\boldsymbol{\mathcal U}^t} f \right]    \boldsymbol z\nonumber \\
& = \max_{\boldsymbol w \in \boldsymbol \Theta^{n_1}, \boldsymbol z \in \boldsymbol \Theta^r} \sum_{i,j} \langle \boldsymbol{\mathcal A}_i(j), \boldsymbol{\mathcal X}_t(i) - \boldsymbol{\mathcal X}^\star(i) \rangle \nonumber \\
& \quad \cdot \langle  \boldsymbol{\mathcal A}_i(j),  \text{bcirc}^* \left( \left( \boldsymbol F_{n_3}^{-1} \otimes \boldsymbol I_{n_1}\right) \boldsymbol w \boldsymbol z^c \left( \boldsymbol F_{n_3} \otimes \boldsymbol I_r\right)\right)* \left( \boldsymbol{\mathcal V}_t * \boldsymbol{\mathcal V}_t^c \right)^{-1} * \boldsymbol{\mathcal V}_t(i)  \rangle  \nonumber \\
& \quad - \boldsymbol w^c \mathbb E \left[  \nabla_{\boldsymbol{\mathcal U}^t} f \right]    \boldsymbol z.
\end{align}
The sub-exponential norm for each independent zero-mean, sub-exponential random variable can be bounded as
\begin{align}
\hat{K}_{ji} \leq \left\|  \boldsymbol{\mathcal X}_t(i) - \boldsymbol{\mathcal X}^\star(i)\right\|_F \left\|  \text{bcirc}^* \left( \left( \boldsymbol F_{n_3}^{-1} \otimes \boldsymbol I_{n_1}\right) \boldsymbol w \boldsymbol z^c \left( \boldsymbol F_{n_3} \otimes \boldsymbol I_r\right)\right)*  \left( \boldsymbol{\mathcal V}_t * \boldsymbol{\mathcal V}_t^c \right)^{-1} *\boldsymbol{\mathcal V}_t(i)   \right\|_F. 
\end{align}
Then based on \eqref{sub-exponen-norm}, we have
\begin{align}
\frac{t^2}{\sum_{i=1}^{n_2} \sum_{j=1}^m  \hat K_{ji}^2} &\geq \frac{t^2} {\left\|  \text{bcirc}^* \left( \left( \boldsymbol F_{n_3}^{-1} \otimes \boldsymbol I_{n_1}\right) \boldsymbol w \boldsymbol z^c \left( \boldsymbol F_{n_3} \otimes \boldsymbol I_r\right)\right)*  \left( \boldsymbol{\mathcal V}_t * \boldsymbol{\mathcal V}_t^c \right)^{-1} *\boldsymbol{\mathcal V}_t(i)   \right\|_F^2    } \nonumber \\
& \quad \cdot \frac 1 {m \max_i \left\|  \boldsymbol{\mathcal X}_t(i) - \boldsymbol{\mathcal X}^\star(i)  \right\|_F^2 } \nonumber \\
& \geq \frac{n_2t^2}{ 1.36^2 m n_3 r  \mu^2\text{Dis}^2(\boldsymbol{\mathcal U}^t, \boldsymbol{\mathcal U}^\star) \left\| \boldsymbol{\mathcal X}^\star \right\|^2 \left\| \boldsymbol{\mathcal V}_t \right\|^2  \left\| \left( \boldsymbol{\mathcal V}_t * \boldsymbol{\mathcal V}_t^c \right)^{-1} \right\|^2 } \nonumber \\
& = \frac{n_2t^2 \sigma_{\min}^4 (\boldsymbol{\mathcal V}_t)  }{ 1.36^2 m n_3 r  \mu^2\text{Dis}^2(\boldsymbol{\mathcal U}^t, \boldsymbol{\mathcal U}^\star) \left\| \boldsymbol{\mathcal X}^\star \right\|^2 \left\| \boldsymbol{\mathcal V}_t \right\|^2  } \nonumber \\
& \overset{(a)}{\geq} \frac{0.94^4\epsilon_1^2 mn_2}{ 1.1^2 \times 1.36^2 n_3 \kappa^4\mu^2r} \nonumber \\
\frac{t}{\max_{i,j} \hat K_{ji}} & \overset{(b)}{\geq } \frac{\epsilon_1 \text{Dis}(\boldsymbol{\mathcal U}^t, \boldsymbol{\mathcal U}^\star) m   }{ \sqrt  {n_3} \max_i \left\|  \boldsymbol{\mathcal V}_t(i) \right\|_F \max_i \left\| \boldsymbol{\mathcal X}_t(i) - \boldsymbol{\mathcal X}^\star(i) \right\|_F \left\| \left( \boldsymbol{\mathcal V}_t * \boldsymbol{\mathcal V}_t^c \right)^{-1} \right\| } \nonumber \\
& = \frac{ \epsilon_1 \text{Dis}(\boldsymbol{\mathcal U}^t, \boldsymbol{\mathcal U}^\star) m (\sigma_{\min} \left( \boldsymbol{\mathcal V}_t \right) )^2 }{1.1 \times 1.36 \sqrt{n_3}\text{Dis}(\boldsymbol{\mathcal U}^t, \boldsymbol{\mathcal U}^\star) \left\| \boldsymbol{\mathcal V}^\star(i) \right\|_F^2} \nonumber \\
& \overset{(c)}{\geq } \frac{ 0.94^2 \epsilon_1 \text{Dis}(\boldsymbol{\mathcal U}^t, \boldsymbol{\mathcal U}^\star) m (\sigma^\star_{\min})^2 }{1.1 \times 1.36 \sqrt{n_3}\text{Dis}(\boldsymbol{\mathcal U}^t, \boldsymbol{\mathcal U}^\star) \left\| \boldsymbol{\mathcal V}^\star(i) \right\|_F^2} \nonumber \\
& \overset{(d)}{\geq} \frac{\epsilon_1 mn_2}{1.1 \times 1.36 \kappa^2\mu^2 r \sqrt{n_3}},
\end{align}
where (a) and (c) use the fact (5) in Lemma \ref{lemma8}. (a) and (b) are due to setting $t:= \epsilon_1 m \text{Dis}(\boldsymbol{\mathcal U}^t, \boldsymbol{\mathcal U}^\star)$, which is different with $t: =\epsilon_1 m \text{Dis}(\boldsymbol{\mathcal U}^t, \boldsymbol{\mathcal U}^\star)(\sigma_{\min}^\star)^2$ in proving Lemma \ref{lemma10}. (d) is due to the Assumption 1. We finish the proof by combining union bound and the same $\epsilon$-net argument as Alt-PGD-Min. 
\end{proof}

Based on Lemma \ref{lemma8} and the property of scale gradient in Lemma \ref{lemma11}, we can prove Theorem 5 by analogous technique in proof of Theorem 4.

\begin{proof}
Firstly, based on \eqref{recur-dis}, we have
\begin{align}
 \text{Dis}(\boldsymbol{\mathcal U}_{t+1}, \boldsymbol{\mathcal U}^\star) \leq \frac{ \left\| \left( \boldsymbol{\mathcal I} - \boldsymbol{\mathcal U}^\star * (\boldsymbol{\mathcal U}^\star)^c\right)*\hat{\boldsymbol{\mathcal U}}_{t+1}\right\| }{ \sigma_{\min} \left(  \boldsymbol{\mathcal U}_t \right) - \eta \left\|   \nabla_{\boldsymbol{\mathcal U}_t}^s f    \right\| }. \label{recur-scale-dis}
\end{align}
Due to the decomposition of scale gradient descent $\hat{\boldsymbol{\mathcal U}}_{t+1} = \boldsymbol{\mathcal U}_t - \eta \mathbb E\left[ \nabla_{\boldsymbol{\mathcal U}_t }^s f\right] + \eta \left( \mathbb E\left[ \nabla_{\boldsymbol{\mathcal U}_t }^s f\right] -  \nabla_{\boldsymbol{\mathcal U}_t }^s f  \right)$, we have
\begin{align}
\left(\boldsymbol{\mathcal I} - \boldsymbol{\mathcal U}^\star* \left(\boldsymbol{\mathcal U}^\star\right)^c \right)* \hat{\boldsymbol{\mathcal U}} _{t+1} & = \left(\boldsymbol{\mathcal I} - \boldsymbol{\mathcal U}^\star* \left( \boldsymbol{\mathcal U}^\star \right)^c  \right)* \left(  \boldsymbol{\mathcal U}_t - \eta m \left( \boldsymbol{\mathcal U}_t* \boldsymbol{\mathcal V}_t - \boldsymbol{\mathcal U}^\star * \boldsymbol{\mathcal V}^\star \right)*\boldsymbol{\mathcal V}^c_t * \left(  \boldsymbol{\mathcal V}_t * \boldsymbol{\mathcal V}_t^c \right)^{-1} \right) \nonumber \\
& \quad + \eta  \left(\boldsymbol{\mathcal I} - \boldsymbol{\mathcal U}^\star* \left( \boldsymbol{\mathcal U}^\star \right)^c \right)* \left( \mathbb E\left[ \nabla_{\boldsymbol{\mathcal U}_t }^s f\right] -  \nabla_{\boldsymbol{\mathcal U}_t }^s f  \right) \nonumber \\
& = \left(\boldsymbol{\mathcal I} - \boldsymbol{\mathcal U}^\star* \left( \boldsymbol{\mathcal U}^\star \right)^c  \right)* \boldsymbol{\mathcal U}_t*\left( 1 - \eta m  \right) \nonumber \\
& \quad + \eta  \left(\boldsymbol{\mathcal I} - \boldsymbol{\mathcal U}^\star* \left( \boldsymbol{\mathcal U}^\star \right)^c \right)* \left( \mathbb E\left[ \nabla_{\boldsymbol{\mathcal U}_t }^s f\right] -  \nabla_{\boldsymbol{\mathcal U}_t }^s f  \right), \label{decompo-scale}
\end{align}
where the first equality is substituting into the fact (1) in Lemma \ref{lemma11}. To guarantee the contraction of $ \text{Dis}(\boldsymbol{\mathcal U}_{t}, \boldsymbol{\mathcal U}^\star)$, we just need that step size $\eta \leq \frac{0.9}{m}$. Integrating with \eqref{recur-scale-dis}, there is
\begin{align}
\text{Dis}(\boldsymbol{\mathcal U}_{t+1}, \boldsymbol{\mathcal U}^\star) & \overset{(a)}{ \leq } \frac{ \text{Dis}(\boldsymbol{\mathcal U}_{t}, \boldsymbol{\mathcal U}^\star) \cdot (1-\eta m ) + 1.4 \epsilon_1\eta m \text{Dis} (\boldsymbol{\mathcal U}_{t}, \boldsymbol{\mathcal U}^\star)  }{ 1 - \eta  \left\|   \nabla_{\boldsymbol{\mathcal U}_t}^s f    \right\| } \nonumber \\
& = \frac{ \text{Dis}(\boldsymbol{\mathcal U}_{t}, \boldsymbol{\mathcal U}^\star) \cdot (1- 0.95 \eta m ) }{ 1 - \eta  \left\|   \nabla_{\boldsymbol{\mathcal U}_t}^s f    \right\| } \nonumber \\
& \leq  \frac{ \text{Dis}(\boldsymbol{\mathcal U}_t, \boldsymbol{\mathcal U}^\star) \cdot (1 - 0.95\eta m  )  }{1-  \eta \left\| \mathbb E \left[ \nabla_{\boldsymbol{\mathcal U}_t}^s f  \right]\right\| - \eta \left\|   \nabla_{\boldsymbol{\mathcal U}_t}^s f - \mathbb E \left[ \nabla_{\boldsymbol{\mathcal U}_t}^s f  \right] \right\| } \nonumber \\
& \overset{(b)}{ \leq } \frac{  \text{Dis}(\boldsymbol{\mathcal U}_t, \boldsymbol{\mathcal U}^\star) \cdot (1 - 0.95\eta m )   }{ 1 - 1.54 \eta \sqrt r m \text{Dis}(\boldsymbol{\mathcal U}_t, \boldsymbol{\mathcal U}^\star) \kappa^2 - \frac{1}{20} \eta m \text{Dis}(\boldsymbol{\mathcal U}_t, \boldsymbol{\mathcal U}^\star)     } \nonumber \\
& \leq  \frac{  \text{Dis}(\boldsymbol{\mathcal U}_t, \boldsymbol{\mathcal U}^\star) \cdot (1 - 0.95\eta m)   }{ 1 - 1.6 \eta \sqrt r m \text{Dis}(\boldsymbol{\mathcal U}_t, \boldsymbol{\mathcal U}^\star) \kappa^2 } \nonumber \\
& \overset{(c)}{\leq }  \text{Dis}(\boldsymbol{\mathcal U}_t, \boldsymbol{\mathcal U}^\star) \cdot (1-0.95\eta m )(1 + 3.2 \eta \sqrt r m \text{Dis}(\boldsymbol{\mathcal U}_t, \boldsymbol{\mathcal U}^\star) \kappa^2  ) \nonumber \\
& \leq  \text{Dis}(\boldsymbol{\mathcal U}_t, \boldsymbol{\mathcal U}^\star) \cdot (1-0.89\eta m ) \nonumber \\
& =  ( 1 - 0.89 c_\eta) \cdot \text{Dis}(\boldsymbol{\mathcal U}_t, \boldsymbol{\mathcal U}^\star) 
\end{align}
with $c_\eta \leq 0.9$. The (a) is due to \eqref{recur-scale-dis}, \eqref{decompo-scale} and the fact (3) in Lemma \ref{lemma11}. (b) is because of fact (1) and fact (2) in Lemma \ref{lemma11}. (c) is due to $\text{Dis}(\boldsymbol{\mathcal U}_t, \boldsymbol{\mathcal U}^\star) \leq \frac{0.016}{\sqrt r  \kappa^2}$, which can be satisfies by initialization in Theorem 3. 
\end{proof}

\section{Technical Lemmas}\label{sec4}
\begin{lemma} \label{lemm3}
For any third order tensor $\boldsymbol{\mathcal X} \in \mathbb R^{n_1 \times n_2 \times n_3}$ that $\boldsymbol{\mathcal X}(:,:,i) = \boldsymbol 0_{n_1 \times n_2}, \forall i = 2, \cdots, n_3$, then $\overline{\boldsymbol{\mathcal X}} = \text{fft}(\boldsymbol{\mathcal X},[\;],3)$ has property that $\overline{\boldsymbol X}^{(1)}=\overline{\boldsymbol X}^{(2)}=\cdots = \overline{\boldsymbol X}^{(n_3)} = \boldsymbol{\mathcal X}(:,:,1)$. 
\end{lemma}

\begin{proof}
This property is due to the property of FFT. The each tuber fiber is the vector that $\boldsymbol{\mathcal X}_{ij:} = [\boldsymbol{\mathcal X}_{ij1}; 0; \cdots ;0] \in \mathbb R^{n_3}$. Based on the definition of FFT for tensor
\begin{align}
\overline{\boldsymbol{\mathcal X}}_{ij(k+1)} & = \sum_{n=0}^{n_3-1} \boldsymbol{\mathcal X}_{ij(n+1)} \exp \left( - i 2\pi \frac{k}{n_3} n \right), \quad  \forall k = 0,\cdots, n_3-1 \nonumber \\
& = \boldsymbol{\mathcal X}_{ij1}.
\end{align}
\end{proof}

\begin{lemma} \label{lemma7}
For given two third order tensors $\boldsymbol{\mathcal A} \in \mathbb R^{n_1 \times n_2 \times n_3}, \boldsymbol{\mathcal B} \in \mathbb R^{n_2 \times n_4 \times n_3}$, then there is
\begin{align}
\text{bcirc} \left( \boldsymbol{\mathcal A} * \boldsymbol{\mathcal B}\right) = \text{bcirc} \left( \boldsymbol{\mathcal A}\right) \text{bcirc} \left( \boldsymbol{\mathcal B}\right).  
\end{align}
\end{lemma}
\begin{proof}
Based on the relationship between $\text{bicrc} \left( \boldsymbol{\mathcal A}\right)$ and $\overline{\boldsymbol{\mathcal A}}$, we have
\begin{align}
 \left( \boldsymbol F_{n_3} \otimes \boldsymbol I_{n_1}\right) \text{bcirc}\left( \boldsymbol{\mathcal A} * \boldsymbol{\mathcal B} \right) \left( \boldsymbol F_{n_3}^{-1} \otimes \boldsymbol I_{n_4}\right)  & = \text{bdiag} \left( \overline{  \boldsymbol{\mathcal A} * \boldsymbol{\mathcal B}  }\right) \nonumber \\
& = \overline{\boldsymbol A} \; \overline{\boldsymbol B} \nonumber \\
& = \left( \boldsymbol F_{n_3} \otimes \boldsymbol I_{n_1}\right) \text{bcirc}\left( \boldsymbol{\mathcal A} \right) \left( \boldsymbol F_{n_3}^{-1} \otimes \boldsymbol I_{n_2}\right) \nonumber \\
& \quad \cdot \left( \boldsymbol F_{n_3} \otimes \boldsymbol I_{n_2}\right) \text{bcirc}\left( \boldsymbol{\mathcal B} \right) \left( \boldsymbol F_{n_3}^{-1} \otimes \boldsymbol I_{n_4}\right) \nonumber \\
& =  \left( \boldsymbol F_{n_3} \otimes \boldsymbol I_{n_1}\right) \text{bcirc}\left( \boldsymbol{\mathcal A} \right) \text{bcirc}\left( \boldsymbol{\mathcal B} \right) \nonumber \\
& \quad \cdot \left( \boldsymbol F_{n_3}^{-1} \otimes \boldsymbol I_{n_4}\right),
\end{align}
where the last equality is due to Kronecker product property that $\left( \boldsymbol F_{n_3}^{-1} \otimes \boldsymbol I_{n_2}\right)\left( \boldsymbol F_{n_3} \otimes \boldsymbol I_{n_2}\right)$ = $\left( \boldsymbol F_{n_3}^{-1} \boldsymbol F_{n_3}\right) \otimes \boldsymbol I_{n_2} = \boldsymbol I_{n_2n_3}$ due to orthogonal property of $\left( \boldsymbol F_{n_3} \otimes \boldsymbol I_{n_2}\right)/\sqrt{n_3}$.  
\end{proof}

\begin{lemma}\label{lemma9}
For any third-order tensor $\boldsymbol{\mathcal X}$, there is $\left( \overline{\boldsymbol X^c} \right)^c = \overline{\boldsymbol X}$. 
\end{lemma}
\begin{proof}
The proof of this lemma is a direct combination of the definition of the conjugate operator and the equation (6) in \citep{rojo2004some}, which concludes that for any real vector $\forall v \in \mathbb R^{n_3}$, there is $\overline{v}_1 \in \mathbb R$ and $\overline{v}_i^c = \overline{v}_{n_3+2-i}, \forall i = 2,\cdots, \lfloor \frac{n_3+1} 2 \rfloor$. 
\end{proof}

\begin{lemma}\label{lemma14}
For any third order tensor $\boldsymbol{\mathcal X} \in \mathbb R^{n_1 \times n_2 \times n_3}$, there is
\begin{align}
\left( \text{bcirc} \left( \boldsymbol{\mathcal X} ^c\right) \right)^c = \text{bcirc} \left( \boldsymbol{\mathcal X}\right).   
\end{align}
\end{lemma}
\begin{proof}
Based on the relationship between $\text{bcirc}(\boldsymbol{\mathcal X})$ and $\overline{\boldsymbol{\mathcal X}}$, there are 
\begin{align}
&\left( \boldsymbol F_{n_3} \otimes \boldsymbol I_{n_1}\right) \text{bcirc}\left( \boldsymbol{\mathcal X} \right) \left( \boldsymbol F_{n_3}^{-1} \otimes \boldsymbol I_{n_4}\right)  = \overline{\boldsymbol X}, \label{bcirc1}\\
&\left( \boldsymbol F_{n_3} \otimes \boldsymbol I_{n_1}\right) \text{bcirc}\left( \boldsymbol{\mathcal X}^c \right) \left( \boldsymbol F_{n_3}^{-1} \otimes \boldsymbol I_{n_4}\right)  = \overline{\boldsymbol X^c}.   
\end{align}
Taking conjugate for both sides, there is
\begin{align}
 \left( \boldsymbol F_{n_3} \otimes \boldsymbol I_{n_1}\right) \left( \text{bcirc}\left( \boldsymbol{\mathcal X}^c \right) \right)^c \left( \boldsymbol F_{n_3}^{-1} \otimes \boldsymbol I_{n_4}\right)  = \left( \overline{\boldsymbol X^c} \right)^c = \overline{\boldsymbol X}, \label{bcirc2}  
\end{align}
where the second equality is due to Lemma \ref{lemma9}. Finally, \eqref{bcirc1} subtracts \eqref{bcirc2} and finish proof based on orthogonal property of $\boldsymbol F_{n_3}$. 
\end{proof}

\begin{lemma} \label{lemma4}
(Wedin $\sin \theta$ theorem with spectral norm \citep{wedin1972perturbation}) Given two matrices $\boldsymbol M^\star, \boldsymbol{ M} \in \mathbb R^{n_1 \times n_2}$, let $\boldsymbol U^\star (\boldsymbol U)$ and $\boldsymbol (V^\star)^T (\boldsymbol V^T)$ denote top $r$ left and right singular vectors of $\boldsymbol M^\star (\boldsymbol M)$, respectively. $\sigma_r^\star, \sigma^\star_{r+1}$ denote the $r$-th and $(r+1)$-th singular values of $\boldsymbol M^\star$. If $\left\| \boldsymbol M - \boldsymbol M^\star \right\| \leq \sigma^\star_r - \sigma^\star_{r+1}$, then 
\begin{align}
\left\| ( \boldsymbol I - \boldsymbol U \boldsymbol U^T) \boldsymbol U^\star\right\| \leq \sqrt 2 \frac{\max \left\{ \left\|( \boldsymbol M - \boldsymbol M^\star )^T \boldsymbol U^\star\right\| , \left\|( \boldsymbol M - \boldsymbol M^\star )^T \left( \boldsymbol V^\star \right)^T \right\| \right\} }{\sigma^\star_r - \sigma^\star_{r+1} - \left\| \boldsymbol M - \boldsymbol M^\star  \right\|}.
\end{align}
\end{lemma}

\begin{lemma} \label{lemma5}
($\epsilon$-netting for computing spectral norm \citep{nayer2022fast,vershynin2018high} ). For matrix $\boldsymbol M \in \mathbb R^{n_1+n_2}$ and fixed vectors $\boldsymbol w \in  \boldsymbol{\mathcal S}_{n_1}, \boldsymbol z \in \boldsymbol{\mathcal S}_{n_2}$, suppose $\text{Pr} \left( | \boldsymbol w^T \boldsymbol M \boldsymbol z| \leq b_0 \right) \geq 1 - p_0$. Considering $\epsilon$-net covering for $\boldsymbol{\mathcal S}_{n_1}$ and $\boldsymbol{\mathcal S}_{n_2}$ are $\hat{\boldsymbol{\mathcal S}}_{n_1}^\epsilon$ and $\hat{\boldsymbol{\mathcal S}}_{n_2}^\epsilon$, respectively. Then the (i) $\max_{\boldsymbol w \in \hat{\boldsymbol{\mathcal S}}_{n_1}^\epsilon, \boldsymbol z \in \hat{\boldsymbol{\mathcal S}}_{n_2}^\epsilon } |\boldsymbol w^T \boldsymbol M \boldsymbol z| \leq b_0$  and (ii) $\max_{\boldsymbol w \in {\boldsymbol{\mathcal S}}_{n_1}, \boldsymbol z \in {\boldsymbol{\mathcal S}}_{n_2} } |\boldsymbol w^T \boldsymbol M \boldsymbol z| \leq \frac{b_0}{1-2\epsilon -\epsilon^2 }$ hold with probability as least $1- | \hat{\boldsymbol{\mathcal S}}^\epsilon_{n_1} | | \hat{\boldsymbol{\mathcal S}}^\epsilon_{n_2}| p_0$. 
\end{lemma}

\begin{lemma}\label{lemma12}(\citep{kumar2022finetuning})
Suppose there are two matrices $\boldsymbol A \in \mathbb R^{r\times s}, \boldsymbol B \in \mathbb R^{s \times t}$ and they are tall matrices that $r\geq s\geq t$. Then we have
\begin{align}
\sigma_{\min}(\boldsymbol A \boldsymbol B) \geq \sigma_{\min}(\boldsymbol A) \sigma_{\min}(\boldsymbol B).
\end{align}
\end{lemma}

\begin{lemma} \label{lemma13} (\citep{vershynin2018high})
For any fixed notation of norm $\left\|\cdot \right\|$, define a unit norm ball $\boldsymbol{\mathcal B}_1:=\{ \boldsymbol x \in \mathbb R^n: \left\| \boldsymbol x\right\| \leq 1 \}$ with distance measure $\left\|\cdot \right\|$. Then the covering number of $\boldsymbol{\mathcal B}_1$ (with respect to the norm $\left\|\cdot \right\|$) satisfies the bound
\begin{align}
\Phi( \boldsymbol{\mathcal B}_1, \left\|\cdot \right\|,\epsilon) \leq \left( \frac{3}{\epsilon}\right)^n, \forall \epsilon \in (0,1).
\end{align}
\end{lemma}

\section{Supplementary Experiments}\label{sec5}
All the experiments are conducted on Matlab2024a in an Intel Core i9-14900K processor with 128 RAM.

\subsection{Details of synthetic data}
For set values of $n_1,n_2,n_2,r,m$, we synthetic the growth truth $\boldsymbol{\mathcal X}^\star$ as follows. We generate a same-size random tensor $\boldsymbol{\mathcal X}$, which has i.i.d. standard Gaussian entries and obtain $\overline{\boldsymbol{\mathcal X}} = \text{fft}(\boldsymbol{\mathcal X},[ \; ],3)$. Then we conduct SVD for each frontal slice $\overline{\boldsymbol X}^{(i)} = \overline{\boldsymbol U}^{(i)} \overline{\boldsymbol S}^{(i)}{\overline{\boldsymbol V}^{(i)}}^\top$ and set first $r$ singular values to be linearly distributed from 1 to $\frac{1}{\kappa}$ and other singular values to be zero in each $\overline{\boldsymbol S}^{(i)}$. We construct $\overline{\boldsymbol{X}^\star}^{(i)} = \overline{\boldsymbol U}^{(i)} \overline{\boldsymbol S}^{(i)}{\overline{\boldsymbol V}^{(i)}}^\top$ and obtain $\boldsymbol{\mathcal X}^\star = \text{ifft}(\overline{\boldsymbol X}^\star,[\;],3)$, which has tubal rank $r$ and tensor condition number $\kappa$. For each of the 20 Monte Carlo trials, each sensing tensor $\boldsymbol{\mathcal A}_i$ has i.i.d. standard Gaussian entries. The local measurements $\boldsymbol Y$ are generated from the local TCS model (2) in the main paper. 

\subsection {Additional video compressed sensing experiments}
Here, we provide additional experiments to clarify why the proposed Alt-PGD-Min performs worse than LRcCS in Figures 3 and 4 of the main paper. The essential reason is that the tensor condition number of the tested video is significantly larger than the condition number of its reshaped matrix. Consequently, to guarantee convergence, we have to set small enough constant step sizes for Alt-PGD-Min and LRcCS as $1\times 10^{-6}$ and $3\times 10^{-5}$, respectively. As shown in Figure \ref{fig5}, with sufficiently large iterations, Alt-PGD-Min ultimately achieves better performance than LRcCS, despite the slower convergence. This demonstrates the necessity of the precondition operator in Alt-ScalePGD-Min to achieve efficient recovery.
\begin{figure}[htbp]
 \centering
\includegraphics[width=6cm,height=4cm]{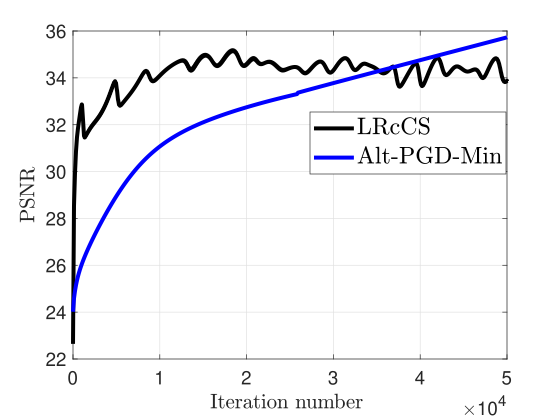}
\caption{ Comparison between LRcCS and Alt-PGD-Min under sufficiently large iterations 50000.}\label{fig5}
\end{figure}

We select five color video sequences Akiyo, Carphone, Foreman, Hall, and Highway from the open-source YUV video dataset \footnote{  http://trace.eas.asu.edu/yuv/ }. For comparing efficiently (due to the slow convergence of LRcCS and Alt-PGD-Min), we choose the first 50 frames of each video to test all algorithms and resize each frame into $72\times 88$. For each video, we test the performance of three algorithms under three different per-frame sample sizes $m=2000, 2250, 2500$. 


\begin{table}[htbp]
\centering
\begin{tabular}{|c|c|c|c|c|c|c|}
\hline
\textbf{Sample size} & \textbf{Videos} & \textbf{Akiyo} & \textbf{Carphone} & \textbf{Container} & \textbf{Hall} & \textbf{Highway} \\
\hline
\multirow{3}{*}{2000} & LRcCS & 33.50 & 31.47 & 33.29 & 30.76 & 33.24 \\
 & Alt-PGD-Min & 26.72 & 25.37  & 27.43 &  25.66& 26.81 \\
 & Alt-ScalePGD-Min & 35.92 & 32.68 & 36.00 & 33.01 & 35.08 \\
\hline
\multirow{3}{*}{2250} & LRcCS & 33.68 & 31.65 & 33.32 & 31.03 & 33.31 \\
 & Alt-PGD-Min & 27.47 & 25.86 & 28.12 &  26.17& 27.35 \\
 & Alt-ScalePGD-Min & 37.87 & 34.09 & 41.06 & 34.79 & 36.70 \\
\hline
\multirow{3}{*}{2500} & LRcCS & 33.68 & 31.77 & 33.63 & 31.50 & 33.35 \\
 & Alt-PGD-Min & 28.16 &  26.19 & 28.54  &  26.73& 27.81 \\
 & Alt-ScalePGD-Min & 38.86 & 35.14 & 45.34 & 36.18 & 37.51 \\
\hline
\end{tabular}
\caption{PSNR values given by the three methods under different per frame sample sizes \(m = 2000, 2250, 2500\).} \label{tab1}
\end{table}

Because the videos are just approximately low rank and have larger matrix and tensor condition numbers, the step size must be small to guarantee LRcCS and Alt-PGD-Min's convergence. Thus, we tune the step size of each algorithm to have the largest step size with a convergence guarantee. For the five videos, the step sizes for Alt-PGD-Min are set as $\eta = 1 \times 10^{-6}, 5 \times 10^{-7}, 2 \times 10^{-7}, 3 \times 10^{-7}, 1 \times 10^{-7}$, step sizes for LRcCS are set as $\eta = 2\times 10^{-5}, 1\times 10^{-5}, 1\times 10^{-5},  1\times 10^{-5}, 8 \times 10^{-6}$. Since the large condition number of videos would result in a very slow convergence speed, we set the maximum iterations for LRcCS and Alt-PGD-Min the same as 2000. For the proposed Alt-ScalePGD-Min, we set the fixed step size $\eta=0.8$ for all settings and the maximal number of iterations is 100. We set the $r=10$ for all algorithms and use the spectral initialization method to obtain the initial point.

We utilize three quantitative image quality indices used in \cite{wu2024smooth} to evaluate the performance of all algorithms numerically, including PSNR,
 structure similarity (SSIM), and feature similarity (FSIM). The larger the three indices values, the better the recovery performance denotes. Table \ref{tab1}, Table \ref{tab2} and Table \ref{tab3} show PSNR, SSIM and FSIM comparison values, respectively.

\begin{table}[htbp]
\centering
\begin{tabular}{|c|c|c|c|c|c|c|}
\hline
\textbf{Sample size} & \textbf{Videos} & \textbf{Akiyo} & \textbf{Carphone} & \textbf{Container} & \textbf{Hall} & \textbf{Highway} \\
\hline
\multirow{3}{*}{2000} & LRcCS & 0.9052 & 0.8810 & 0.9024 & 0.8366 & 0.9233 \\
 & Alt-PGD-Min & 0.6649 & 0.6322  & 0.6943 &  0.6766& 0.6849 \\
 & Alt-ScalePGD-Min & 0.9350 & 0.8911 & 0.9131 & 0.8798 & 0.9300 \\
\hline
\multirow{3}{*}{2250} & LRcCS & 0.9060 & 0.8858 & 0.9118 & 0.8417 & 0.9250 \\
 & Alt-PGD-Min & 0.6990 & 0.6570 & 0.7158 & 0.6892& 0.7038 \\
 & Alt-ScalePGD-Min & 0.9562 & 0.9101 & 0.9750 & 0.8921 & 0.9350 \\
\hline
\multirow{3}{*}{2500} & LRcCS & 0.9086 & 0.8880 & 0.9223 & 0.8540 & 0.9269 \\
 & Alt-PGD-Min & 0.7310 & 0.6730 &0.7467  &  0.6924& 0.7155 \\
 & Alt-ScalePGD-Min & 0.9644 & 0.9250 & 0.9908 & 0.9189 & 0.9404 \\
\hline
\end{tabular}
\caption{SSIM values given by the three methods under different per frame sample sizes \(m = 2000, 2250, 2500\).} \label{tab2}
\end{table}

\begin{table}[htbp]
\centering
\begin{tabular}{|c|c|c|c|c|c|c|}
\hline
\textbf{Sample size} & \textbf{Videos} & \textbf{Akiyo} & \textbf{Carphone} & \textbf{Container} & \textbf{Hall} & \textbf{Highway} \\
\hline
\multirow{3}{*}{2000} & LRcCS & 0.9304 & 0.9356 & 0.9452 & 0.9175 & 0.9417 \\
 & Alt-PGD-Min & 0.8146 & 0.8127 & 0.8321 &  0.8215& 0.8427 \\
 & Alt-ScalePGD-Min & 0.9569 & 0.9430 & 0.9579 & 0.9415 & 0.9506 \\
\hline
\multirow{3}{*}{2250} & LRcCS & 0.9311 & 0.9379 & 0.9458& 0.9189 & 0.9424 \\
 & Alt-PGD-Min & 0.8306 & 0.8258 & 0.8533 & 0.8327& 0.8636 \\
 & Alt-ScalePGD-Min & 0.9701 & 0.9526 & 0.9857 & 0.9573 & 0.9583 \\
\hline
\multirow{3}{*}{2500} & LRcCS & 0.9316 & 0.9405 & 0.9513 & 0.9211 & 0.9430 \\
 & Alt-PGD-Min & 0.8468 & 0.8297 &0.8719  &  0.8548& 0.8597 \\
 & Alt-ScalePGD-Min & 0.9751 & 0.9614 & 0.9942 & 0.9673 & 0.9648 \\
\hline
\end{tabular}
\caption{FSIM values given by the three methods under different per frame sample sizes \(m = 2000, 2250, 2500\).} \label{tab3}
\end{table}

\subsection{MRI data compressed sensing }
In this section, we utilize MRI data from the Simulated Brain Database \footnote{http://brainweb.bic.mni.mcgill.ca/brainweb/}. This database includes a collection of realistic MRI data volumes generated by an MRI simulator, which is widely employed by the neuroimaging community. Due to the slow convergence of LRcCS and Alt-PGD-Min, we choose the middle 50 frames and resize each frame into $60 \times 75$ resolution. We compare three methods under three different per-frame sample sizes as $m=2000, 2250, 2500$, whose visual comparison results are shown in Figure \ref{fig1}, Figure \ref{fig2}, and Figure \ref{fig3}. We set $r=10$ for three methods in all settings. For the LRcCS and Alt-PGD-Min, we select the maximum step size as  $1 \times 10^{-5}$ and $2\times 10^{-6}$, respectively. The maximum step sizes are set to guarantee convergence. The maximum iteration number is 2000. The step size is set 0.5 for Alt-ScalePGD-Min with the maximum number of iterations is 50, which can converge and obtain satisfied performance already. We also plot the PSNR values along each recovered frame, which are shown in Figure. \ref{fig4}. These results indicate that Alt-ScalePGD-Min can achieve the best recovery performance with minimum iteration numbers. In addition, with the sample size increasing, only the performance of Alt-ScalePGD-Min can be improved significantly. 

\begin{figure}[htbp]
\subfigure[ Original]{
\begin{minipage}[t]{0.25\linewidth}
\centering
\includegraphics[scale=0.5]{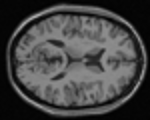}
\end{minipage}%
}%
\subfigure[ LRcCS]{
\begin{minipage}[t]{0.25\linewidth}
\centering
\includegraphics[scale=0.5]{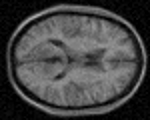}
\end{minipage}%
}%
\subfigure[Alt-PGD-Min ]{
\begin{minipage}[t]{0.25\linewidth}
\centering
\includegraphics[scale=0.5]{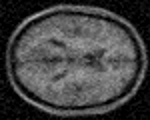}
\end{minipage}%
}%
\subfigure[ Alt-ScalePGD-Min]{
\begin{minipage}[t]{0.25\linewidth}
\centering
\includegraphics[scale=0.5]{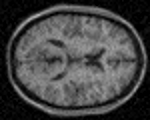}
\end{minipage}%
}%
\caption{Comparison of $35$-th frame in recovered MRI with sample size $m=2000$ for per frame. LRcCS: PSNR = 27.64, SSIM = 0.7888, FSIM = 0.8707. Alt-PGD-Min: PSNR = 24.58, SSIM = 0.6522, FSIM = 0.8018. Alt-ScalePGD-Min: PSNR = 28.46, SSIM = 0.8096, FSIM = 0.8800.  }
\label{fig1}
\end{figure}

\begin{figure}[htbp]
\subfigure[ Original]{
\begin{minipage}[t]{0.25\linewidth}
\centering
\includegraphics[scale=0.5]{ori-35.png}
\end{minipage}%
}%
\subfigure[ LRcCS ]{
\begin{minipage}[t]{0.25\linewidth}
\centering
\includegraphics[scale=0.5]{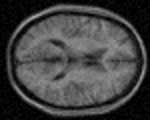}
\end{minipage}%
}%
\subfigure[Alt-PGD-Min ]{
\begin{minipage}[t]{0.25\linewidth}
\centering
\includegraphics[scale=0.5]{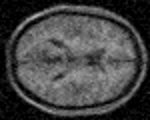}
\end{minipage}%
}%
\subfigure[ Alt-ScalePGD-Min]{
\begin{minipage}[t]{0.25\linewidth}
\centering
\includegraphics[scale=0.5]{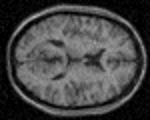}
\end{minipage}%
}%
\caption{Comparison of $35$-th frame in recovered MRI with sample size $m=2250$ for per frame. LRcCS: PSNR= 27.66, SSIM = 0.7868, FSIM = 0.8708. Alt-PGD-Min: PSNR = 25.18, SSIM = 0.6783, FSIM = 0.8151. Alt-ScalePGD-Min: PSNR = 29.92, SSIM = 0.8583, FSIM = 0.9026. }
\label{fig2}
\end{figure}

\begin{figure}[htbp]
\subfigure[ Original]{
\begin{minipage}[t]{0.25\linewidth}
\centering
\includegraphics[scale=0.5]{ori-35.png}
\end{minipage}%
}%
\subfigure[ LRcCS ]{
\begin{minipage}[t]{0.25\linewidth}
\centering
\includegraphics[scale=0.5]{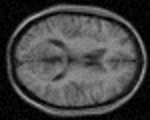}
\end{minipage}%
}%
\subfigure[Alt-PGD-Min ]{
\begin{minipage}[t]{0.25\linewidth}
\centering
\includegraphics[scale=0.5]{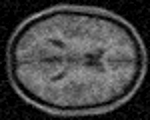}
\end{minipage}%
}%
\subfigure[ Alt-ScalePGD-Min]{
\begin{minipage}[t]{0.25\linewidth}
\centering
\includegraphics[scale=0.5]{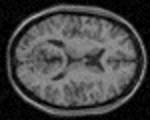}
\end{minipage}%
}%
\caption{Comparison of $35$-th frame in recovered MRI with sample size $m=2500$ for per frame. LRcCS: PSNR = 28.08, SSIM = 0.8027, FSIM = 0.8766. Alt-PGD-Min: PSNR = 25.75, SSIM = 0.7029, FSIM = 0.8275. Alt-ScalePGD-Min: PSNR = 31.79, SSIM = 0.8977, FSIM = 0.9276.}
\label{fig3}
\end{figure}

\begin{figure}[htbp]
\subfigure[ $m = 2000$ ]{
\begin{minipage}[t]{0.3\linewidth}
\centering
\includegraphics[scale=0.35]{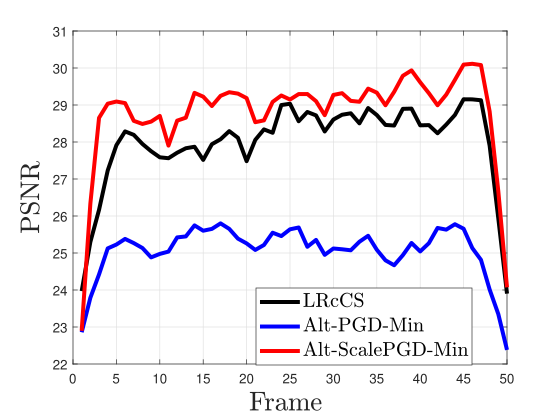}
\end{minipage}%
}%
\subfigure[ $m = 2250$ ]{
\begin{minipage}[t]{0.3\linewidth}
\centering
\includegraphics[scale=0.35]{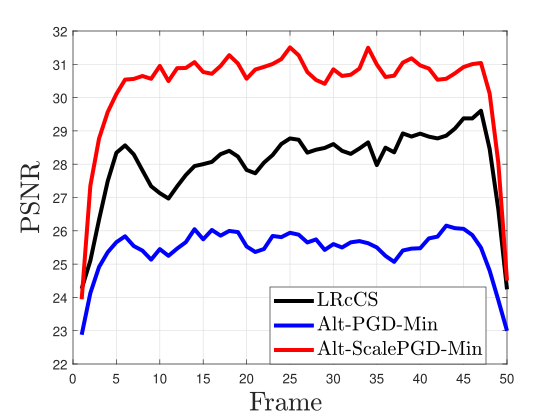}
\end{minipage}%
}%
\subfigure[ $m = 2500$ ]{
\begin{minipage}[t]{0.3\linewidth}
\centering
\includegraphics[scale=0.35]{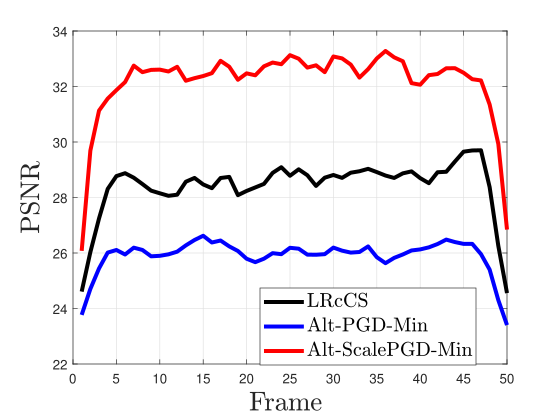}
\end{minipage}%
}%
\caption{ PSNR evaluation for each frame under different sample size. }
\label{fig4}
\end{figure}


\end{document}